\documentclass[opre,nonblindrev]{myinforms} %

\OneAndAHalfSpacedXI %

\usepackage{natbib}
\bibpunct[, ]{(}{)}{,}{a}{}{,}%
\def\bibfont{\small}%

\TheoremsNumberedThrough     %
\ECRepeatTheorems

\EquationsNumberedThrough    %

\usepackage[utf8]{inputenc} %
\usepackage{hyperref}       %
\usepackage{url}            %
\usepackage{booktabs}       %
\usepackage{caption}
\usepackage{subcaption}
\usepackage{amsfonts}       %
\usepackage{nicefrac}       %
\usepackage{microtype}      %
\usepackage{xcolor}         %
\usepackage{amsmath}
\usepackage{amssymb}
\usepackage{multirow}
\usepackage{makecell}
\usepackage{algorithm}
\usepackage{tabularray}
\usepackage{algpseudocode}
\usepackage[capitalise]{cleveref}
\usepackage{comment}
\usepackage{shortcutsg}

\usepackage{tikz}
\usetikzlibrary{positioning,arrows,decorations,decorations.markings,shadows,positioning,arrows.meta,matrix,fit}

\newcommand{\pdim}{\text{Pdim}}
\newcommand{\ndim}{\text{Ndim}}
\newcommand{\grow}{\text{Grow}}

\newcommand{\F}{\mathcal{F}}
\newcommand{\Z}{\mathcal{Z}}
\newcommand{\X}{\RR^p}
\newcommand{\Y}{\RR^d}
\newcommand{\D}{\mathcal{D}}

\newcommand{\expect}{\mathbb{E}}
\newcommand{\ind}{\mathbb{I}}
\newcommand{\pr}{\mathbb{P}}
\newcommand{\edit}[1]{#1}

\newcommand{\Data}{O}
\newcommand{\data}{o}
\newcommand{\nuisance}{\nu}
\newcommand{\propensity}{e}

\crefname{assumption}{Assumption}{Assumptions}
\crefname{remark}{Remark}{Remarks}

\usepackage[compact]{titlesec}

\begin{document}
		
	\RUNAUTHOR{Hu et al.}
	
	\RUNTITLE{CLO with Partial Feedback}
	
	\TITLE{Contextual Linear Optimization under Partial Feedback}
	
	\newcommand*\samethanks[1][\value{footnote}]{\footnotemark[#1]}
	\ARTICLEAUTHORS{%
		\AUTHOR{Yichun Hu\thanks{Authors are listed in alphabetical order.}}
 		\AFF{Cornell University, Ithaca, NY 14853,  \EMAIL{yh767@cornell.edu}}
        \AUTHOR{Nathan Kallus\samethanks}
 		\AFF{Cornell University, New York, NY 10044, \EMAIL{kallus@cornell.edu}}
        \AUTHOR{Xiaojie Mao\samethanks}
 		\AFF{Tsinghua University, 100084 Beijing, China, \EMAIL{maoxj@sem.tsinghua.edu.cn}}
        \AUTHOR{Yanchen Wu\samethanks}
 		\AFF{Tsinghua University, 100084 Beijing, China, \EMAIL{wu-yc23@mails.tsinghua.edu.cn}}
} %

\ABSTRACT{%
Contextual linear optimization (CLO) uses predictive contextual features to reduce uncertainty in random cost coefficients in the objective and thereby improve decision-making performance. A canonical example is the stochastic shortest path problem with random edge costs (\eg, travel time) and contextual features (\eg, lagged traffic, weather). While existing work on CLO assumes fully observed cost coefficient vectors, in many applications the decision maker observes only partial feedback corresponding to each chosen decision in the history. In this paper, we study both a bandit-feedback setting (\eg, only the overall travel time of each historical path is observed) and a semi-bandit-feedback setting (\eg, travel times of the individual segments on each chosen path are additionally observed). 
  We propose a unified class of offline learning algorithms for CLO with different types of feedback, following a powerful induced empirical risk minimization (IERM) framework that integrates estimation and optimization. We provide a novel fast-rate regret bound for IERM that allows for misspecified model classes and flexible choices of estimation methods. To solve the partial-feedback IERM, we also tailor computationally tractable surrogate losses. A byproduct of our theory of independent interest is the fast-rate regret bound for IERM with full feedback and a misspecified policy class. We compare the performance of different methods numerically using  stochastic shortest path examples on simulated and real data and provide practical insights from the empirical results.
}%

\maketitle

\section{Introduction}

Decision-making under uncertainty is a prevalent challenge across numerous real-world applications, ranging from optimizing urban shortest routes and planning product inventory to managing complex investment portfolios. These problems typically involve optimization with uncertain input quantities (\eg, travel time, product demand, asset returns), whose values are unknown at the time of decision-making.  
The increasing availability of data, coupled with advancements in machine learning (ML), has spurred the development of data-driven approaches. In particular, Contextual Stochastic Optimization (CSO) emerges as a powerful framework that leverages auxiliary information predictive of the uncertain quantities, known as \emph{contextual features} or \emph{covariates} (\eg, traffic information, customer demographics, economic fundamentals), to reduce the inherent uncertainty and improve decision quality. CSO focuses on finding optimal decicisions that minimize the expected costs, conditioned on observed contextual feature values. 

Contextual linear optimization (CLO) stands out as a fundamental and extensively studied instance within the broader CSO paradigm. In CLO, the objective cost function is linear with respect to both the decision variables and the uncertain quantities of interest, which can be  predicted using contextual features (see \cref{eq:CLO}). One prominent example of CLO is the stochastic shortest path problem where the decision-maker wants to choose routes to transport units from a starting point to a desired destination as fast as possible. This involves a linear optimization problem where the linear objective function gives the total travel time of each possible path and the decision constraints characterize the feasible paths. In this problem, the objective relies on the travel time of all edges within the road network, which is uncertain and unknown at the time of making route decision. Meanwhile, the decision-maker may observe useful contextual features such as traffic, weather, holidays, time and hope to leverage these to aid  decision-making. 

Traditionally, such problems have been addressed using a two-stage approach commonly referred to as Predict-Then-Optimize (PTO) or Estimate-Then-Optimize (ETO). This framework separates the prediction/estimation task from the subsequent optimization task, treating them as distinct and independent steps.
In the initial prediction/estimation stage, an ML model is trained to predict the uncertain input quantities, such as edge travel times, from contextual features. This training typically involves minimizing standard prediction error metrics, like the sum of squared errors, without explicit consideration of the downstream optimization task. Subsequently, in the optimization stage, these predictions of the uncertain quantities are "plugged in" as deterministic inputs to the optimization problem, which is then solved to derive the final decisions. While this approach is convenient for implementation, the two stages are decoupled and not well coordinated. 

Recently, a range of literature has proposed a new paradigm that aims to \emph{integrate} the prediction/estimation and optimization components, where the prediction/estimation is directly informed by the downstream optimization task. This paradigm is referred to by various names in the literature, including Smart Predict-then-Optimize, Integrated Learning and Optimization, Decision-Focused Learning, and End-to-End Learning, etc.  In this paper we call them \emph{integrated} or \emph{end-to-end} approaches. 
The core idea is to train the ML model by directly minimizing a "task loss" or "decision error" that quantifies the suboptimality of the decisions induced by ML predictions, rather than focusing solely on the prediction error. Thereby, this approach seeks the best possible \emph{decision quality} by optimizing over the class of decision policies induced by the ML model, incorporating the structure of the downstream optimization problem during training. 
This approach has demonstrated significant performance advantages over ETO, particularly in scenarios where the underlying ML model is misspecified (\ie, the true  relationship between contextual features and target uncertain quantities is not perfectly captured by the chosen model class). 

However, a crucial limitation shared by all existing CLO literature (and the more broad CSO literature), including both ETO and integrated approaches, is the common assumption of \emph{full-feedback} data. This typically implies that for each historical data instance, the decision-maker had access to and observed the complete, true vector of uncertain quantities (\eg, the travel time of every edge in the network). However,  this assumption deviates significantly from many real-world applications, where decision-makers usually have access to only \emph{partial-feedback} data. In these  settings, observations are limited to outcomes directly related to the specific decisions that were historically made, with information about counterfactual outcomes remaining unobserved. 

\begin{figure}[t]
    \centering
    \begin{subfigure}[t]{0.32\textwidth}
        \centering
        \includegraphics[width=\textwidth]{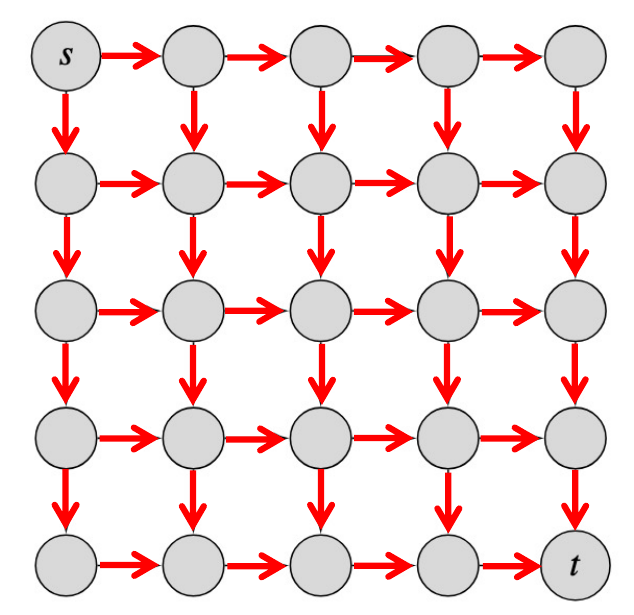}
        \caption{Full Feedback}
        \label{fig:full}
    \end{subfigure}
    \hfill
    \begin{subfigure}[t]{0.32\textwidth}
        \centering
        \includegraphics[width=\textwidth]{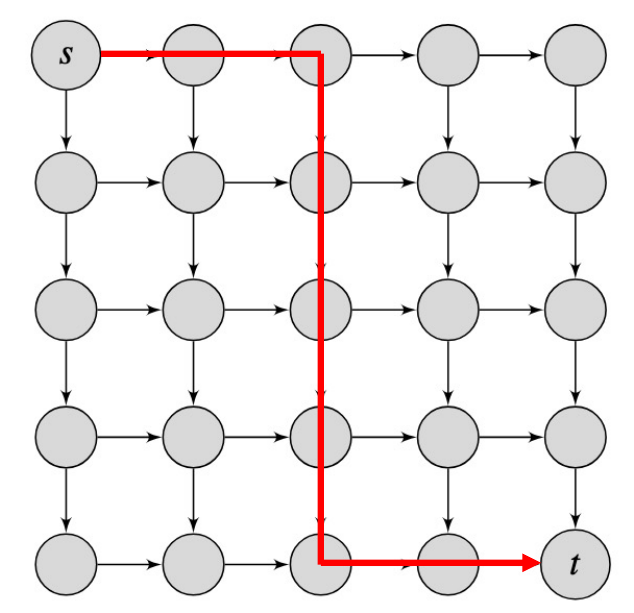}
        \caption{Bandit Feedback}
        \label{fig:bandit}
    \end{subfigure}
    \hfill
    \begin{subfigure}[t]{0.32\textwidth}
        \centering
        \includegraphics[width=\textwidth]{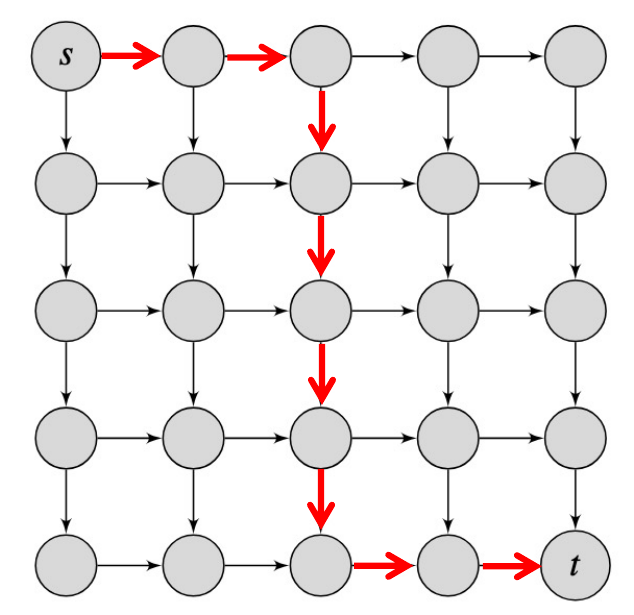}
        \caption{Semi-Bandit Feedback}
        \label{fig:semi}
    \end{subfigure}
    \caption{Illustrations of full-feedback, bandit-feedback, and semi-bandit-feedback settings for the stochastic shortest path problem on a $5 \times 5$ grid. 
Highlighted edges indicate the information observed under each feedback model. 
In the full-feedback setting, the travel times of all $40$ edges in the road network are observed for each data point. 
In the bandit-feedback setting, only the scalar total travel time of the chosen path is observed. 
In the semi-bandit-feedback setting, the travel time of each edge along the chosen path is observed.
We will conduct synthetic experiments on this road network in \cref{sec: synthetic}.
}
    \label{fig:feedback-comparison}
\end{figure}

The partial-feedback problem can be well illustrated in the aforementioned stochastic shortest-path problem, as shown in \cref{fig:feedback-comparison}. 
In this problem, 
if a user chooses a particular path, they may be able to observe only the total travel time for that specific path, but not the travel times for any alternative paths that were not taken (\cref{fig:bandit}). This scenario corresponds to a \emph{bandit-feedback} setting, where only the final outcome of the chosen action is revealed \citep{lattimore2020bandit}. An intermediate scenario between full feedback and bandit feedback also exists, which we refer to as \emph{semi-bandit feedback}. In this case, a decision-maker observes more granular information than pure bandit feedback but still not the full vector of uncertain quantities of interest. For example, in the shortest path problem, the travel times for each segment along the chosen path (instead of only the total travel time) might also be observed, but still not the costs of segments on unchosen paths (\cref{fig:semi}). As another example, \cite{elmachtoub20a} consider a news article recommendation problem and model it as a full-feedback CLO problem. However, in reality, this may also correspond to partial feedback, since the decision-maker may observe the user feedback (\eg, clicks) for articles that were actually recommended, but not necessarily the feedback for articles that were not recommended.

This paper addresses the critical research gap by focusing on the development of a novel, integrated, end-to-end framework for CLO under partial feedback, including both  bandit-feedback and semi-bandit settings. Our paper tackles several key challenges inherent to these limited information environments and makes the following key contributions:
\begin{itemize}
    \item \textbf{A New End-to-End Learning Framework with Partial Feedback.} 
    We develop a novel, end-to-end learning framework that adapts the powerful Induced Empirical Risk Minimization (IERM) paradigm to partial-feedback settings, including both bandit and semi-bandit feedback. A central challenge is evaluating a policy's expected cost without observing the full cost vector. Our framework overcomes this by introducing a score function that can be estimated from the observed data to serve in place of the unobserved cost vector. We propose a range of methods for constructing this score function. By optimizing the policy cost estimated via these scores, our framework directly targets decision quality. This approach is highly general, unifying the full-feedback, bandit, and semi-bandit settings within a single, coherent IERM framework.
    
    \item \textbf{Novel Fast-Rate Regret Guarantees under Model Misspecification.}
    We establish novel theoretical guarantees for the general IERM framework by providing a unified upper bound on the regret (\ie, decision sub-optimality) of the resulting IERM policy across different feedback settings. Crucially, our analysis explicitly accounts for model misspecification, a critical consideration for real-world applications where models are rarely perfect. It also incorporates a margin condition that characterizes the problem's intrinsic difficulty, enabling potentially fast convergence rates. The resulting regret bound can be decomposed into several interpretable components that trace different sources of error, which, when instantiated across feedback settings, lead to different fast-rate regret bounds.
    Our proof technique is novel and, as a significant byproduct of independent interest, it yields the first fast-rate regret bound for the traditional full-feedback CLO setting under misspecification.
    
    \item \textbf{Practical and Efficient Optimization via Surrogate Losses.} For practical implementation, we demonstrate the flexibility of our framework to leverage existing computational tools. Specifically, we show that popular surrogate losses designed for full-feedback CLO, including but not limited to the well-known SPO+ loss in \cite{elmachtoub2022smart}, can be seamlessly adapted to our bandit and semi-bandit settings. This makes the policy optimization problem computationally efficient and practical. We further support our approach with extensive empirical studies, including simulations on a stochastic shortest path problem and a real-world case study using Uber Movement data. These experiments demonstrate the effectiveness of our method.

\end{itemize}

\subsubsection*{Paper Organization.}  
We review the related literature in \cref{sec: literature}. 
In \cref{sec: problem}, we formally define contextual linear optimization with bandit feedback and discuss the key challenges in designing an end-to-end learning algorithm. 
We then introduce our proposed method for the bandit-feedback setting in \cref{sec: ierm-bandit} and demonstrate that the bandit and full-feedback settings can be unified within a single generic end-to-end learning framework. 
Theoretical regret guarantees for this framework are established in \cref{section: theory}, and computationally tractable solutions are developed in \cref{sec: spo+}. 
We further instantiate the framework for the semi-bandit-feedback setting in \cref{sec: semi}. 
Finally, we evaluate the performance of our proposed methods through computational experiments, using both synthetic and real-world data, in \cref{sec: numerical}.
Concluding remarks are provided in \cref{sec: conclusions}.

\section{Related Literature} \label{sec: literature}

\subsubsection*{Contextual Stochastic Optimization}

Our paper studies contextual stochastic optimization (CSO) for data-driven decision-making under uncertainty. CSO is a powerful framework that has been extensively studied in the recent literature, as surveyed systematically in  \cite{sadana2024survey} and \cite{mandi2024decision}. \cite{sadana2024survey} categorize methodologies for CSO problems into sequential learning and optimization (SLO), integrated learning and optimization (ILO), and decision rule optimization. \cite{mandi2024decision} focus on reviewing end-to-end decision-focused learning methods, roughly corresponding to the ILO category. This paper develops an integrated approach within contextual linear optimization (CLO), a specific and canonical class of CSO problems. 

For CLO problems, a seminal approach is the "smart" predict-then-optimize (SPO) framework proposed by \cite{elmachtoub2022smart}, which integrates learning and optimization through an SPO loss that directly targets the decision-making quality in predictive model training.
To address the non-convex and non-smooth nature of the SPO loss, a convex surrogate, SPO+, was also proposed for tractable optimization. Subsequent research has explored the theoretical properties of these losses, establishing generalization error bounds \citep{el2023generalization} and calibration properties \citep{liu2021risk,ho2022risk}.
Some recent theoretical analyses have revealed a nuanced performance landscape for integrating learning and optimization. For instance, \cite{hu2022fast} show that the conventional ETO approach may achieve faster regret convergence than its integrated counterpart with a well-specified model class. \cite{elmachtoub2023estimate} employ  a stochastic dominance analysis for parametric statistical models, showing that the integrated approach outperforms ETO under model misspecification, but the order reverses under correct specification. 
\cite{huang2024decision,bennounacontextual} focus on  misspecified models and develop new surrogate losses for more effective learning and optimization.  
Despite this extensive body of work, existing research has universally assumed a full-feedback setting. This paper is the first to study CLO under the more challenging partial-feedback settings.

\subsubsection*{Bandit and Semi-bandit Feedback}
Our paper studies CLO under partial feedback, encompassing both bandit and semi-bandit settings. 
Bandit feedback has been widely studied in the bandit learning literature \citep{lattimore2020bandit,slivkins2019introduction}. 
Given that our problem involves a linear objective and leverages contextual features, it naturally connects to the rich literature on linear and contextual bandits \citep[e.g.,][]{auer2002using,dani2008stochastic,li2010contextual, abbasi2011improved,rusmevichientong2010linearly}. 
In addition, the bandit literature has also studied different granularities of feedback. The standard bandit-feedback model reveals only a single overall feedback observation for each chosen decision, but when the decision consists of multiple components (\eg, the segments in a chosen path or items in a chosen assortment), it is possible to observe the feedback for each component of a decision separately. The latter is called semi-bandit feedback in the literature \citep[\eg,][]{chen2013combinatorial,neu2013efficient,combes2015combinatorial,kveton2015tight}. 
This established literature, however, focuses on online learning, where data is collected sequentially to manage the exploration-exploitation trade-off. By contrast, our work addresses the offline setting, where the goal is to learn an effective decision policy from a pre-existing batch of historical data for future decision-making.

\subsubsection*{Off-policy Evaluation and Optimization of Contextual Bandits} \label{sec: background ope}

There also exists a vast body of literature on contextual bandits in the offline setting \citep[\eg,][]{zhao2012estimating,dudik2014doubly,swaminathan2015batch,kallus2018balanced,kitagawa2018should,athey2021policy,zhou2018offline,kallus2022doubly,si2023distributionally}. This  literature studies both off-policy evaluation (\ie, evaluating the  performance of a given decision policy from offline bandit-feedback data) and off-policy optimization (\ie, searching for a policy that optimizes the evaluated performance).
However, the literature is predominantly centered on a discrete action space. In contrast, our paper studies these problems with a general polytopal decision space. 

While a few recent studies have considered continuous decisions \citep[\eg,][]{kallus2018policy,ai2024data},  \cite{chernozhukov2019semi} is the most relevant. \cite{chernozhukov2019semi} also propose a doubly robust off-policy evaluation approach for a setting  where feedback is linear in the decision (or its  transformations). However, our paper differs from \cite{chernozhukov2019semi} in several key aspects. First, our proposed policy evaluation approach allows the linear span of the decision space to be strictly smaller than the entire space, which is the case in our canonical example of the stochastic shortest path problem. This renders a matrix-valued component in the doubly robust evaluation rank-deficient, thereby violating a key assumption imposed implicitly in \cite{chernozhukov2019semi}. Second, our learning guarantees are very different. We provide a fast-rate regret analysis under a margin condition, which contrasts with the slow-rate regret analysis in \cite{chernozhukov2019semi}. 
Finally, we tackle a computationally more complex policy optimization problem. Because the integration of learning and optimization leads to a policy class implicitly defined through an inner optimization problem, our overall task is fundamentally bi-level. We address this challenge by leveraging efficient surrogate loss functions. In contrast, the policy class in \cite{chernozhukov2019semi} can be specified directly and amenable to standard optimization methods.

\section{Contextual Linear Optimization with Bandit Feedback}\label{sec: problem}

In a CLO problem, a decision-maker needs to select a decision, denoted by $z$, from a set of feasible options $\Z$. The decision incurs an uncertain linear cost $Y\tr z$, which is determined by a random coefficient vector $Y\in\Y$ that is not observed at the time of the decision.
The decision-maker does, however, observe predictive features (or contexts) $X\in\mathcal{X}\subseteq \X$ prior to decision, which help reduce uncertainty in the random cost coefficients $Y$. The goal of CLO is to use the features $X$ to choose the decision $z$ that minimizes the expected cost. Mathematically, CLO can be expressed either as a contextual stochastic optimization problem or a linear optimization problem where the cost vector is a conditional expectation:
\begin{equation}\label{eq:CLO}\textstyle
\min_{z\in\Z}\Eb{Y\tr z\mid X=x}=\min_{z\in\Z}f^*(x)\tr z,\quad\text{where}\quad f^*(x)=\Eb{Y\mid X=x}.
\end{equation}
Throughout this paper, we assume that $\Z$ is a polytope; that is, $\Z$ is the convex hull of a finite set of vertices $\Z^\angle$ with $\sup_{z\in\Z^\angle}\norm{z}\leq B$. We also assume that $Y$ is bounded and, without loss of generality, that $Y\in \mathcal{Y} = \braces{y: \norm{y}\le 1}$.

To illustrate these concepts, recall our stochastic shortest path example. Here, the context $X$ is a vector of features predictive of travel times (\eg, weather, time of day). The cost coefficient vector $Y$ contains the random travel time for each of the $d$ edges in the road network, while $f^*(x)$ represents the expected travel time on each edge, conditional on context $x$. A decision $z$ is a vector that selects a feasible path from an origin to a destination, and the decision set $\mathcal{Z}$ is the collection of all such valid paths. The total cost, $Y\tr z$, is then the total travel time along the path $z$.

The central question in CLO is how to use {data} to learn a good policy, $\pi:\mathcal{X}\to\Z$, mapping contextual feature observations to effective decisions.
We begin in \Cref{sec: full-feedback} by reviewing the well-studied full-feedback setting. We provide background on two major existing approaches and highlight the end-to-end Induced Empirical Risk Minimization (IERM) approach, which integrates estimation and optimization.
We then turn to the bandit-feedback setting in \Cref{sec:ierm}---the primary focus of this paper---where we formulate the problem and describe the associated challenges for developing an end-to-end IERM framework. An extension to a semi-bandit feedback setting is deferred to \cref{sec: semi}.

\subsection{Background: CLO with Full Feedback}\label{sec: full-feedback}
The existing CLO literature assumes that the decision-maker can observe a batch of existing data $(X_i, Y_i)$ for $i = 1, \dots, n$, where the data are viewed as independent and identically distributed (i.i.d.) draws from the population joint distribution of $(X, Y)$.
This corresponds to a \emph{full-feedback} setting, where the random coefficient vector $Y_i$ for each past event is fully observed, so that the corresponding costs $Y_i^\top z$ for \textit{any} potential decision $z$ is fully known. In the stochastic shortest path example, this means that for each historical trip, we observe the complete vector of travel times on all segments of the road network (\cref{fig:full}). This complete information allows us to retroactively calculate the total travel time for any alternative route, even those that were not actually taken.

We consider two approaches for CLO with full-feedback data: an estimate-then-optimize (ETO) approach, and an end-to-end induced empirical risk minimization (IERM) approach.
Both the ETO and the IERM approaches are built on the concept of a generic plug-in policy (also referred to as an induced policy). Formally, for any given function $f: \R{p} \to \R{d}$ that can be viewed as a potential estimate for $f^*$ in \cref{eq:CLO}, the plug-in policy $\pi_f$ is defined as:
\begin{align}\label{eq: plug in policy}\textstyle
\pi_f(x) \in \arg\min_{z\in\Z} f(x)^\top z.
\end{align}
Note that for any given context value $x$, this policy solves a linear programming problem with coefficients given by $f(x)$.
An optimal policy for \cref{eq:CLO} corresponds to a plug-in policy $\pi_{f^*}$, which uses the true conditional expectation function, $f^*(x)$.
Without loss of generality, we restrict the value of $\pi_f(x)$ to the set of vertices $\Z^\angle$ of the polytope $\Z$, with ties broken according to some fixed rule (\eg, a total ordering over $\Z^\angle$ such as lexicographic).

The ETO approach first uses a supervised learning method to train an estimator $\hat{f}$ for the function $f^*$ in \cref{eq:CLO}, which predicts the coefficient vector $Y$ given the context $X$.
This can typically be implemented by choosing a hypothesis class of functions $\F \subseteq [\mathcal{X} \to \R{d}]$ (\eg, linear functions, decision trees, neural networks) and finding $\hat f \in \F$ that minimizes a prediction fitness criterion, such as the sum of squared errors. 
After training the estimator $\hat f$, the ETO approach then makes decisions according to the induced policy $\pi_{\hat f}$.  
 Notably, the ETO approach provides a two-step procedure to choose the policy $\pi_{\hat f}$ from the class of plug-in policies induced by $\F$:
 \begin{align}\label{eq: induced class}
    \Pi_\F = \braces{\pi_f: f \in \F} 
 \end{align}
 A key feature of this class is that because each policy $\pi_f$ is the result of an optimization problem in the form of \cref{eq: plug in policy}, it inherently satisfies the CLO constraints.
 However, in ETO, the first-step estimation of $\hat f$ completely ignores the downstream optimization; consequently, the process of choosing the policy $\pi_{\hat f}$ from the induced class $\Pi_\F$ is not entirely coherent.
 
 To integrate estimation and the optimization, we consider the IERM framework,  which directly searches for an induced policy within $\Pi_\F$ in \cref{eq: induced class} to minimize the expected cost \citep{hu2022fast}:  
 \begin{align} \label{eq: best in class policy full}%
    \Tilde{\pi}^* \in \arg\min_{\pi\in \Pi_\F} V(\pi), ~~ \text{ where } V(\pi)=\expect\bracks{f^*(X)^\top \pi(X)} = \expect\bracks{Y^\top \pi(X)}. 
 \end{align}
 Although the expected cost $V(\pi)$ is unknown, it can be unbiasedly estimated by the sample average, and the IERM approach minimizes the sample average cost on the training data: 
 \begin{align}\label{eq: IERM-full}%
    \hat\pi_{\text{F}} \in \arg\min_{\pi \in \Pi_\F} \frac{1}{n}\sum_{i=1}^n Y_i^\top \pi(X_i).
\end{align}
 This can be viewed as a one-step process that directly chooses the least-average-cost policy $\hat\pi_{\text{F}}$ within $\Pi_\F$. Alternatively, because any such policy is a plug-in policy, we can write $\hat\pi_{\text{F}} = \pi_{\hat f}$ for some $\hat f \in \F$. From this perspective, IERM can also be viewed as first choosing a $\hat f \in \F$ and then using the policy $\pi_{\hat f}$ to make decisions.\footnote{The framework in \cite{elmachtoub2022smart} follows this perspective but is slightly different in that break ties in \cref{eq: plug in policy} in a worst-case way. See also \Cref{sec: spo+}.} However, unlike ETO, IERM chooses $\hat f$ according to the downstream decision-making objective rather than the statistical accuracy (\eg, sum of squared errors) of the intermediate estimate. In this sense, the IERM approach is integrated and end-to-end.

Recent literature demonstrates that such end-to-end approaches often outperform ETO \citep{elmachtoub2023estimate}, with significant benefits in the \textit{misspecified} setting -- that is, when the function class $\F$ fails to contain the true expected cost function $f^*$, and the policy class $\Pi_\F$ does not include the globally optimal decision policy $\pi_{f^*}$ \citep[\eg,][]{elmachtoub2022smart,elmachtoub2023estimate}. In this case, the IERM approach still targets the best cost-minimizing policy within the plug-in policy class $\Pi_\F$, but the ETO approach may target a severely suboptimal policy without taking into account the downstream decision-making task. 
Notably, because the policy class $\Pi_\F$ is itself defined by the optimization problem in \cref{eq: plug in policy}, 
the IERM formulation in \cref{eq: IERM-full} involves a challenging bi-level optimization that is generally non-convex and non-smooth. Fortunately, the recent literature has proposed a range of practical computational approximations that optimize more tractable surrogate loss functions \citep[e.g.,][]{elmachtoub2022smart,huang2024learning,tang2024pyepo}.

Nonetheless, the IERM formulation in \cref{eq: IERM-full} requires observing the full feedback $Y_i$ in the data. 
In the next subsection, we turn to the challenges of developing an effective end-to-end approach for the more difficult bandit-feedback setting.

\subsection{Our Main Problem: CLO with Bandit Feedback}\label{sec:ierm}

The primary setting of our paper is the \emph{bandit-feedback} setting, where for each data point $i$, the decision-maker can only observe the total cost $C_i=Z_i\tr Y_i$ for the specific decision $Z_i$ but not the coefficient vector $Y_i$. In other words, the decision-maker can only  observe the cost for  a chosen decision but not the counterfactual cost of any alternative decision, much like in bandit problems \citep{lattimore2020bandit}. In the stochastic shortest path example, this is equivalent to knowing only the total travel time for the specific path that was taken.

Formally, assume we have \edit{an offline dataset consisting of} $n$ data points $\D_B = \braces{\prns{X_1, Z_1, C_1}, \dots, \prns{X_n, Z_n, C_n}}$ that are independent draws from a distribution on $(X,Z,C)$ generated in the following way. We first draw $(X, Z, Y)$ where the $(X,Y)$ distribution is as in the full feedback setting \edit{and $Z$ is generated according to some historical decision policies for data collection}. Then we set $C=Z\tr Y$ and omit the full $Y$. 

The learning task is to use these data $\D_B$ to learn a data-driven policy $\hat{\pi}_{\text{B}}: \mathcal{X}\to\Z$ that has a low \emph{regret}. Specifically, the regret of a given policy $\pi$ measures the expected sub-optimality of $\pi$ compared to the optimal policy $\pi_{f^*}$:
\begin{align}\textstyle
    \text{Reg}(\pi) = \expect_X\bracks{f^*(X)^\top {\pi(X) - \min_{z\in\Z}f^*(X)^\top z}} = V(\pi)-V(\pi_{f^*}),
\end{align}
where we recall that $V(\pi)=\expect_X\bracks{f^*(X)\tr \pi (X)}$ is the expected cost of policy $\pi$. A low-regret $\hat{\pi}_{\text{B}}$ implies the expected cost of $\hat{\pi}_{\text{B}}$ is close to the best achievable expected cost and is thus desirable.

Given the advantages of the end-to-end IERM approach in the full-feedback setting, a natural question is whether it can be extended to the bandit-feedback case. 
In particular, one can still consider the plug-in policy class $\Pi_\F$ (cf. \cref{eq: best in class policy full,eq: plug in policy}), and it is natural to again target the best policy within this class. However, doing so in the bandit-feedback setting introduces three major challenges.
The first challenge is \textit{evaluation}. In the bandit setting, we do not observe the full cost vector $Y$. Consequently, we can no longer use a simple sample average like that from \cref{eq: IERM-full} to estimate the expected cost $V(\pi)$ of a policy $\pi$, let alone optimize this estimate to find the best policy.
The second challenge is \textit{theoretical}. Even with a valid estimator for $V(\pi)$, it is unclear whether the policy that minimizes this estimate is guaranteed to perform well under misspecification. This is a crucial question, as a key advantage of IERM is its robustness in misspecified settings. However, the theory for contextual linear optimization under misspecification is explored only recently and has so far been restricted to the full-feedback setting \citep{huang2024decision,bennounacontextual}.
The third challenge is \textit{computational}. The optimization over the induced policy class $\Pi_\F$ remains difficult under bandit feedback, necessitating the development of practical and efficient computational surrogates.

We will address each of these challenges in turn: we tackle the evaluation challenge in \cref{sec: ierm-bandit}, the theoretical challenge in \cref{section: theory}, and the computational challenge in \cref{sec: spo+}.

\section{Induced Empirical Risk Minimization with Bandit Feedback}\label{sec: ierm-bandit}

In this section, we develop an integrated, end-to-end framework for learning CLO under bandit feedback. Our method extends the core principles of the full-feedback IERM in \cref{eq: IERM-full} to accommodate the limited information available in the bandit setting. At the end of this section, we unify the bandit-feedback and full-feedback IERM into a general framework, which underpins the theoretical and computational developments in \cref{section: theory,sec: spo+}.

Our IERM approach for the bandit setting is based on the same principle as the full-feedback case in \cref{eq: IERM-full}: directly minimizing an empirical estimate of the expected policy cost $V(\pi)$ over the induced policy class $\Pi_\F$. 
The key challenge, however, is that we do not observe the full cost vector $Y_i$. Therefore, we can no longer use the sample average estimator $\sum_{i=1}^n Y_i^\top \pi(X_i)/n$ to evaluate the expected policy cost. Instead, we must find alternative ways to evaluate $V(\pi)$ based on the observed bandit-feedback data $\D_B = \braces{\prns{X_1, Z_1, C_1}, \dots, \prns{X_n, Z_n, C_n}}$. 

\subsection{Population Evaluation of the Expected Policy Cost}

To evaluate the expected policy cost, we begin by analyzing how it can be identified from the population distribution that generates the bandit-feedback data. 
This resolves the evaluation problem at the population level and lays the foundation for finite-sample evaluation.

We make two basic assumptions about how the historical decisions were generated. 

\begin{assumption}\label{assump: dgp}
The data generating process satisfies the following two properties:
\begin{enumerate}
    \item (Ignorability) $\mathbb{E}\bracks{C \mid Z, X} = Z^\top f^*(X)$ (which holds, for example, if decisions are independent of coefficients given the context, \ie, $Z \indep Y \mid X$).
    \item (Coverage) 
    $\inf_{z\in\mathrm{span}(\mathcal{Z}):\norm{z}=1} z^\top \Sigma^*(X)z > 0$ almost surely, where $\Sigma^*(X) = \expect\bracks{ZZ^\top \mid X}$.
\end{enumerate}
\end{assumption}

The ignorability condition is a common assumption that plays a fundamental role in learning with bandit feedback and causal inference \citep[\eg,][]{imbens2015causal,athey2021policy,chernozhukov2019semi}. It requires that the assignment of the historical decisions $Z$ does not depend on any unobserved factors potentially dependent on $Y$. If this condition fails—\ie, if there are unobserved factors related to both $Z$ and $Y$—then we still have $\mathbb{E}\bracks{C \mid Z, X} = Z^\top \mathbb{E}\bracks{Y \mid Z, X}$, but $\mathbb{E}\bracks{Y \mid Z, X}$ may arbitrarily deviate from the function of interest $f^*(X) = \mathbb{E}\bracks{Y \mid  X}$ due to correlation between $Z$ and $Y$ induced by those unobserved factors.
In this case, it is inherently impossible to faithfully learn the target function $f^*$ from bandit-feedback data alone. This necessitates the ignorability assumption.

The coverage condition is an analogue of the overlap or positivity assumption in causal inference \citep{imbens2015causal}. 
It ensures that, for any given context $X$, the historical decisions are sufficiently diverse to explore all directions in the linear span of the constraint set $\mathcal{Z}$. 
Equivalently, it means that the conditional Gram matrix, $\Sigma^*(X)$, is non-singular on $\mathrm{span}(\mathcal{Z})$. 
This requirement is weaker than demanding exploration of the entire space $\mathbb{R}^d$ or that $\Sigma^*(X)$ be full-rank on $\mathbb{R}^d$.
Indeed, in many applications $\mathrm{span}(\mathcal{Z})$ is much smaller than $\mathbb{R}^d$, so this weaker assumption suffices (see the shortest-path example in \cref{sec: numerical}). Without coverage, the bandit-feedback data may provide no information about the decision cost of some feasible policies.

To characterize the expected policy cost $V(\pi)$ in terms of the bandit-feedback data distribution, we rely on two functions. First, the conditional Gram matrix, $\Sigma^*(X) = \expect\bracks{ZZ^\top \mid X}$, which depends only on the distribution of the observed variables \((Z,X)\) and is therefore estimable. Second, let $\tilde f^*$ be any solution to the least-squares problem
\begin{align}\label{eq: f-nuisance}
    \tilde f^* \in \mathcal{F}^* \coloneqq \argmin_{f: \mathcal{X} \to \R{}} \mathbb{E}[(C - Z^\top f(X))^2].    
\end{align}
Notably, \cref{eq: f-nuisance} involves only the distribution of the observed data $(X, Z, C)$, so it is possible to estimate one of its solutions from data.
When \(\Sigma^*(X)\) is almost surely non-singular, \cref{eq: f-nuisance} has an almost surely unique solution \(\tilde f^* = f^*\). When \(\Sigma^*(X)\) is rank-deficient, multiple solutions may exist (and may differ from \(f^*\)), but any such solution suffices for our purposes. We discuss estimation of \(\tilde f^*\) and \(\Sigma^*\) in \Cref{remark:f,remark:sigma0}. 
We refer to these functions as \emph{nuisance functions}, following the existing literature on off-policy evaluation and learning in contextual bandits and reinforcement learning \citep{uehara2022review}, because they are not directly used for decision-making but serve merely as intermediaries for evaluating the policy value.

Our key idea is to replace the unobserved coefficient vector $Y$ with a score function, $\theta(X,Z,C; \tilde f^*, \Sigma^*)$, that depends only on the observed data $(X, Z, C)$ and estimable nuisance functions $\tilde f^*, \Sigma^*$. The score function must enable the evaluation of expected policy cost, just as the unobserved $Y$ would: for any policy $\pi:\mathcal{X}\to \Z$, 
\begin{align} \label{eq: theta condition}
    \mathbb{E}_{X,Z,C}\bracks{\theta(X,Z,C; \tilde f^*, \Sigma^*)^\top \pi(X)} = V(\pi).
\end{align}
\cref{prop: unbiased} summarizes several possible choices of the score function. 
\begin{lemma}\label{prop: unbiased}
  Under \cref{assump: dgp}, 
  the following choices for the score function $\theta(X,Z,C; f, \Sigma)$, when evaluated at the true nuisance functions $f = \tilde f^* \in \mathcal{F}^*$ and $\Sigma=\Sigma^*$, all satisfy the policy evaluation identity in \cref{eq: theta condition}:
  \begin{enumerate}
  \item (Direct Method) $\theta_{\text{DM}}(X,Z,C; f,\Sigma) = f(X)$.
  \item (Inverse Spectral Weighting) $\theta_{\text{ISW}}(X,Z,C; f,\Sigma) = \Sigma^{\dagger}(X)ZC$.
      \item (Doubly Robust) $\theta_{\text{DR}}(X,Z,C; f,\Sigma) = f(X) + \Sigma^{\dagger}(X) Z (C - Z^\top f(X))$.
  \end{enumerate}
Here, $\Sigma^\dagger(X)$ denotes the Moore–Penrose pseudo-inverse of matrix $\Sigma(X)$.
\end{lemma}

\Cref{prop: unbiased} provides three score functions for evaluating the expected cost of any decision policy. 
In particular, the ``Direct Method'' (DM) score uses only the least-squares nuisance \(\tilde f^*\); the ``Inverse Spectral Weighting'' (ISW) score uses only the conditional Gram matrix \(\Sigma^*\); and the ``Doubly Robust'' (DR) score combines both nuisance components. 
While a similar DR score was considered by \cite{chernozhukov2019semi}, our formulation is more general. It explicitly handles rank-deficient design matrices $\Sigma^*$ via the pseudo-inverse and accommodates cases where the least-squares problem \cref{eq: f-nuisance} for $\tilde f^*$ may have non-unique solutions.
This generalization is practically important for applications like our shortest path experiment (\cref{sec: numerical}), where the decision set's linear span is a low-dimensional subspace, leading to a rank-deficient $\Sigma^*$.

\subsection{Sample Estimation of the Expected Policy Cost}

Since the data-generating distribution and the true nuisance functions $\tilde f^*$ and $\Sigma^*$ are unknown, the objective in \cref{eq: theta condition} must be estimated from data. To this end, we adopt the $K$-fold cross-fitting procedure in \cite{chernozhukov2018double}, which is common in the offline contextual bandit literature \citep{athey2021policy,zhou2018offline}. For simplicity, we assume that $n/K$ is an integer.

The procedure begins by splitting the dataset $\mathcal{D}_B$ into $K \ge 2$ equal-sized subsamples (or folds), denoted as $\mathcal{D}_B^{(1)}, \dots, \mathcal{D}_B^{(K)}$. For each fold $k\in[K]$, we define $\Dcal^{(-k)}_B = \Dcal_B\setminus\Dcal_B^{(k)}$ as the dataset with the $k$-th fold held out. The data $\mathcal{D}_B^{(-k)}$ can be used as an auxiliary sample to train nuisance function estimators, $\hat{f}^{(-k)}$ and $\hat{\Sigma}^{(-k)}$, for functions $\tilde f^*$ and $\Sigma^*$ respectively. These estimates are then used on the held-out fold $\mathcal{D}_B^{(k)}$ to form an estimate of the policy cost $V(\pi)$:
\begin{align*} 
    \textstyle 
    \frac{K}{n} \sum_{i\in \mathcal{D}_B^{(k)}} \theta\prns{X_i, Z_i, C_i;\hat{f}^{(-k)}, \hat{\Sigma}^{(-k)}}^\top \pi(X_i).
\end{align*}
This process is repeated for all $K$ folds. The final IERM policy, $\hat{\pi}_B$, is found by minimizing the average of these $K$ cross-fitted policy cost estimates over the induced policy class $\Pi_\F$:
\begin{align} \label{eq: pi ierm}
    \textstyle 
    \hat{\pi}_B \in \arg\min_{\pi\in \Pi_\F} \frac{1}{n} \sum_{k\in[K]} \sum_{i\in \mathcal{D}_B^{(k)}} \theta\prns{X_i, Z_i, C_i;\hat{f}^{(-k)}, \hat{\Sigma}^{(-k)}}^\top \pi(X_i),
\end{align}
where the subscript $B$ stands for ``Bandit'' feedback. 
As a concrete example, substituting the doubly robust (DR) score $\theta_{{DR}}$ into this objective gives the following policy optimization problem: 
\begin{align}\label{eq: pi ierm dr}
    \hat\pi_{\text{DR-B}} \in \arg\min_{\pi\in \Pi_\F} \frac{1}{n} \sum_{k\in[K]} \sum_{i\in \mathcal{D}_B^{(k)}}\prns{\hat{f}^{(-k)}(X_i) + \prns{\hat{\Sigma}^{(-k)}(X_i)}^{\dagger} Z_i \prns{C_i - Z_i^\top \hat{f}^{(-k)}(X_i)} }^\top \pi(X_i).
\end{align}

\begin{remark}[Estimation of $\tilde f^*$]\label{remark:f}
The nuisance function $\tilde f^*$ is a solution to the least squares problem in \cref{eq: f-nuisance}. We can replace the expected objective by its sample analogue on each dataset $\mathcal{D}_B^{(-k)}$, and minimize the sample loss over a proper function class  $\F^{\text{N}}$ (\eg, linear functions, tree models, neural networks) to obtain the corresponding nuisance estimator:  
\begin{equation} \label{eq: nuisance f}
    \textstyle 
 \hat{f}^{(-k)}\in \arg\min_{f\in \F^{\text{N}}} \frac{1}{|\mathcal{D}_B^{(-k)}|}\sum_{i\in \mathcal{D}_B^{(-k)}}  \prns{C_i - Z_i^\top f(X_i)}^2.
\end{equation}
Note that the nuisance function class  $\F^{\text{N}}$ can also be viewed as a model class for $f^*$, since $f^*$ is a canonical solution to the least squares problem in \cref{eq: f-nuisance}. However, it is different from the class $\F$ that induces the plug-in policy class $\Pi_\F$ in \cref{eq: best in class policy full}, as $\F^{\text{N}}$ is used to construct nuisance estimators for evaluating induced policies.
 In practice, we do not need to use the same class for $\F$ and $\F^{\text{N}}$. In fact, it might be desirable to use a more flexible function class for $\F^{\text{N}}$ to make sure it is well-specified for accurate policy evaluation, and a simpler class for $\F$ to make the policy optimization more tractable. We numerically test out different choices of $\F$ and $\F^{\text{N}}$ in \cref{sec: numerical}.
\end{remark}

\begin{remark}[Estimation of $\Sigma^*$]\label{remark:sigma0}
In many applications, the conditional Gram matrix $\Sigma^*(X) = \mathbb{E}[ZZ^\top \mid X]$ can be computed exactly and need not be estimated from data. This is often the case when the logging policy that generated the data is known (\eg, from a prior A/B test or a previously deployed policy) so that the $Z \mid X$ conditional distribution is known, which is analogous to the setting of known propensity scores in the causal inference and off-policy contextual bandit literature \citep[e.g.,][]{swaminathan2015batch,si2023distributionally}.
   When the logging policy is not available, there are multiple ways to estimate $\Sigma^*$. For example, 
   \cite{chernozhukov2019semi} suggest estimating $\Sigma^*$ by running a multi-task regression  for all $(j, k)$ entries to the matrix over some appropriate hypothesis spaces\footnote{To ensure a positive semi-definite estimator, we may posit each hypothesis $\Sigma$ to be the outer product of some matrix-valued hypothesis of appropriate dimension.}  $\mathcal{S}_{jk}$: $ \hat{\Sigma}^{(-k)} = \arg\min_{\Sigma_{11} \in \mathcal{S}_{11}, \dots, \Sigma_{dd}  \in \mathcal{S}_{dd}} \sum_{i\in 
\mathcal{D}_B^{(-k)}}\|Z_iZ_i^\top - \Sigma(X_i)\|_{Fro}^2$. 
Alternatively, in settings with finitely many feasible decisions $z_1,\dots,z_m$ (\eg, feasible paths in stochastic shortest-path problems; see \cref{sec: numerical}; or, more generally, the vertices in $\mathcal{Z}^{\angle}$), one can first estimate the propensity scores $e^*(z \mid X)=\mathbb{P}(Z=z \mid X)$ for $z = z_1, \dots, z_m$ using a suitable estimator $\hat e^{(-k)}(z \mid X)$ and then estimate $\Sigma^*(X) = \sum_{j=1}^m z_jz_j^\top e^*(z_j \mid X)$ by $\hat \Sigma^{(-k)}(X) = \sum_{j=1}^m z_jz_j^\top \hat e^{(-k)}(z_j \mid X)$. 
\end{remark}

\subsection{A General IERM Framework}
\label{sec: general ierm}
We conclude this section by introducing a generic IERM framework with $K$-fold cross-fitting, which includes the full-feedback setting in \Cref{sec: full-feedback} and the bandit-feedback setting in this section as special cases. In fact, it can also accommodate the semi-bandit-feedback setting in \Cref{sec: semi}. 

Let $\mathcal{D}=\braces{\Data_i}_{i=1}^n$ be an observed, i.i.d. historical dataset, which we partition into $K$ disjoint subsets, $\mathcal{D}^{(1)},\dots, \Dcal^{(K)}$, each of size $n/K$. We consider the existence of a general score function that can enable the evaluation of policy cost.

\begin{assumption}\label{assump: score general}
    Suppose there exists a score function $\theta(\Data;\nuisance)$ with nuisance $\nu = \nu^*$ such that 
    \begin{align*}
        \textstyle 
        \expect_\Data \bracks{\theta(\Data; \nuisance^*)^\top \pi(X)} = V(\pi), ~~ \forall \text{ fixed policy } \pi: \mathcal{X} \to \Z.
    \end{align*}
\end{assumption}

We again use cross-fitting to estimate the unknown nuisance $\nuisance^*$, and denote by $\hat{\nuisance}^{(-k)}$ the estimator trained on $\mathcal{D}^{(-k)} = \Dcal \setminus \Dcal^{(k)}$. We then obtain the IERM policy over the induced policy class $\Pi_{\F}$ by solving the following optimization problem: 
\begin{align} \label{eq: general ierm}
    \textstyle 
    \hat{\pi} \in \arg\min_{\pi\in \Pi_\F} \frac{1}{n} \sum_{k\in[K]} \sum_{i\in \Dcal^{(k)}} \theta\prns{\Data_i;\hat{\nuisance}^{(-k)}}^\top \pi(X_i).
\end{align}

This formulation provides a unified view of different settings. 
The full-feedback setting is a special case with observation $\Data = (X, Y)$, no nuisance parameter ($\nuisance^*=\emptyset$), and a simple score function $\theta(\Data;\nuisance) = Y$.
The bandit feedback setting corresponds to $\Data = (X, Z,C)$ with nuisances $\nuisance^*=(\tilde f^*, \Sigma^*)$, and when \Cref{assump: dgp} holds,  the score function $\theta$ has three possible forms as detailed in \cref{prop: unbiased}. The instantiation in the semi-bandit-feedback setting will be presented in \Cref{sec: semi}.

For generality, our subsequent analysis of theoretical guarantees in \cref{section: theory} and computationally efficient surrogates in \cref{sec: spo+} will focus on the general IERM framework.

\section{Theoretical Analysis} \label{section: theory}

In this section, we provide a theoretical regret analysis for the general IERM framework in \cref{sec: general ierm}. 
We begin by establishing a general fast-rate regret bound.
Crucially, our analysis allows for model misspecification, where the induced policy class $\Pi_\F$ may not contain the globally optimal policy $\pi_{f^*}$. 
We then instantiate this general theorem for the bandit-feedback and full-feedback settings to derive specific corollaries. 
The proof of the main theorem is detailed in the final part of this section, as it introduces new analytical tools that may be of independent interest.

\subsection{Preliminaries}
Given an induced policy class $\Pi_\F$, recall that the best-in-class policy $\Tilde{\pi}^*$ (defined in \cref{eq: best in class policy full}) minimizes the expected policy cost within this class. 
As a key departure from many existing theoretical analyses in the literature, we do not assume that the policy class $\Pi_\F$ is well-specified. This means the globally optimal policy $\pi_{f^*}$ may not be contained in $\Pi_\F$, in which case $\Tilde{\pi}^*$ differs from $\pi_{f^*}$. The regret of this best-in-class policy, $\text{Reg}(\Tilde{\pi}^*)$, quantifies the degree of misspecification inherent to the induced policy class $\Pi_\F$: 
\begin{align*}
    \textstyle 
    \text{Reg}(\Tilde{\pi}^*) = V(\Tilde{\pi}^*) - V(\pi_{f^*}) = \min_{\pi\in\Pi_\F}V(\pi) - \min_{\pi: \mathcal{X} \to \Z}V(\pi). 
\end{align*}

Throughout \cref{section: theory}, we impose two key assumptions.
The first, \cref{assumption: nuisance}, formalizes the requirement that the nuisance function estimator is well-behaved. It requires that the resulting score function is bounded and that the nuisance estimate converges to its true value at a specific rate.

\begin{assumption}[Nuisance Estimation] \label{assumption: nuisance}
The nuisance estimator $\hat{\nuisance}$, trained on an i.i.d sample of size $n$, satisfies $\|\theta(\Data;\hat{\nuisance})\| \le \Theta$ almost surely. Furthermore, for any $\delta\in(0,1)$ and $\pi\in \Pi_\F$, with probability at least $1-\delta$, the following holds:
\begin{align*} \textstyle
    \mathbb{E}_\Data\bracks{\prns{\theta(\Data; \nuisance^*)- \theta(\Data;\hat{\nuisance})}^\top \prns{\pi(X)- \Tilde{\pi}^*(X)}  }
\le \operatorname{Rate}^{\textbf{N}}(n, \delta).
\end{align*}
\end{assumption}
This assumption is trivial for the full-feedback setting, where $\nuisance^*=\hat{\nuisance}=\emptyset$ and $\operatorname{Rate}^{\textbf{N}}(n, \delta)=0$. For the bandit-feedback setting, in \cref{sec: theory bandit} we will relate  $\operatorname{Rate}^{\textbf{N}}(n, \delta)$ to the estimation errors of the individual nuisances, $\tilde f^*$ and $\Sigma^*$, for various choices of the score function $\theta$.

Our second assumption, \cref{assumption: margin}, is a margin condition that restricts the probability of small sub-optimality gaps in the CLO problem instance, which enables faster regret rates. Such conditions originated from the binary classification literature \citep{tsybakov2004optimal,audibert2007fast}; we adopt the recent extension to contextual linear optimization by \cite{hu2022fast}.

\begin{assumption}[Margin Condition]\label{assumption: margin}
Let $\mathcal{Z}^*(x) = \arg\min_{z\in \mathcal{Z}^\angle} f^*(x)^\top z$ define the set of optimal decisions. Let the sub-optimality gap be $\Delta(x) = \inf_{z\in\mathcal{Z}^\angle\setminus \mathcal{Z}^*(x)} f^*(x)^\top z - \inf_{z\in\mathcal{Z}^\angle} f^*(x)^\top z$ if $\mathcal{Z}^*(x) \ne \mathcal{Z}^\angle$, and $\Delta(x) = 0$ otherwise. Assume for some constants $\alpha, \gamma\ge 0$, 
\[
\pr_X(0<\Delta(X)\le \delta )\le (\gamma\delta/B)^{\alpha} \quad \forall \delta>0.
\]
\end{assumption}
Lemmas 4 and 5 in \cite{hu2025fast} establish that \cref{assumption: margin} typically holds with $\alpha = 1$ when $f^*$ is sufficiently well-behaved and $X$ is continuous, and with $\alpha = \infty$ when $X$ is discrete. Moreover, any CLO instance trivially satisfies the case $\alpha = 0$. Intuitively, larger values of $\alpha$ indicate that the sub-optimality gap between the best and second-best decisions tends to be larger across more contexts, making it easier to identify the optimal decision. We will show that a larger $\alpha$ leads to faster regret rates.

Finally, for $k\in[K]$, we define the function class $\mathcal{G}^{(-k)}$ as:
\begin{align} \label{eq: function class G}
    \mathcal{G}^{(-k)} = \braces{\data \rightarrow \frac{\theta\prns{\data;\hat{\nuisance}^{(-k)}}^\top \prns{\pi(x) - \Tilde{\pi}^*(x)} \rho}{2B\Theta} :  \pi\in\Pi_\F, \rho\in[0,1]}.
\end{align}
Our theoretical bounds involve the critical radius of this function class. The critical radius is a generic measure of function complexity \citep{wainwright2019high} that characterizes many function classes, both parametric (for which the critical radius is generally $O(1/\sqrt{n})$) and non-parametric (for which the critical radius generally decays more slowly). We derive bounds on critial radii in \cref{sec: critical radius}.

For any class of functions $\Gcal$ mapping from the data space to $\mathbb{R}$, its local Rademacher complexity at radius $r>0$ and sample size $n$ is defined as:
\begin{align*}
    \textstyle 
    \Rcal_n(\Gcal, r) = \mathbb{E}\bracks{\sup_{g\in \mathcal{G}, \norm{g}_2\le r} \abs{\frac{1}{n} \sum_{i=1}^n \epsilon_i g(\Data_i)}},
\end{align*}
where $\{\epsilon_i\}$ are i.i.d. Rademacher random variables (\ie, $\epsilon_i$ takes $\pm 1$ with equal probability) and $\norm{g}_2 = \sqrt{\mathbb{E}_\Data [g^2(\Data)]}$. Finally, the critical radius of $\Gcal$ at sample size $n$ is any $\tilde r>0$ that satisfies $\Rcal_{n}\prns{\Gcal, \tilde{r}} \le \tilde{r}^2$. Because $\Gcal^{(-k)}$ is star-shaped ($g\in\Gcal^{(-k)},\rho\in[0,1]\implies\rho g\in\Gcal^{(-k)}$), we have that $\Rcal_{n}(\Gcal^{(-k)}, r)/r$ is non-increasing, so there exists a solution and any larger value is a critical radius.

\subsection{Main Theorem}\label{sec:main-thm}

We now state our main theorem, an upper bound on the regret of the IERM policy $\hat \pi$ in \cref{eq: general ierm}.

\begin{theorem}\label{thm: doubly robust fast rates}
Suppose that \cref{assump: score general,assumption: nuisance,assumption: margin} hold, and that the set of optimal decisions $\mathcal{Z}^*(X)$ (defined in \cref{assumption: margin}) is almost surely a singleton. Let $\tilde{r} > 0$ bound above the critical radii of $\Gcal^{(-k)}$ at sample size $n/K$ for all $k\in[K]$ as well as satisfy
$3n\tilde{r}^2/(64K) \ge \log \log_2(1/\tilde{r})$
and
$2\exp(-3n\tilde{r}^2/(64K)) \le \delta/(2K)$.
Then, there exists a positive constant $\tilde{C}(\alpha,\gamma)$ such that with probability at least $1-\delta$, we have 
\begin{align}
    \text{Reg}(\hat{\pi}) \le & 2\text{Reg}(\tilde{\pi}^*) 
    + 2\text{Rate}^{\textbf{N}}\prns{\frac{(K-1)n}{K}, \frac{\delta}{2K}} 
    + B \prns{12\Theta \sqrt{\tilde{C}(\alpha, \gamma)} \tilde{r}}^{\frac{2\alpha+2}{\alpha + 2}}  \notag
    \\& + 24 B\Theta \prns{ \sqrt{\tilde{C}(\alpha, \gamma)}\prns{\frac{\text{Reg}(\tilde{\pi}^*)}{B}}^{\frac{\alpha}{2(1+\alpha)}}\tilde{r} + \tilde{r}^2}  . \label{eq: upper bound}
\end{align}
\end{theorem}

We note that the fourth term on the right-hand side of \cref{eq: upper bound} is composed of higher-order components. Specifically, its first component, $O\prns{\text{Reg}(\tilde{\pi}^*)^{\frac{\alpha}{2(1+\alpha)}} \tilde{r}}$, is dominated by $O\prns{\text{Reg}(\tilde{\pi}^*) + \tilde{r}^{\frac{2\alpha+2}{\alpha+2}}}$. Its second component, $O(\tilde r^2)$, is dominated by $O\prns{\tilde r^{\frac{2\alpha+2}{\alpha+2}}}$ because $\tilde r=o(1)$ generally holds. 
This fourth term is therefore dominated by the first three terms. These three terms, which we will focus on in the rest of our analysis,  characterize 
three primary sources of error. 
\begin{enumerate}
    \item \textit{Misspecification Error.} The first term, $\text{Reg}(\tilde{\pi}^*)$, is the misspecification error, which quantifies the sub-optimality of the best-in-class policy $\tilde{\pi}^*$ relative to the globally optimal policy $\pi_{f^*}$. This term is zero if the policy class $\Pi_\F$ is well-specified (\ie, if $\pi_{f^*} \in \Pi_\F$). It is natural to expect a bigger decision regret when the misspecification error $\text{Reg}\prns{\Tilde{\pi}^*}$ is higher.

    \item \textit{Nuisance Estimation Error.} The second term, $\operatorname{Rate}^{\textbf{N}}((K-1)n/K, \delta/(2K))$, results from estimation errors of the $K$ cross-fitted nuisance estimators  for $\nuisance^*$, each trained on a subsample of size $(K-1)n/K$. Similar nuisance error terms also appear in the literature on offline contextual bandits \citep{athey2021policy,zhou2018offline,chernozhukov2018double} and other learning problems with nuisance components \citep{foster2023orthogonal}. 
    The behavior of this error term is what primarily distinguishes the different feedback settings. We will further characterize this term for the bandit-feedback setting in \cref{sec: theory bandit} and the semi-bandit-feedback setting in \cref{sec: semi}.

\item \textit{Statistical Complexity.} The third term is governed by the critical radius $\tilde{r}$, which characterizes the complexity of each function class $\Gcal^{(-k)}$, when evaluated on the  subsample $\mathcal{D}^{(k)}$ for $k\in[K]$. This is a standard function class complexity measure in statistics and machine learning \citep{wainwright2019high}.  As a concrete example, we show in \cref{sec: critical radius} that if each $\Gcal^{(-k)}$ has a finite VC-subgraph dimension,
  then $\tilde{r}$ is of the order $\tilde{O}(\sqrt{K/n})$. However, stating our main bound directly in terms of $\tilde{r}$ provides greater generality and accommodates wider function classes. A key consequence of the margin condition (\cref{assumption: margin}) is that, unlike generic bounds that  scale linearly with the critical radius, our regret bound scales with a polynomial of $\tilde{r}$, enabling potentially faster convergence.
    
\end{enumerate}

\begin{remark} \label{remark: tighter bound}
The constant coefficients in the regret upper bound can be tightened when $2(1+\alpha)/\alpha$ is an integer. This condition notably includes the important case of $\alpha = 1$, which holds for well-behaved functions $f^*$ and continuous contexts $X$ \citep{hu2025fast}. In such scenarios, the bound improves to:
\begin{align}
    \text{Reg}(\hat{\pi}) \le & \text{Reg}\prns{\tilde{\pi}^*} + 
    \frac{2\alpha+2}{\alpha+2}\text{Rate}^{\textbf{N}}\prns{\frac{(K-1)n}{K}, \frac{\delta}{2K}}  +
    B \prns{12\Theta \sqrt{\tilde{C}(\alpha, \gamma)} \tilde{r}}^{\frac{2\alpha+2}{\alpha + 2}}  \notag
    \\ & + \frac{24(\alpha+1)}{\alpha+2} B\Theta \prns{ \sqrt{\tilde{C}(\alpha, \gamma)}\prns{\frac{\text{Reg}(\tilde{\pi}^*)}{B} }^{\frac{\alpha}{2(1+\alpha)}}\tilde{r} + \tilde{r}^2}. \label{eq: tighter bound}
\end{align}
Notably, the coefficient on the misspecification error, $\text{Reg}(\tilde{\pi}^*)$, is reduced from $2$ to $1$, which we believe is the optimal constant. The derivation of this tighter bound is presented in \cref{sec: main proof} following the proof of our main theorem.
\end{remark}

\subsubsection{Critical Radius for VC-Subgraph Classes} \label{sec: critical radius}

In this section, we present a concrete example of computing the critical radius. Specifically, we focus on function classes with finite VC-subgraph dimension, a widely used complexity measure in the literature \citep{van1996weak}.

\begin{definition}[VC-Subgraph Dimension] 
The VC-subgraph dimension (also known as the pseudo-dimension) of a  function class $\Gcal$ is defined as the VC dimension of its subgraph class: 
\[
    \braces{(x,t) \mapsto \ind[t \le g(x)] : g \in \Gcal}.
\]
\end{definition}

The following proposition establishes that the critical radius $\tilde{r}$ is of order $\tilde{O}(\sqrt{K\eta/n})$ when each function class $\Gcal^{(-k)}$ has a finite VC-subgraph dimension $\eta$ almost surely.

\begin{proposition} \label{lemma: critical radius}
Suppose $\Gcal^{(-k)}$ has VC-subgraph dimension $\eta$ almost surely. Then, for any $\delta \in (0,1)$, there exists a universal constant $\Tilde{C}>0$ such that
\begin{align}\label{eq: critical radius}\textstyle
    \Tilde{r} = \Tilde{C} \sqrt{\frac{\eta K \log(n+1) + K \log(8/\delta)}{n}},
\end{align}
satisfies the inequalities $\Rcal_{n/K}\!\prns{\Gcal^{(-k)}, \Tilde{r}}\le \Tilde{r}^2$, 
$3n \Tilde{r}^2 / (64K) \ge \log \log_2(1/\Tilde{r})$, 
and $2\exp\!\prns{-3 n \Tilde{r}^2/(64K)} \le \delta/(2K)$ in \Cref{thm: doubly robust fast rates}. 
\end{proposition}

In the contextual linear optimization (CLO) literature, a more common approach is to impose complexity measures directly on $\Pi_\F$. Since $\Pi_\F$ represents a class of multi-class classifiers (mapping each $x \in \X$ to a vertix point in $\Z^\angle$), existing theoretical analyses have focused on its Natarajan dimension \citep{hu2022fast,el2023generalization}, which generalizes the VC-dimension for binary-class settings to multi-class settings \citep[Chap 29]{shalev2014understanding}.

\begin{definition}[Natarajan Dimension]
    The Natarajan dimension of a function class $\Gcal\subseteq [\RR^p\rightarrow \mathcal{S}]$ with co-domain $\mathcal{S}$ is the largest $m$ for which there exist $x_1, \dots, x_m \in \RR^p, s_1\neq s_1', \dots, s_m \neq s'_m \in \mathcal{S}$ such that
    \begin{align*}
        \braces{\prns{\ind[g(x_1) = s_1], \dots,  \ind[g(x_m) = s_m]} : g\in\Gcal , g(x_1)\in \braces{s_1, s_1'}, \dots, g(x_m)\in \braces{s_m, s_m'}} = \braces{0,1}^m.
    \end{align*}
\end{definition}

The next proposition shows that a finite Natarajan dimension for $\Pi_\F$ is sufficient to guarantee that $\Gcal^{(-k)}$ has a finite VC-subgraph dimension almost surely.

\begin{proposition} \label{prop: natarajan}
If $\Pi_\F$ has Natarajan dimension $\ndim(\Pi_\F)$, then for any $k \in [K]$, $\Gcal^{(-k)}$ has VC-subgraph dimension $\tilde{O}\!\prns{\ndim(\Pi_\F)\log\!\abs{\Z^\angle}}$ almost surely.
\end{proposition}
The proof of \cref{prop: natarajan} is novel and deferred to \cref{sec: critical radius proof}. As an intermediate result, we establish a lemma characterizing the VC-subgraph dimension of star-hulls of VC-subgraph function classes, which may be of independent interest.

Finally, combining \cref{prop: natarajan,lemma: critical radius}, we conclude that if $\Pi_\F$ has a finite Natarajan dimension (as assumed in the existing CLO theory literature\footnote{\cite{hu2022fast} Theorem 2 also proves that $\Pi_\F$ has a finite Natarajan dimension when $\F$ has a finite VC-linear-subgraph dimension.}), the critical radius is of order $\tilde{O}(\sqrt{K/n})$. However, we emphasize that our main theorem assumes only a generic critical radius, allowing our results to extend beyond Natarajan or VC-subgraph classes and thereby providing a more general theoretical framework than prior work in the CLO literature.

\subsubsection{Regret Bounds for Bandit-Feedback IERM} \label{sec: theory bandit}

The bandit-feedback setting is a special case of \cref{eq: general ierm} where the data is $\Data = (X,Z,C)$, the nuisance parameter is $\nu^* = (\tilde f^*,\Sigma^*)$, and the score function $\theta$ is one of $\theta_{\text{DM}}$, $\theta_{\text{ISW}}$, or $\theta_{\text{DR}}$ in \Cref{prop: unbiased}. In this section, we discuss the implications of our main result, \cref{thm: doubly robust fast rates}, for this setting. The key is to bound the nuisance estimation error, $\text{Rate}^{\textbf{N}}(n, \delta)$. The following proposition shows how this term is bounded by the estimation errors for the individual nuisance components.

\begin{proposition} \label{example: DR}
Suppose that for any given $\delta \in (0, 1)$, the nuisance estimators $\hat{f},\hat{\Sigma}$ trained on an i.i.d. sample of size $n$ satisfy the following mean-squared error bounds with probability at least $1-\delta$ for some positive sequences $\chi_f(n, \delta)\to 0$, $\chi_\Sigma(n, \delta)\to 0$ as $n \to \infty$:
\begin{align*}
   \mathbb{E}_X\bracks{\norm{\operatorname{Proj}_{\operatorname{span}(\Z)}\prns{\hat{f}(X) - f^*(X)}}^2} \le \chi_f^2(n, \delta), ~~ \mathbb{E}_X\bracks{\norm{\hat{\Sigma}^\dagger(X) - \Sigma^{*\dagger}(X)}_{\text{Fro}}^2}\le \chi^2_\Sigma(n, \delta),
\end{align*}
 where $\operatorname{Proj}_{\operatorname{span}(\Z)}(\cdot)$ denotes the projection onto $\operatorname{span}(\Z)$, $\|\cdot\|_2$ denotes the Euclidean norm and $\|\cdot\|_{\text{Fro}}$ denotes the matrix Frobenius norm. 
Then we have the following bounds on nuisance errors: 
\begin{enumerate}
    \item If $\theta=\theta_{\text{DM}}$, then $\text{Rate}^{\textbf{N}}(n, \delta) = O(\chi_f(n, \delta))$.
    \item If $\theta=\theta_{\text{ISW}}$, then $\text{Rate}^{\textbf{N}}(n,\delta) = O(\chi_\Sigma(n, \delta))$.
    \item If $\theta=\theta_{\text{DR}}$, then $\text{Rate}^{\textbf{N}}(n,\delta) = O(\chi_f(n, \delta)\chi_\Sigma(n, \delta))$.
\end{enumerate}
\end{proposition}

Compared to the DM and ISW scores, the impact of nuisance estimation error on the DR score is of second order, \ie, the product of the two estimation errors instead of the individual error terms. 
Effectively, the nuisance estimator $\hat\Sigma^\dagger$ in the doubly robust score helps reduce the bias introduced by the nuisance estimator $\hat f$, and vice versa. 
This debiasing property echoes the well-known benefit of doubly robust methods in causal inference and offline contextual bandit learning \citep{athey2021policy,chernozhukov2018double,chernozhukov2019semi}. Notably, our condition only requires bounding the \textit{projected} error of the nuisance estimator $\hat{f}$, which handles the setting where $\operatorname{span}(\Z)$ does not cover the entire space $\mathbb{R}^d$, such as the shortest path application in \cref{sec: numerical}. In this case, the solution set $\F^*$ for the least squares problem in \cref{eq: f-nuisance} is non-unique and contains solution $\tilde f^* \neq f^*$. Fortunately, for any $\tilde f^* \in \F^*$, it can be shown that $\tilde f^* - f^*$ is orthogonal to $\operatorname{span}(\Z)$, so the projected error in \cref{example: DR} is invariant to the particular solution targeted by the nuisance estimator $\hat f$.

Taking the doubly robust IERM policy as an example, we can combine our main result in \cref{thm: doubly robust fast rates} with the nuisance bounds in \cref{example: DR} to obtain the following regret bound. 
\begin{corollary}
    Let $\hat\pi_{\text{DR-B}}$ be an  IERM policy that solves \cref{eq: pi ierm dr}, where the cross-fitted nuisance estimators $\hat f^{(-k)}$, $\hat \Sigma^{(-k)}$ all satisfy the assumptions in \Cref{example: DR}. Define the following DR score-based function classes, for $k \in [K]$: 
    \begin{align*}
        \textstyle  
        \mathcal{G}^{(-k)}_{DR-B} = \braces{(x,z,c) \rightarrow \frac{\prns{\hat f^{(-k)}(x) + \prns{\hat \Sigma^{(-k)}}^{\dagger}(x) z \prns{c - z^\top \hat{f}^{(-k)}(x)}}^\top \prns{\pi(x) - \Tilde{\pi}^*(x)} \rho}{2B\Theta} :  \pi\in\Pi_\F, \rho\in[0,1]}.
    \end{align*}
Suppose that $\tilde{r}_{\text{DR}} > 0$ satisfies $\Rcal_{n/K}\!\prns{\Gcal^{(-k)}_{DR}, \Tilde{r}}\le \Tilde{r}^2$ for all $k\in[K]$ almost surely,  
$3n \Tilde{r}^2 / (64K) \ge \log \log_2(1/\Tilde{r})$, 
and $2\exp\!\prns{-3 n \Tilde{r}^2/(64K)} \le \delta/(2K)$. Further, suppose \cref{assump: dgp,assumption: nuisance,assumption: margin} hold, and that the set of optimal actions $\mathcal{Z}^*(X)$ is a singleton almost surely. Then, with probability at least $1-\delta$, we have 
\begin{align*}
    \text{Reg}(\hat\pi_{DR}) = O\prns{\text{Reg}(\tilde{\pi}^*) +\chi_f((K-1)n/K,\delta/(2K)) \cdot\chi_\Sigma((K-1)n/K,\delta/(2K)) + \tilde{r}_{\text{DR}}^{(2\alpha+2)/(\alpha + 2)}}.
\end{align*}
\end{corollary}
Furthermore, if each $\mathcal{G}^{(-k)}_{DR}$ has finite VC-subgraph dimension $\eta$ almost surely, then $\tilde r_{\text{DR}} = O(\sqrt{K\eta/n})$. Substituting this into the regret bound above yields a final regret bound.

\subsubsection{Byproduct: Fast Rates for Full Feedback with Misspecification} \label{section: full feedback theory}

Taking $\Data = (X, Y)$, $\nuisance^*=\hat{\nuisance}=\emptyset$, $\theta(\Data;\nuisance) = Y$, and $K=1$, \cref{eq: general ierm} specializes to IERM in the full-feedback setting.
In this case, since no nuisance function needs to be estimated, the nuisance estimation error term $\text{Rate}^{\textbf{N}}(n, \delta/2)$ is zero. This leads to the following corollary of \cref{thm: doubly robust fast rates}. To the best of our knowledge, this is the first CLO result that establishes a margin-dependent fast rate under potential policy misspecification in the full-feedback setting, which may be of independent interest.

\begin{corollary}\label{cor: full feedback fast rates}
Suppose \cref{assumption: margin} holds and that $\mathcal{Z}^*(X)$ is almost surely a singleton. Define the function class $\Gcal_{\text{F}} = \braces{(x,y) \rightarrow \frac{y^\top \prns{\pi(x) - \tilde{\pi}^*(x)} \rho}{2B} : \pi\in\Pi_\F, \rho \in [0,1]}$. 
Let $\tilde{r}_{\text{F}} > 0$ bound the critical radius of $\Gcal_{\text{F}}$ at sample size $n$ as well as satisfy
$3 n\tilde{r}_{\text{F}}^2/64 \ge \log \log_2(1/\tilde{r}_{\text{F}})$
and
$2\exp(-3 n\tilde{r}_{\text{F}}^2/64) \le \delta/2$. 
Then, with probability at least $1-\delta$, the regret of the full-feedback estimator $\hat{\pi}_{\text{F}}$ in \cref{eq: IERM-full} is bounded by:
\begin{align*}
    \text{Reg}(\hat{\pi}_{\text{F}}) \le  2\text{Reg}(\tilde{\pi}^*) + B \prns{12 \sqrt{\tilde{C}(\alpha, \gamma)} \tilde{r}_{\text{F}}}^{\frac{2\alpha+2}{\alpha + 2}}  + 24 B \prns{ \sqrt{\tilde{C}(\alpha, \gamma)}\prns{\frac{\text{Reg}(\tilde{\pi}^*)}{B}}^{\frac{\alpha}{2(1+\alpha)}}\tilde{r}_{\text{F}} + (\tilde{r}_{\text{F}})^2}.
\end{align*}
\end{corollary}

 When $\Pi_\F$ is a well-specified class with Natarajan dimension $\eta$, we have $\text{Reg}(\tilde{\pi}^*)=0$ and, by \cref{prop: natarajan,lemma: critical radius}, $\tilde{r}_{\text{F}} = \tilde{O}(\sqrt{\eta/n})$. In this case, the bound in \cref{cor: full feedback fast rates} simplifies to $O\prns{(\eta/n)^{(\alpha+1)/(\alpha+2)}}$. This bound interpolates between $O(n^{-1/2})$ and $O(n^{-1})$ depending on the margin parameter $\alpha$, recovering the fast rate results of \cite{hu2022fast}. Importantly, our result additionally quantifies the impact of model misspecification through the $\text{Reg}(\tilde{\pi}^*)$ term.

\subsection{Proof of Main Theorem} \label{sec: main proof}
Now, we sketch the proof of our main theoretical result, \cref{thm: doubly robust fast rates}. 
To the best of our knowledge, this is the first result establishing fast regret rates under the margin condition in the presence of partial feedback and model misspecification.

To simplify notation, we let $\expect_{n_k}$ denote the empirical average over data $\Dcal^{(k)}$: for any function $g$,
\begin{align*}
    \expect_{n_k}(g)
    = \frac{K}{n}\sum_{i\in \Dcal^{(k)}} g(\Data_i).
\end{align*}

\proof{Proof of \cref{thm: doubly robust fast rates}}

First, we can decompose the regret of our learned policy $\hat{\pi}$ as follows: 
\begin{align}
    \text{Reg}\prns{\hat{\pi}} = & V(\hat\pi) - V(\pi_{f^*})\\
    =&  V(\tilde\pi^*) - V(\pi_{f^*}) +V(\hat\pi) - V(\tilde \pi^*) \notag \\
    = & \text{Reg}\prns{\Tilde{\pi}^*} +\expect_\Data\bracks{\theta(\Data; \nuisance^*)^\top \prns{\hat{\pi}(X) - \Tilde{\pi}^*(X)}} \notag \\
    \le & \text{Reg}\prns{\Tilde{\pi}^*} +\expect_\Data\bracks{\theta(\Data; \nuisance^*)^\top \prns{\hat{\pi}(X) - \Tilde{\pi}^*(X)}} 
    -\frac{1}{K}\sum_{k\in[K]} \expect_{n_k} [{\theta({\Data;\hat{\nuisance}^{(-k)}})^\top (\hat{\pi}(X)-\tilde{\pi}^*(X))}] \notag\\
    =  & \underbrace{\text{Reg}\prns{\Tilde{\pi}^*}}_{(i)} 
    + \underbrace{\frac{1}{K} \sum_{k\in[K]} \expect_\Data[{({\theta(\Data; \nuisance^*)- \theta({\Data;\hat{\nuisance}^{(-k)}})})^\top \prns{\hat{\pi}(X)- \Tilde{\pi}^*(X)}  }]}_{(ii)} \notag \\
    & + \underbrace{\frac{1}{K}\sum_{k\in[K]} (\expect_\Data - \expect_{n_k})[{ \theta({\Data; \hat{\nuisance}^{(-k)}})^\top \prns{\hat{\pi}(X) - \Tilde{\pi}^* (X)} }]}_{(iii)} , \label{eq: regret decomp}
\end{align}
where the third equality follows from the definition of and the fact that $V(\pi) = \expect_\Data[\theta(\Data;\nuisance^*)^\top \pi(X)]$ for any policy $\pi$, the inequality follows from the fact that $\hat{\pi}$ minimizes the empirical risk, and the last equality follows from subtracting and adding $\sum_{k\in[K]} \expect_\Data \bracks{ \theta(\Data; \hat{\nuisance}^{(-k)})^\top \prns{\hat{\pi}(X) - \Tilde{\pi}^* (X)} }/K$.  

\cref{eq: regret decomp} shows three sources of errors: (i) the errors incurred by a misspecified  policy class, (ii) the errors incurred by nuisance function estimation, and (iii) the errors incurred by empirical average approximation. By \cref{assumption: nuisance}, with probability at least $1-\delta/2$, we can bound (ii) as
\begin{align*}
   \frac{1}{K} \sum_{k\in[K]} \expect_\Data\bracks{\prns{\theta(\Data; \nuisance^*)- \theta\prns{\Data;\hat{\nuisance}^{(-k)}}}^\top \prns{\hat{\pi}(X)- \Tilde{\pi}^*(X)}  } \le \text{Rate}^{\textbf{N}}\prns{\frac{(K-1)n}{K}, \frac{\delta}{2K}}.
\end{align*}
It remains to control (iii). To this end, we introduce the following lemma, which shows that the uniform deviation between the sample average and true expectation over the class $\mathcal{G}^{(-k)}$ is bounded by the corresponding critical radius with high probability. This lemma adapts Theorem 14.20 in \cite{wainwright2019high} to our setting and follows similar proofs (deferred to \cref{sec: lemma proof}).
\begin{lemma} \label{lemma: uniform bound G2}
Let $\tilde r > 0$ be satisfy the assumptions of \cref{thm: doubly robust fast rates}. 
Then, with probability at least $1-\delta/(2K)$,
\begin{align} \label{eq: event 1}
  \sup_{g\in\mathcal{G}^{(-k)}} \frac{\abs{(\expect_{n_k} - \expect_\Data) g}}{\norm{g}_2 + \Tilde{r}} \le 6 \Tilde{r}.  
\end{align}
\end{lemma}

Assuming the high-probability event in  \cref{eq: event 1} holds, we can bound each component of (iii):
\begin{align}\label{eq:emp-process-n1}
        & (\expect_\Data - \expect_{n_k})[{ \theta({\Data; \hat{\nuisance}^{(-k)}})^\top \prns{\hat{\pi}(X) - \Tilde{\pi}^* (X)} }]  \\
    = & {2B\Theta}(\expect_\Data - \expect_{n_k})[{ \theta({\Data; \hat{\nuisance}^{(-k)}})^\top \prns{\hat{\pi}(X) - \Tilde{\pi}^* (X)} }/(2B\Theta)] \notag \\
    \le & 12 B\Theta \prns{\norm{{ \theta({\Data; \hat{\nuisance}^{(-k)}})^\top \prns{\hat{\pi}(X) - \Tilde{\pi}^* (X)} }/(2B\Theta)}_2 \Tilde{r} + \Tilde{r}^2} \notag \\
    \le & 12 B\Theta \prns{\norm{{ \theta({\Data; \hat{\nuisance}^{(-k)}})^\top \prns{\hat{\pi}(X) - \pi_{f^*} (X)} }/(2B\Theta)}_2\tilde r + \norm{{ \theta({\Data; \hat{\nuisance}^{(-k)}})^\top \prns{\tilde \pi^* - \pi_{f^*} (X)} }/(2B\Theta)}_2\tilde r + \tilde r^2}. \notag
 \end{align}

We then relate the first two terms above to regret using the following lemma, whose proof follows directly from Lemma 1 in \cite{hu2022fast}.
\begin{lemma}\label{lemma: margin bound}
    Suppose \cref{assumption: margin} holds and the optimal decision set $\mathcal{Z}^*(X)$ (defined in \cref{assumption: margin}) is almost surely a singleton. Then there exists a constant $\Tilde{C}(\alpha, \gamma)$ such that for any $\pi\in \Pi_\F$,
    \begin{align*}
        \pr\prns{\pi(X)\ne \pi_{f^*}(X)} \le \Tilde{C}(\alpha, \gamma)\prns{\frac{\text{Reg}(\pi)}{B}}^{\frac{\alpha}{1+\alpha}}.
    \end{align*}
\end{lemma}
By \cref{lemma: margin bound}, for any $\pi\in \Pi_\F$,
\begin{align*}
   & \norm{{ \theta({\Data; \hat{\nuisance}^{(-k)}})^\top \prns{\pi(X) - \pi_{f^*} (X)} }/(2B\Theta)}_2^2\\
   = & \EE\bracks{\prns{{ \theta({\Data; \hat{\nuisance}^{(-k)}})^\top \prns{\pi(X) - \pi_{f^*} (X)} }/(2B\Theta)}^2 \ind\braces{\pi(X) \ne \pi_{f^*}(X)}} \\
    \le & \pr\prns{\pi(X)\ne \pi_{f^*}(X)} 
    \le  \Tilde{C}(\alpha, \gamma)\prns{\frac{\text{Reg}(\pi)}{B}}^{\frac{\alpha}{1+\alpha}}.
\end{align*}
Applying this inequality for both $\pi = \hat\pi$ and $\pi = \pi_{f^*}$ and invoking \cref{eq:emp-process-n1}, we get 
\begin{align*}
    & (\expect_\Data - \expect_{n_k})[{ \theta({\Data; \hat{\nuisance}^{(-k)}})^\top \prns{\hat{\pi}(X) - \Tilde{\pi}^* (X)} }]\\
    \le & 12 B\Theta \prns{\sqrt{\Tilde{C}(\alpha, \gamma)}\prns{\frac{\text{Reg}(\hat{\pi})}{B}}^{\frac{\alpha}{2(1+\alpha)}}\Tilde{r} + \sqrt{\Tilde{C}(\alpha, \gamma)}\prns{\frac{\text{Reg}(\Tilde{\pi}^*)}{B}}^{\frac{\alpha}{2(1+\alpha)}}\Tilde{r} + \Tilde{r}^2}.
\end{align*}

Taking a union bound over $k\in[K]$ and combining all pieces together, we get that with probability at least $1-\delta$,
\begin{align}
   \frac{ \text{Reg}(\hat{\pi})}{B} 
   \le & 12 \Theta \prns{\sqrt{\Tilde{C}(\alpha, \gamma)}\prns{\frac{\text{Reg}(\hat{\pi})}{B}}^{\frac{\alpha}{2(1+\alpha)}}\Tilde{r} + \sqrt{\Tilde{C}(\alpha, \gamma)}\prns{\frac{\text{Reg}(\Tilde{\pi}^*)}{B}}^{\frac{\alpha}{2(1+\alpha)}}\Tilde{r} + \Tilde{r}^2} \notag \\
   & + \frac{1 }{B}\text{Rate}^{\textbf{N}}\prns{\frac{(K-1)n}{K}, \frac{\delta}{2K}} + \frac{\text{Reg}\prns{\Tilde{\pi}^*}}{B}. \label{eq: inequality}
\end{align}

Note that this is a transcendental inequality that involves $\text{Reg}(\hat{\pi})$ on both sides. It is difficult to solve \cref{eq: inequality} exactly, so we rely on the following lemma to get an upper bound on its solution. The proof of \cref{lemma: inequality} is novel, and we defer the proof to \cref{sec: lemma proof}.
\begin{lemma} \label{lemma: inequality}
Let $c_1, c_2, \alpha, x$ be positive constants. If $x \le c_1 x^{\frac{\alpha}{2(1+\alpha)}} + c_2$, then we have $x \le c_1^{\frac{2\alpha+2}{\alpha+2}} + 2c_2$. 
\end{lemma}

Applying \cref{lemma: inequality} with $x=\text{Reg}(\hat{\pi})/B$, $c_1 = 12 \Theta \tilde{r}\sqrt{\Tilde{C}(\alpha, \gamma)} $, and $c_2 = 12 \Theta \prns{\sqrt{\Tilde{C}(\alpha, \gamma)}\prns{\text{Reg}(\Tilde{\pi}^*)/B}^{\frac{\alpha}{2(1+\alpha)}}\Tilde{r} + \Tilde{r}^2} + \text{Rate}^{\textbf{N}}((K-1)n/K, \delta/(2K))  /B + \text{Reg}\prns{\Tilde{\pi}^*}/B$,
we prove the bound in \cref{thm: doubly robust fast rates}: 
\begin{align*}
   \text{Reg}(\hat{\pi}) \le & B \prns{12\Theta \sqrt{\Tilde{C}(\alpha, \gamma)} \Tilde{r}}^{\frac{2\alpha+2}{\alpha + 2}} 
    + 24 B\Theta \prns{ \sqrt{\Tilde{C}(\alpha, \gamma)}\prns{\frac{\text{Reg}(\Tilde{\pi}^*)}{B}}^{\frac{\alpha}{2(1+\alpha)}}\Tilde{r} + \Tilde{r}^2}  \\
    & + 2\text{Rate}^{\textbf{N}}\prns{\frac{(K-1)n}{K}, \frac{\delta}{2K}}  + 2\text{Reg}\prns{\Tilde{\pi}^*}.
\end{align*}
\endproof

\begin{remark}[Improved Bounds When $2(\alpha+1)/\alpha$ is an Integer]
    Recall that the proof of \cref{thm: doubly robust fast rates} involves bounding the solution to a transcendental inequality (\cref{eq: inequality}) using \cref{lemma: inequality}. In the special case when $2(1+\alpha)/\alpha$ is an integer, this bound can be tightened via the following lemma.
\begin{lemma} \label{lemma: inequality integer}
Let $c_1, c_2, r, y, z, \alpha$ be positive constants such that $2(1+\alpha)/\alpha$ is an integer. If $x > 0$ satisfies $x \le c_1 x^{\frac{\alpha}{2(1+\alpha)}} r + c_1 y^{\frac{\alpha}{2(1+\alpha)}} r + c_2 r^2 + z + y$, then $x \le c_1^{\frac{2\alpha+2}{\alpha +2}} r^{\frac{2\alpha+2}{\alpha+2}} + \frac{2\alpha+2}{\alpha+2} c_1 y^{\frac{\alpha}{2\alpha+2}} r + \frac{2\alpha+2}{\alpha+2} c_2 r^2 + \frac{2\alpha+2}{\alpha+2} z + y$. 
\end{lemma}
Applying \cref{lemma: inequality integer} with $x = \text{Reg}(\hat{\pi})/B$, $y=\text{Reg}(\tilde{\pi}^*)/B$, $z=\text{Rate}^{\textbf{N}}((K-1)n/K, \delta/(2K)) /B$, $r=\tilde{r}$, $c_1 = 12 \Theta \sqrt{\tilde{C}(\alpha, \gamma)}$, and $c_2 =12 \Theta$ yields the tighter bound presented in \cref{eq: tighter bound}.
\end{remark}

\begin{remark}[Comparison with \cite{hu2022fast}]\label{remark: comparison hu2022fast}
As discussed in \Cref{section: full feedback theory}, our general \cref{thm: doubly robust fast rates}, when specialized to the well-specified full-feedback setting,  reduces to the result in \cite{hu2022fast}.
However, the proofs are fundamentally different. The argument in \citet{hu2022fast} relies on sophisticated analyses of a self-normalized empirical process, making it challenging to incorporate additional misspecification error and nuisance-estimation error into their proof. 
In contrast, our proof is much more direct and modular: it uses the regret decomposition in \Cref{eq: regret decomp}, an empirical-process deviation bound in \Cref{lemma: uniform bound G2}, and the margin condition analysis in \Cref{lemma: margin bound}. 
These steps reduce the regret analysis to solving a transcendental inequality, as further addressed in \Cref{lemma: inequality,lemma: inequality integer}. 
To our knowledge, this proof strategy is novel, and can be broadly applicable to a wide range of offline policy optimization problems.
\end{remark}

\section{Computationally Tractable Surrogate Losses}\label{sec: spo+}

The IERM problem in \cref{eq: general ierm} involves a complex induced policy class $\Pi_\F$ that is generally non-convex and non-smooth, so it is computationally challenging to optimize.
In this section, we demonstrate that any computationally tractable surrogate loss previously developed for the full-feedback IERM problem can be seamlessly adapted to the generic IERM problem. 
We conclude the section with three concrete examples of surrogate losses.

\subsubsection*{Surrogate Losses in the Full-Feedback Setting.}
The full-feedback IERM problem in \cref{eq: IERM-full} is equivalent to minimizing the following loss over the function class $\F$:
\begin{align*}
    \min_{f\in\F} \frac{1}{n}\sum_{i=1}^n l_{\text{IERM}}(f(X_i), Y_i),
\end{align*}
where $l_{\text{IERM}}(f(x), y) = y^\top \pi_f(x) - \min_{z\in\Zcal} y^\top z$. 
This loss coincides with the ``smart predict-then-optimize'' (SPO) loss introduced in Definition 1 of \cite{elmachtoub2022smart}, which is known to be computationally intractable in general. 

To overcome this difficulty, the full-feedback CLO literature has proposed several surrogate losses that replace the intractable $l_{\text{IERM}}$ with a computationally efficient alternative $l_{\text{SR}}$, yielding the surrogate optimization problem:
\begin{align}\label{eq: surrogate full}
    \hat f_{\text{SR-F}} \in \arg\min_{f \in \F} \frac{1}{n} \sum_{i=1}^n l_{\text{SR}}(f(X_i), Y_i).
\end{align}
The resulting plug-in policy $\pi_{\hat{f}_{\text{SR-F}}}$ is then deployed as the decision policy.

\subsubsection*{Extending Surrogates to the Generic IERM Setting.}
A key insight is that in our general IERM objective (\cref{eq: general ierm}), the score function $\theta(\Data_i;\hat{\nuisance}^{(-k)})$ plays the same role as the true coefficient vector $Y_i$ in the full-feedback objective.
This observation suggests a natural extension of the surrogate loss approach: simply replace $Y_i$ with $\theta(\Data_i;\hat{\nuisance}^{(-k)})$ to obtain
\begin{equation*}
    \hat{f}_{\text{SR}}\in \arg\min_{f\in \F}~\frac{1}{n} \sum_{k\in[K]} \sum_{i\in \Dcal^{(k)}} l_{\text{SR}}\prns{f(X_i), \theta(\Data_i;\hat{\nuisance}^{(-k)})},
\end{equation*}
and use the resulting plug-in policy $\pi_{\hat{f}_{\text{SR}}}$ as the decision rule.

This procedure highlights that our framework is agnostic to the particular surrogate loss chosen: \emph{any} surrogate loss $l_{\text{SR}}$ developed for the full-feedback setting can be directly employed to handle partial feedback, making the approach both practical and general.

\subsubsection*{Examples of Surrogate Losses.}
Finally, we briefly review three representative surrogate losses frequently used in the literature and employ them in our numerical experiments in \cref{sec: numerical}. These losses all allow for efficient gradient- or subgradient-based optimization.

\begin{enumerate}
    \item \textit{SPO+ Loss.} 
    \cite{elmachtoub2022smart} propose the SPO+ surrogate loss:
    \begin{align*}
        l_{\text{SPO+}}(f(x), y) = \max_{z\in\Zcal} \prns{y - 2 f(x)}^\top z - \prns{y - 2 f(x)}^\top z^*(y),
    \end{align*}
    where $z^*(y) \in \arg\min_{z\in\Zcal} y^\top z$ with the same tie-breaking rule as in $\pi_f$. 
    The SPO+ loss has several desirable properties: for any fixed $y$, it upper bounds the IERM loss, is convex in $f(x)$, and has a closed-form subgradient  $2\!\prns{z^*(y) - z^*(2f(x) - y)}$ with respect to $f(x)$. 

\item \textit{Perturbation-Based Surrogates (PGC \& PGB).} \cite{huang2024decision} point out that the decision loss $\ell(f(x), y) = y^\top \pi_f(x)$ can be viewed as a directional derivative of a certain optimization problem, and propose to approximate it using finite differences. This results in surrogate losses that are Lipschitz continuous, bounded, and differentiable. For a smoothing parameter $h>0$, the two main variants are Perturbation Gradient Central (PGC) and Backward (PGB) losses as follows:
\[
\ell_h^{\text{PGC}}(f(x), y) = \frac{1}{2h} \left( \min_{z\in\mathcal{Z}} (f(x) + hy)^\top z - \min_{z\in\mathcal{Z}} (f(x) - hy)^\top z \right),
\]
\[
\ell_h^{\text{PGB}}(f(x), y) = \frac{1}{h} \left( \min_{z\in\mathcal{Z}} f(x)^\top z - \min_{z\in\mathcal{Z}} (f(x) - hy)^\top z \right).
\]
Given any fixed $y$, these two losses admit closed-form subgradients with respect to $f(x)$,  given by $\frac{1}{2h} (z^*(f(x)+hy) - z^*(f(x)-hy))$ and $\frac{1}{h} (z^*(f(x)) - z^*(f(x)-hy))$, respectively.

    \item \textit{Perturbed Fenchel-Young Loss (PFYL).} 
    \cite{berthet2020learning} propose a Fenchel-Young loss $l_{\text{FY}}$, which was applied to IERM problem in \cite{tang2024pyepo}. This loss measures decision error relative to $z^*(y) \in \arg\min_{z\in\Zcal} y^\top z$ in terms of the predicted cost $f(x)$:
    \[
    l_{\text{FY}}(f(x), y) = f(x)^\top z^*(y) - F(f(x)) - \Omega(z^*(y)),
    \]
    where $F(y) = \E_\xi\!\left[\min_{z \in \Zcal} (y + \sigma \xi)^\top z\right]$ is the expected perturbed cost and $\Omega$ its Fenchel conjugate.
    The gradient of $l_{\text{FY}}$ with respect to $f(x)$ can be approximated using Monte Carlo sampling:
    \[
    \textstyle
    \frac{\partial l_{\text{FY}}(f(x), y)}{\partial f(x)} \approx 
    z^*(y) - \frac{1}{M} \sum_{m=1}^M 
    z^*(f(x) + \sigma \xi_m),
    \]
    where $\{\xi_m\}_{m=1}^M$ are random perturbations (\eg, from a Gaussian distribution) and $\sigma > 0$ is a smoothing parameter.
\end{enumerate}

\section{Extension: Contextual Linear Optimization with Semi-Bandit Feedback} \label{sec: semi}

So far, we have studied two feedback structures for contextual linear optimization: a full-feedback setting where the entire cost vector $Y$ is observed, and a bandit setting where only a single scalar total cost $C=Y^\top Z$ is observed. In this section, we explore an intermediate scenario, which we term \textit{semi-bandit feedback}.

Recall our stochastic shortest path example (\cref{fig:feedback-comparison}). The bandit setting means that only the total travel time of a chosen path is recorded. In the semi-bandit setting, more granular information is available: we observe the travel time for each individual edge that was part of the chosen path.

Mathematically, we assume access to an i.i.d. offline dataset $\mathcal{D}_S = \braces{(X_i, Z_i, \tilde{Y}_i)}_{i=1}^n$. We focus on the special case where each decision $Z_i = (Z_{i1}, \dots, Z_{id})^\top$ is a binary vector, i.e., $Z_{ij} \in \braces{0,1}$. In the shortest path example, $Z_{ij} = 1$ indicates that edge $j$ was part of the chosen path. For each data point $i$, we observe the context $X_i$, the decision $Z_i$, and a partially revealed cost vector $\tilde{Y}_i$, where for each component $j \in \{1, \dots, d\}$, we have $\tilde{Y}_{ij} = Y_{ij}$ if $Z_{ij} = 1$ (\eg, we observe the travel time for edge $j$ when it is part of a chosen path), but $\tilde{Y}_{ij}$ is a missing value if $Z_{ij} = 0$.

In this semi-bandit setting, the ignorability and coverage conditions in \cref{assump: dgp} reduce to two simpler conditions in the corollary below. For this corollary, we say a component $j  \in \{1, \dots, d\}$ is \textit{feasible} if there exists $z \in \Z$ with $j$-th entry $z_j = 1$, \ie, this component is active in one feasible decision. Conversely, a component $j$ is \textit{infeasible} if there exists no $z \in \Z$ with $z_j = 1$.
\begin{corollary}\label{coro: dgp semi}
    Suppose \cref{assump: dgp} holds. Then, for any feasible component $j \in \{1, \dots, d\}$,  we have: 
    \begin{enumerate}
        \item (Ignorability) The expected cost coefficient $f^*_j(X_i) = \mathbb{E}[Y_{ij} \mid X_i]$ satisfies $f^*_j(X_i) = \mathbb{E}[\tilde Y_{ij} \mid Z_{ij} = 1, X_i]$ almost surely. 
        \item (Coverage) The propensity score  $e_j^*(X_i) = \mathbb{P}(Z_{ij}=1\mid X_i)$ is strictly positive almost surely.
    \end{enumerate}
\end{corollary}

\Cref{coro: dgp semi} shows that, under \cref{assump: dgp}, for any feasible component $j$ that may be selected by a feasible decision, the corresponding expected cost coefficient $f^*_j(X_i)$ can be recovered from the non-missing cost observation $\tilde Y_{ij}$ through the conditional expectation $\mathbb{E}[\tilde Y_{ij} \mid Z_{ij} = 1, X_i]$. Moreover, any such component must also be explored with positive probability in  the data, as indicated by the strictly positive propensity score $e_j^*(X_i) > 0$. 
Although infeasible components may exist, they are irrelevant to the decision-making problem and therefore need not be considered.

For evaluating the expected policy cost $V(\pi)$ for any policy $\pi$, we again rely on some score functions. In the semi-bandit setting, the score functions involve two types of nuisance functions that can be estimated from the observed data: the conditional expected coefficients $f^*(x) = (f^*_1(x), \dots, f^*_d(x))^\top$, where $f^*_j(x) = \mathbb{E}[\tilde Y_{ij}\mid Z_{ij} = 1, X_i=x]$; and the conditional {propensity scores} $e^*(x) = (e^*_1(x), \dots, e^*_d(x))^\top$, where $e^*_j(x) = \pr(Z_{ij}=1\mid X_i=x)$. 
Since any infeasible component $j$ does not contribute to the evaluation of $V(\pi)$, without loss of generality, we set $f^*_j(x)=0$ and $e^*_j(x)=0$ for such components.

The following proposition provides three valid score functions.

\begin{proposition}\label{lemma: theta functions semi} 
Suppose \cref{assump: dgp} holds.  The following choices of $\theta(x,z,\tilde{y}; f, e)$, when evaluated at $f = f^*$ and $e = e^*$, all satisfy \Cref{assump: score general}, i.e., $\mathbb{E}_{X,Z,\tilde Y}\bracks{\theta(X,Z,\tilde Y; f^*, e^*)^\top \pi(X)} = V(\pi)$ for any fixed policy $\pi: \mathcal{X} \to \Z$:
\begin{enumerate}
   \item (Direct Method) $\theta_{\text{DM-S}}(x,z,\tilde y; f,\propensity) = f(x)$;
  \item (Inverse Propensity Weighting) $\theta_{\text{IPW-S}}(x,z,\tilde y; f,\propensity) = \prns{\frac{ \tilde y_1 z_1   }{\propensity_1(x)}, \dots, \frac{ \tilde y_d z_d   }{\propensity_d(x)}}^\top$;
      \item (Doubly Robust) $\theta_{\text{DR-S}}(x,z,\tilde y; f,\propensity) = \prns{f_1(x) + \frac{ \prns{\tilde y_1 - f_1(x) } z_1   }{\propensity_1(x)}, \dots, f_d(x) + \frac{ \prns{\tilde y_d - f_d(x) } z_d   }{\propensity_d(x)}}^\top$.
\end{enumerate}
Here, the subscript $j$ denotes the $j$-th component of the vector, and we adopt the convention that terms involving the fraction $z_j/e_j(x)$ are treated as zero\footnote{Since any infeasible component would never appear in the data, this operation would  effectively set the corresponding term in the score function to zero.} whenever $z_j=0$. 
\end{proposition}

For each score function in \cref{lemma: theta functions semi}, we again estimate the involved nuisance functions using $K$-fold cross-fitting: split the dataset $\mathcal{D}_S$ into $K$ equal-sized folds $\mathcal{D}_S^{(1)},\dots,\mathcal{D}_S^{(K)}$, and repeatedly train nuisance estimators on datasets with each fold held out. 
This estimation can be performed component-wise for both  $f^* = (f^*_1, \dots, f^*_d)$ and  $e^* = (e^*_1, \dots, e^*_d)$, restricted to feasible components. 
Specifically, for each feasible cost component, the estimator $\hat f_j^{(-k)}$ can be obtained by regressing the observed costs $\tilde Y_{ij}$ for component $j$ on the corresponding context $X_i$ using all non-missing observations excluding the $k$-th fold, \ie, data points in $\braces{i\in \mathcal{D}_S\setminus \mathcal{D}_S^{(k)}: Z_{ij} = 1}$. Similarly, for each feasible propensity score component, the estimator $\hat{e}_j^{(-k)}$ can be obtained by fitting a binary classification for the corresponding decision observation $Z_{ij}$ given the context $X_i$ using all but the $k$-th fold data, \ie, data points in $\mathcal{D}_S\setminus \mathcal{D}_S^{(k)}$.

With the nuisance estimators in hand, we can now instantiate the general IERM framework from \cref{eq: general ierm} for the semi-bandit setting by taking $\Data = (X, Z, \tilde Y)$, $\nuisance^* = (f^*, \propensity^*)$ and $\theta$ as any of the score functions in \cref{lemma: theta functions semi}. Using the DR score as a concrete example, the DR-IERM policy, $\hat{\pi}_{\text{DR-S}}$, is obtained by solving the following optimization problem:
\begin{align*}
    \textstyle
    \hat{\pi}_{\text{DR-S}} \in \arg\min_{\pi\in \Pi_\F} \frac{1}{n} \sum_{k\in[K]} \sum_{i\in \mathcal{D}_S^{(k)}} \sum_{j=1}^d \prns{\hat{f}_j^{(-k)}(X_i)+{ \prns{\tilde{Y}_{ij} - \hat{f}_j^{(-k)}(X_i)}Z_{ij}}/{\hat{e}_j^{(-k)}(X_i)}} \pi_j(X_i).
\end{align*}

Following from our main result in \cref{thm: doubly robust fast rates}, the regret of an IERM policy with score function $\theta(x,z,\tilde{y}; f, e)$ is additive in three primary components: the misspecification error, $\text{Reg}(\tilde{\pi}^*)$, the nuisance estimation error, $\operatorname{Rate}^{\textbf{N}}(n, \delta)$, and the statistical error, $O(\tilde{r}_S^{(2\alpha+2)/(\alpha+2)})$.
Here, the value $\tilde{r}_S$ upper-bounds the critical radius of the following function classes for $k \in [K]$:
\begin{align*}
    \mathcal{G}^{(-k)}_S = \braces{(x,z,\tilde y) \rightarrow \frac{\theta\prns{x,z,\tilde{y}; \hat{f}^{(-k)}, \hat{e}^{(-k)}}^\top \prns{\pi(x) - \Tilde{\pi}^*(x)} \rho}{2B\Theta} :  \pi\in\Pi_\F, \rho\in[0,1]}.
\end{align*}
The following proposition specifies the nuisance estimation error for different score functions. 
\begin{proposition} \label{nuisance rate semi}
Suppose that for any given $\delta \in (0, 1)$, the nuisance estimators trained on an i.i.d. sample of size $n$ satisfy the following mean-squared error bound with probability at least $1-\delta$ for some positive sequences $\chi_f(n, \delta) \to 0$, $\chi_e(n, \delta) \to 0$ as $n \to \infty$:
\begin{align*}
    \mathbb{E}_X\bracks{\sum_{\text{feasible } j}\prns{\hat{f}_j(X) - f^*_j(X)}^2}\le \chi_f^2(n, \delta), ~~ \mathbb{E}_X\bracks{\sum_{\text{feasible } j}\prns{{1}/{\hat e_j(X)} - {1}/{e^*_j(X)}}^2} \le \chi_e^2(n, \delta).
\end{align*}
Then, the nuisance estimation error is bounded as follows:
\begin{enumerate}
    \item If $\theta=\theta_{\text{DM-S}}$, then $\text{Rate}^{\textbf{N}}_{\text{DM-S}}(n, \delta) = O(\chi_f(n, \delta))$.
    \item If $\theta=\theta_{\text{IPW-S}}$, then $\text{Rate}^{\textbf{N}}_{\text{IPW-S}}(n,\delta) = O(\chi_e(n, \delta))$.
    \item If $\theta=\theta_{\text{DR-S}}$, then $\text{Rate}^{\textbf{N}}_{\text{DR-S}}(n,\delta) = O(\chi_f(n, \delta)\chi_e(n, \delta))$.
\end{enumerate}
\end{proposition}

Finally, since the policy optimization in the semi-bandit setting is also a special case of the general IERM problem in \cref{eq: general ierm}, we can again use the surrogate losses discussed in \cref{sec: spo+} (\eg, SPO+, PGB, PGC, and PFYL) to address the computational challenge.

\section{Numerical Experiments} \label{sec: numerical}
In this section, we conduct extensive numerical experiments to gain empirical insights for our proposed IERM framework. We evaluate our methods under various conditions on synthetic data, detailed in \cref{sec: synthetic}, and further test them on a real-world Uber dataset in \cref{sec: real data}.

\subsection{Synthetic Experiments}\label{sec: synthetic}

We begin with a simulated stochastic shortest path problem following \citet{elmachtoub2022smart,hu2022fast}.

\subsubsection*{Experimental Setup}
The task is to travel from a start node $s$ to a terminal node $t$ on a $5 \times 5$ grid with $d=40$ edges (the road network is shown in \cref{fig:feedback-comparison}). Edge costs are given by a random vector $Y \in \mathbb{R}^{40}$. We consider $3$-dimensional contexts $X$ and a conditional mean $f^*(x) = \mathbb{E}[Y \mid X = x]$ whose components are cubic polynomials in $x$; implementation details are in \cref{sec:exp-setup}. The corresponding shortest path problem is formulated as a CLO problem with the constraint set $\mathcal{Z}$ given by standard flow-preservation constraints, with a source of $+1$ at $s$ and a sink of $-1$ at $t$.
The optimal solution $z^*$ belongs to $\{0, 1\}^{40}$, indicating whether passing each edge or not. There are $m = 70$ feasible paths from the start node to the terminal node, and the feasible paths are linearly dependent with a rank of $18$. Therefore, this problem corresponds to the rank-deficient setting discussed below \Cref{assump: dgp}. 

We consider two logging policies for generating observed decisions. 
In this section, we report results using a \emph{random} logging policy, which selects paths uniformly at random from all feasible ones, independent of the context. For these selected paths, we then generate the bandit-feedback and semi-bandit-feedback data.
In \cref{sec: x1x2}, we also study a \emph{covariate-dependent} logging policy that depends on the signs of $X_1$ and $X_2$. Empirically, the conclusions under the covariate-dependent policy are qualitatively similar to those under the random logging policy. 

In the experiments, we need to specify two function classes, the policy-inducing class $\mathcal{F}$ and the nuisance estimation class $\mathcal{F}^{\mathrm{N}}$ (see also \Cref{remark:f}). For both the ETO and IERM approaches, we use the same class $\mathcal{F}$ to induce decision policies. We consider three choices for $\mathcal{F}$ (details in \cref{sec:exp-setup}): (i) a correctly specified polynomial class; (ii) a misspecified class omitting two higher-order terms (``degree-2 misspecification''); and (iii) a misspecified class omitting four higher-order terms (``degree-4 misspecification''). The IERM approach (evaluated via DM and DR) additionally requires estimating the nuisance function $\tilde f^*$. We implement this by running a least squares bandit-feedback regression in \cref{eq: nuisance f}, with additional ridge regularization to reduce overfitting. We also consider the same three (mis)specification scenarios for $\F^{\text{N}}$ as for $\mathcal{F}$. However, we allow $\mathcal{F}^{\mathrm{N}}$ to differ from $\mathcal{F}$ in each experiment.

Performance of IERM policies further depends on the policy-evaluation method and the surrogate loss used for optimization. For evaluation, we report Direct Method (DM) and Doubly Robust (DR) results, omitting ISW/IPW due to substantially worse performance. For DR, we use the true $\Sigma^*(x)$ in this section for clarity, and show in \cref{sec: estimated sigma} that the findings are robust to using an estimated $\hat{\Sigma}(x)$. We further consider two DR variants when plugging in $\Sigma^*(x)$: \emph{DR PI}, which uses the Moore–Penrose pseudoinverse as in the definition of $\theta_{\mathrm{DR}}$; and \emph{DR Lambda}, which replaces $\Sigma^{*\dagger}$ with $(\Sigma + \lambda I)^{-1}$ for a small positive constant $\lambda$. For optimization, we consider the SPO+, PGC, PGB, and PFYL surrogate losses discussed in \cref{sec: spo+}.

For each configuration, we use training sample sizes $n \in \{400,1000,1600\}$ and evaluate on an independent test set of size $2000$. We report the \emph{relative regret ratio}, defined as the ratio of policy regret to the expected cost of the globally optimal policy $\pi_{f^*}$:
$\mathrm{Reg}(\hat{\pi})/V(\pi_{f^*})$.
All results are averaged over $50$ independent replications.

\subsubsection*{Experimental Findings}

Our main experimental results are presented in \cref{tab:Feedback type,tab:Nuisance Specifications,tab:Evaluation Methods}. We now analyze these results from several perspectives to draw key insights.

\begin{table}[t]
\centering
\begin{tblr}{
  cells = {c},
  cell{1}{3} = {c=3}{},
  cell{1}{6} = {c=3}{},
  cell{3}{1} = {r=2}{},
  cell{5}{1} = {r=4}{},
  cell{9}{1} = {r=4}{},
  vlines,
  hline{1-3,5,9,13} = {-}{},
  hline{4,6-8,10-12} = {2-8}{},
}
                        &                & $\F$ well-specified &        &        & $\F$ misspecified degree 4 &         &         \\
Feedback                & Methods        & 400              & 1000~  & 1600   & 400                     & 1000    & 1600    \\
{Full\\Feedback}        & ETO            & 0.02\%           & 0.01\% & 0.00\% & 3.80\%                  & 3.63\%  & 3.59\%  \\
                        & SPO+           & 0.14\%           & 0.06\% & 0.04\% & 3.31\%                  & 3.23\%  & 3.17\%  \\
{Semi-bandit\\Feedback} & ETO            & 2.19\%           & 0.25\% & 0.03\% & 9.07\%                  & 5.08\%  & 4.55\%  \\
                        & SPO+ DM        & 2.86\%           & 0.05\% & 0.02\% & 4.79\%                  & 3.04\%  & 3.01\%  \\
                        & SPO+ DR   & 2.80\%           & 0.16\% & 0.09\% & 4.95\%                  & 3.43\%  & 3.37\%  \\
                        & SPO+ DR Lambda & 3.03\%           & 0.11\% & 0.07\% & 4.85\%                  & 3.26\%  & 3.24\%  \\
{Bandit\\Feedback}      & ETO          & 16.35\%          & 2.92\% & 1.04\% & 25.95\%                 & 17.23\% & 14.35\% \\
                        & SPO+ DM        & 4.66\%           & 0.30\% & 0.09\% & 5.03\%                  & 3.17\%  & 3.03\%  \\
                        & SPO+ DR PI     & 4.63\%           & 0.47\% & 0.24\% & 5.38\%                  & 3.51\%  & 3.37\%  \\
                        & SPO+ DR Lambda & 4.37\%           & 0.28\% & 0.12\% & 4.91\%                  & 3.33\%  & 3.23\%  
\end{tblr}
\caption{Mean relative regret ratio across feedback regimes and policy-inducing class specifications. The nuisance class $\mathcal{F}^{\mathrm N}$ is well specified, and the logging policy selects feasible paths uniformly at random.}
\label{tab:Feedback type}
\end{table}

\paragraph{Feedback Types.}
\cref{tab:Feedback type} compares the performance of the two-stage ETO approach against our end-to-end IERM methods (using the SPO+ surrogate) across the three feedback settings, under both correctly specified and misspecified policy classes $\F$. 

When the policy-inducing class $\mathcal{F}$ is well specified, we observe a clear trend: for all methods, performance degrades as the information decreases, with regret increasing from the full-feedback to the semi-bandit and bandit settings. This is intuitive, as less information is available for learning. In the full-feedback case, ETO outperforms IERM, consistent with findings in \cite{hu2022fast,elmachtoub2022smart}. However, this advantage vanishes under partial feedback; in the most challenging bandit setting, IERM methods are substantially better than ETO. We conjecture this is because ETO's first-stage cost estimation becomes highly inaccurate with limited bandit data, causing it to optimize a substantially misleading objective.

When $\mathcal{F}$ is misspecified, the IERM methods consistently and significantly outperform ETO across all feedback types. This highlights the key advantage of the IERM framework: its end-to-end nature provides robustness against model misspecification, even under the more challenging partial feedback settings. While both ETO's and IERM's performance degrade as feedback becomes more limited, the increase in regret is noticeably smaller for IERM.

In the remainder of \cref{sec: numerical}, we will focus on the numerical results for the bandit-feedback setting. Similar experiments for the semi-bandit-feedback setting are deferred to \cref{sec: semi experiments}, and we note that the insights regarding function class specification, evaluation methods, and optimization methods hold qualitatively true for that setting as well.

\begin{table}
\centering
\begin{tabular}{|c|c|c|c|c|c|c|c|c|c|} 
\hline
           & \multicolumn{3}{c|}{$\F^{\text{N}}$ well-specified} & \multicolumn{3}{c|}{$\F^{\text{N}}$ misspecified degree 2} & \multicolumn{3}{c|}{$\F^{\text{N}}$ misspecified degree 4}  \\ 
\hline
Evaluation method & 400    & 1000   & 1600                                                                     & 400     & 1000   & 1600                                                                           & 400     & 1000    & 1600                                                                           \\ 
\hline
DM         & 4.66\% & 0.30\% & 0.09\%                                                                   & 21.21\% & 10.70\% & 8.33\%                                                                         & 23.45\% & 12.38\% & 9.90\%                                                                         \\ 
\hline
DR PI      & 4.63\% & 0.47\% & 0.24\%                                                                   & 18.64\%  & 9.13\% & 6.76\%                                                                         & 21.42\% & 10.97\% & 8.75\%                                                                         \\ 
\hline
DR Lambda  & 4.37\% & 0.28\% & 0.12\%                                                                   & 19.11\%  & 8.88\% & 6.53\%                                                                         & 21.59\% & 10.88\% & 8.70\%                                                                         \\
\hline
\end{tabular}
\caption{Mean relative regret ratio across evaluation methods under bandit feedback with a random logging policy. The policy-inducing class $\mathcal{F}$ is well specified, while the nuisance class $\mathcal{F}^{\mathrm N}$ is misspecified to varying degrees. Optimization uses SPO+.}
\label{tab:Nuisance Specifications}
\end{table}

\paragraph{Nuisance Class ($\F^{\text{N}}$) Specification.}
In \cref{tab:Nuisance Specifications}, under a random logging policy and bandit feedback, we report the relative regret for different evaluation methods when the nuisance class $\mathcal{F}^{\mathrm N}$ is misspecified to varying degrees, while the policy-inducing class $\mathcal{F}$ is well specified. Optimization uses SPO+ throughout.

A key observation, in comparison to the results in \cref{tab:Feedback type}, is that the performance of the end-to-end IERM approach degrades significantly when the nuisance model $\F^{\text{N}}$ is misspecified. In this setting, policy evaluation becomes biased, so the end-to-end approaches also target a ``wrong'' objective that may not accurately capture the decision quality. This highlights a challenge unique to partial-feedback CLO: sensitivity of end-to-end methods to nuisance misspecification. This sensitivity suggests a practical guideline: one may prefer using a highly flexible model for the nuisance class $\F^{\text{N}}$ to ensure accurate estimation, while using a simpler class $\F$ for the policy to maintain end-to-end tractability.

Furthermore, the results reveal a clear tradeoff between the DM and DR scores. When the nuisance model $\F^{\text{N}}$ is misspecified, the DR score (especially the regularized DR-Lambda variant) consistently outperforms the DM score, as its structure is designed to debias the misspecified nuisance (see discussion below \Cref{example: DR}). Conversely, when the nuisance model $\F^{\text{N}}$ is correctly specified, DM often performs better, likely because it’s free from the inflated variance due to the inverse Gram matrix in the DR method. This highlights that the DM and DR scores face different bias-and-variance tradeoffs under nuisance model misspecification.

\begin{table}
\centering
\begin{tabular}{|c|c|c|c|c|c|c|c|c|c|} 
\hline
             & \multicolumn{3}{c|}{DM}  & \multicolumn{3}{c|}{DR PI} & \multicolumn{3}{c|}{DR Lambda}  \\ 
\hline
Optimization method & 400    & 1000   & 1600   & 400    & 1000   & 1600     & 400    & 1000   & 1600          \\ 
\hline
SPO+         & 4.66\% & 0.30\% & 0.09\% & 4.63\% & 0.47\% & 0.24\%   & 4.37\% & 0.28\% & 0.12\%        \\ 
\hline
PGC          & 1.14\% & 0.21\% & 0.09\% & 1.59\% & 0.32\% & 0.23\%   & 1.49\% & 0.30\% & 0.20\%        \\ 
\hline
PGB          & 1.23\% & 0.23\% & 0.09\% & 1.66\% & 0.44\% & 0.42\%   & 1.46\% & 0.42\% & 0.34\%        \\ 
\hline
PFYL         & 1.69\% & 0.25\% & 0.08\% & 1.90\% & 0.48\% & 0.42\%   & 1.78\% & 0.31\% & 0.16\%        \\
\hline
\end{tabular}
\caption{Mean relative regret ratio across evaluation and optimization methods under bandit feedback with a random logging policy. Both the policy-inducing class $\mathcal{F}$ and the nuisance class $\mathcal{F}^{\mathrm N}$ are well specified.}
\label{tab:Evaluation Methods}
\end{table}

\paragraph{Comparison of Evaluation and Optimization Methods.}
\cref{tab:Evaluation Methods} compares various surrogate losses and evaluation methods under a random logging policy and bandit feedback when both the policy class $\F$ and the nuisance class $\F^{\text{N}}$ are well-specified. Consistent with our finding from \cref{tab:Nuisance Specifications}, DM outperforms DR when the nuisance function is correctly specified (and as a reminder, the reverse holds under misspecification). When comparing the optimization methods (surrogates), we observe that all  losses lead to policies whose regret decreases with the  sample size, demonstrating their general effectiveness. Among them, PGC exhibits the strongest performance with small sample sizes ($n=400$), while SPO+ and PGC emerge as the top performers as the sample size increases.

\subsection{Real Data Experiments} \label{sec: real data}
We now evaluate our methods on a real-world dataset from Uber Movement (\url{https://movement.uber.com})\footnote{This website is no longer actively maintained as of \the\month/\the\year; we use data archived prior to its shutdown.}, focusing on 45 census tracts in downtown Los Angeles. The dataset contains average travel times between adjacent tracts across 93 edges for 2018–2019. 
We again use a \emph{random} logging policy that selects paths from feasible ones completely at random, and generate the corresponding bandit- and semi-bandit-feedback data. We specify the policy-inducing class $\F$ and the nuisance class $\F^{\text{N}}$ to be the same class of functions linear in features $X \in \mathbb{R}^{12}$ (including covariates such as wind speed and visibility).\footnote{We also tested more flexible neural network classes for  $\F^{\text{N}}$. The results are similar and thus omitted. } We again use the true conditional Gram matrix $\Sigma^*$ in the DR scores.

\begin{table}[t]
\centering
\begin{tabular}{|c|c|c|c|c|c|c|c|c|} 
\hline
               & \multicolumn{2}{c|}{Full Feedback} & \multicolumn{3}{c|}{Semi-bandit Feedback}                                  & \multicolumn{3}{c|}{Bandit Feedback}                                         \\ 
\hline
 \begin{tabular}[c]{@{}c@{}}Period\\(Months)\end{tabular} & ETO    & SPO+~                     & ETO    & SPO+ DM & \begin{tabular}[c]{@{}c@{}}SPO+ DR\\Lambda\end{tabular} & ETO     & SPO+ DM & \begin{tabular}[c]{@{}c@{}}SPO+ DR\\Lambda\end{tabular}  \\ 
\hline
6            & 1.01\% & 1.01\%                    & 2.24\% & 1.01\%  & 1.01\%                                                  & 15.46\% & 13.23\% & 14.81\%                                                  \\
\hline
12             & 0.97\% & 0.98\%                    & 1.64\% & 0.98\%  & 0.98\%                                                  & 15.83\% & 10.21\% & 9.71\%                                                  \\ 
\hline
18            & 0.93\% & 0.92\%                    & 1.29\% & 0.92\%  & 0.92\%                                                  & 12.52\% & 4.89\%  & 3.25\%                                                   \\ 
\hline
24              & 0.84\% & 0.84\%                    & 1.07\% & 0.84\%  & 0.84\%                                                  & 10.72\% & 3.50\%  & 1.99\%                                                   \\ 
\hline
\end{tabular}
\caption{Mean relative regret ratio across methods in the real-data experiment under a random logging policy.}
\label{tab:Real data result}
\end{table}

We consider four different sample sizes corresponding to four time horizons: 6 months (920 data points), 12 months (1{,}825 data points), 18 months (2{,}735 data points), and 24 months (3{,}640 data points). For each horizon, we randomly split the data into 40\% training, 40\% validation, and 20\% testing, and repeat the split across 50 independent trials. As in the synthetic studies, we compare across feedback regimes, evaluation methods, optimization surrogates, and logging policies. In the main text, we present a core set of results in \cref{tab:Real data result}, comparing our IERM approach (using the SPO+ surrogate with DM and DR-Lambda evaluation) against the ETO baseline under a random logging policy. Additional results are provided in \cref{sec: real data exp extra}.

The results, summarized in \cref{tab:Real data result}, are consistent with our findings from the synthetic experiments. As expected, all methods perform worse as the feedback becomes more limited, with performance degrading from the full-feedback to the semi-bandit and bandit settings. 
In the full-feedback setting, ETO and IERM approaches perform comparably. However, in the partial-feedback settings, the end-to-end IERM approaches demonstrate a clear advantage.
This performance gap is most pronounced in the challenging bandit-feedback setting, highlighting the robustness of the IERM framework when data is limited.

\section{Conclusions} \label{sec: conclusions}
This paper proposes a unified learning framework for contextual linear optimization across multiple feedback regimes, with a particular focus on partial feedback, including the prevalent bandit and semi-bandit settings. Our framework adapts the powerful induced empirical risk minimization approach that integrates estimation and optimization, and it can be solved efficiently using existing surrogate losses. We provide novel regret analyses for the resulting policies, allowing for policy-class misspecification and nuisance-estimation error. Empirical experiments demonstrate the effectiveness of our method and offer practical insights from multiple perspectives.

\bibliographystyle{informs2014} %
\bibliography{literature} %

\ECSwitch
\EquationsNumberedBySection

\ECHead{Electronic Companions}

\section{Omitted Proofs} \label{sec: proofs}

\subsection{\cref{sec: ierm-bandit} Omitted Proofs}

\proof{Proof of \cref{prop: unbiased}}
According to the first-order condition of the least squares problem in \cref{eq: f-nuisance}, any $\tilde f^* \in \F^*$ (in particular, $f^*$) must satisfy 
\begin{align*}
    {\Sigma^*(X)\tilde f^*(X)} = \mathbb{E}\bracks{ZC \mid X}. 
\end{align*}
Since $Z \in \Z$, $\Sigma^*(X) = \mathbb{E}[ZZ^\top \mid X]$ and the coverage assumption in \cref{assump: dgp} holds, the row space and column space of $\Sigma^*$ are both identical to $\text{Span}(\Z)$. This means that $\tilde f^* - f^*$ must be orthogonal to $\text{Span}(\Z)$. In particular, let $M$ be a matrix whose columns include all basis vectors of the span of  $\Z$. Then $M^\top (\tilde f^*(X) - f^*(X)) = 0$ or $M^\top \tilde f^*(X) =  M^\top f^*(X)$. Moreover, we have $\pi(X) \in \Z \in \text{Range}(M)$, i.e., there exists a vector of functions $a(X)$ such that $\pi(X) = Ma(X)$. Therefore, for any $\tilde f^* \in \F^*$ and any policy $\pi: \mathcal{Z} \to \Z$, we have 
\begin{align*}
    \mathbb{E}\bracks{\tilde f^*(X)^\top \pi(X)} = \mathbb{E}\bracks{\tilde f^*(X)^\top M a(X)} = \mathbb{E}\bracks{f^*(X)^\top M a(X)} = \mathbb{E}\bracks{ f^*(X)^\top \pi(X)} = V(\pi). 
\end{align*}
This finishes the proof for the score $\theta = \theta_{DM}$. 

For ISW, we have 
\begin{align*}
    \expect\bracks{\prns{\Sigma^{*\dagger}(X)ZC}^\top \pi(X)} - \expect\bracks{f^*(X)^\top \pi(X)} = \expect\bracks{f^*(X)^\top(I-\Sigma^{*\dagger}(X)\Sigma^*(X))^\top \pi(X)}.
\end{align*}
Recall that the column space of $\Sigma^*(x)$ is identical to the column space of $M$. By the property of pseudo-inverse, the column space of $(I - \Sigma^{*\dagger}\Sigma^*)$ is orthogonal to the column space of $M$. Therefore, $(I-\Sigma^{*\dagger}(X)\Sigma^*(X))^\top \pi(X) = 0$ for any $\pi \in \Z$. This proves that 
\begin{align*}
    \expect\bracks{\prns{\Sigma^{*\dagger}(X)ZC}^\top \pi(X)} =  \expect\bracks{f^*(X)^\top \pi(X)} = V(\pi),
\end{align*}
which finishes the proof for the score $\theta = \theta_{ISW}$.

For the DR score, we have that for any function $\tilde f^*, \Sigma^*$,
    \begin{align*}
        &\expect\bracks{\pi(X)^\top \prns{\tilde f^*(X)+\Sigma^{*\dagger}(X) Z(C-Z^\top \tilde f^*(X))}} - \expect[\pi(X)^\top  f^*(X)] \\
        = &\expect\bracks{\pi(X)^\top \prns{(I-\Sigma^{*\dagger}(X)\Sigma^*(X))(\tilde f^*(X)-f^*(X))}} = 0,
    \end{align*}
because $(I-\Sigma^{*\dagger}(X)\Sigma^*(X))^\top \pi(X) = 0$ for any $\pi \in \Z$. This finishes the proof for $\theta=\theta_{DR}$. 
\endproof

\subsection{\cref{sec: critical radius} Omitted Proofs} \label{sec: critical radius proof}

\subsubsection{Supporting Lemmas}

\begin{lemma}\label{lemma: star hull vc}
    Suppose function class $\Gcal$ has a VC-subgraph dimension $\text{Pdim}(\Gcal)$. Then its star-hull, $\Gcal^\star=\braces{ag: g\in\Gcal, a\in [0,1]}$, has a VC-dimension at most $4(\text{Pdim}(\Gcal)+1)\log_2(\text{Pdim}(\Gcal)+1)$.
\end{lemma}
\proof{Proof of \cref{lemma: star hull vc}}
For any positive integer $m$, define the growth functions for $\Gcal$ and $\Gcal^\star$ as
\begin{align*}
    & \grow_\Gcal(m) = \max_{x_1, \dots, x_m, s_1, \dots, s_m} \abs{ \braces{ \prns{ \ind\bracks{g(x_i)\ge s_i} }_{i=1}^m  : g\in \Gcal  }  },\\
    & \grow_{\Gcal^\star}(m) = \max_{x_1, \dots, x_m, r_1, \dots, r_m} \abs{ \braces{ \prns{ \ind\bracks{a g(x_i)\ge r_i} }_{i=1}^m  : g\in \Gcal, a\in[0,1]  }  }.
\end{align*}

\paragraph{Step I: quantifying the relationship between $\grow_\Gcal(m)$ and $\grow_{\Gcal^\star}(m)$.}
Fix $g\in \Gcal$, $x=(x_1, \dots, x_m)$, and $r=(r_1, \dots, r_m)$. As $a$ goes from $0$ to $1$, each coordinate of $\prns{ \ind\bracks{ag(x_1)\ge r_1}, \dots, \ind\bracks{a g(x_m)\ge r_m} }$ can flip at most once, so this vector can take at most $m+1$ distinct values over $a\in [0,1]$. Therefore, we can find $a_{g,1}, \dots, a_{g,m+1} \in (0,1]$ such that
\begin{align}\label{eq: m+1 values}
   \braces{ \prns{ \ind\bracks{a g(x_i)\ge  r_i} }_{i=1}^m  : a\in[0,1] } = \braces{ \prns{ \ind\bracks{a_{g,j} g(x_i)\ge r_i} }_{i=1}^m  : j\in[m+1]  }.
\end{align}
Now we allow $g$ to vary, and \cref{eq: m+1 values} implies that
\begin{align}
   \abs{ \braces{ \prns{ \ind\bracks{a g(x_i)\ge  r_i} }_{i=1}^m  : g\in \Gcal, a\in[0,1]  }  }  
   = & \abs{ \braces{ \prns{ \ind\bracks{a_{g,j} g(x_i)\ge  r_i} }_{i=1}^m  : g\in \Gcal, j\in[m+1]  }  }  \notag\\
   \le & \sum_{j=1}^{m+1} \abs{ \braces{ \prns{ \ind\bracks{a_{g,j} g(x_i)\ge  r_i} }_{i=1}^m  : g\in \Gcal}  }\notag \\
   = & \sum_{j=1}^{m+1} \abs{ \braces{ \prns{ \ind\bracks{ g(x_i)\ge  r_i/ a_{g,j}} }_{i=1}^m  : g\in \Gcal}  } \notag\\
   \le & (m+1)\grow_\Gcal(m). \label{eq: growth m+1}
\end{align}
Taking the worst case over $x, r$ in \cref{eq: growth m+1}, we get
\begin{align} \label{eq: growth ineq}
    \grow_{\Gcal^\star}(m) \le (m+1) \grow_\Gcal(m).
\end{align}

\paragraph{Step II: Computing the maximal possible VC-subgraph dimension of $\Gcal^\star$ based on \cref{eq: growth ineq}.}
By the Sauer-Shelah lemma \citep[Proposition 4.18]{wainwright2019high}, for any $m>d$,
\begin{align*}
    \grow_\Gcal(m) \le (m+1)^{\text{Pdim}(\Gcal)},
\end{align*}
which, combined with \cref{eq: growth ineq}, implies that
\begin{align}
   \grow_{\Gcal^\star}(m) \le  (m+1)^{\text{Pdim}(\Gcal)+1}.
\end{align}
Taking $m^* = 4(\text{Pdim}(\Gcal)+1)\log_2(\text{Pdim}(\Gcal)+1)$, we can easily verify that
\begin{align*}
    2^{m^*} > (m^*+1)^{\text{Pdim}(\Gcal)+1} \ge \grow_{\Gcal^\star}(m^*).
\end{align*}
Since the VC-subgraph dimension of $\Gcal^\star$ is the maximal solution to the equality $\grow_{\Gcal^\star}(m) = 2^m$, we know its VC-subgraph dimension is smaller than $m^*$, concluding the proof.
\endproof

\subsubsection{Proofs of \cref{lemma: critical radius,prop: natarajan}}

\proof{Proof of \cref{lemma: critical radius}}

Define
\[
\Psi(t) = \frac{1}{5}\exp(t^2).
\]
Note that whenever $\expect\Psi(|W|/w) \le 1$ for some random variable $W$, we have by Markov's inequality that
\begin{align}
    \pr(|W|>t) \le 5 \exp(-t^2/w^2), \notag\\
    \expect|W| = \int_0^\infty \pr(|W|>t) dt \le 5w. \label{eq: ec 2}
\end{align}
Throughout the proof, we condition on the event that $\Gcal^{(-k)}$ has VC-subgraph dimension $\eta$. We finish the proof in three steps. 

\paragraph{Step I: Critical radius for empirical Rademacher complexity.}
Define the localized empirical Rademacher complexity
\begin{align*}
        \hat{\Rcal}_n(\Gcal^{(-k)}, r) = \expect_\epsilon\bracks{\sup_{g\in \mathcal{G}, \norm{g}_n\le r} \abs{\frac{1}{n} \sum_{i=1}^n \epsilon_i g(\Data_i)}},
\end{align*}
where $\epsilon_1, \dots, \epsilon_n$ are i.i.d. Rademacher random variables, and $\norm{g}_n = \sqrt{\sum_{i=1}^n g^2(\Data_i)/n}$.
Let $\hat{r}_n^*$ be the smallest positive solution to $\hat{\Rcal}_n(\Gcal^{(-k)}, r)\le r^2/32$. In what follows, we show that there exists a universal constant $C$ such that
\begin{align}\label{eq: r_hat_n}
    \pr\prns{\hat{r}_n^*\le \Tilde{C}\sqrt{\frac{\eta \log(n+1)}{n}}} = 1.
\end{align}

For any $g\in \Gcal^{(-k)}$, define set
\begin{align*}
    \mathbf{G} = \braces{\prns{g(\Data_1), \dots, g(\Data_n)}: g\in\Gcal^{(-k)}, \norm{g}_n \le r}.
\end{align*}
Let $D(t, \mathbf{G})$ be the $t$-packing number of $\mathbf{G}$ and $N(t, \mathbf{G})$ be the $t$-covering number.
Note that $\norm{\mathbf{g}}\le \sqrt{n}r$ for all $\mathbf{g} \in \mathbf{G}$. By Theorem 3.5 in \cite{pollard1990empirical},
\begin{align*}
    \expect_\epsilon \Psi \prns{\frac{1}{J} \sup_{g\in \Gcal^{(-k)}, \norm{g}_n\le r}\abs{\sum_{i=1}^n \epsilon_i g(\Data_i)} }\le 1,
\end{align*}
where
\begin{align*}
    J = 9\int_0^{\sqrt{n} r} \sqrt{\log D(t, \mathbf{G})} dt.
\end{align*}
So by \cref{eq: ec 2},
\begin{align*}
    \hat{\Rcal}_n (\Gcal^{(-k)}, r) \le \frac{5}{n} J.
\end{align*}
Consider the function class
\begin{align*}
    (\Gcal^{(-k)})' = \braces{g: g\in \Gcal^{(-k)}, \norm{g}_n \le r}.
\end{align*}
Note that $\sqrt{n}r$ is the envelope of $(\Gcal^{(-k)})'$ on $\Data_1, \dots, \Data_n$. Applying \citep[Theorem 2.6.7]{van1996weak} gives
\begin{align*}
    D(\sqrt{n}rt, \mathbf{G}) & \le N(\sqrt{n}rt/2, \mathbf{G})\\
    & \le \Tilde{C}(\eta+1)(16e)^{\eta+1}\prns{\frac{4n}{t^2}}^\eta
\end{align*}
for a universal constant $\Tilde{C}$. Thus,
\begin{align*}
    J = & 9\sqrt{n} r\int_0^1 \sqrt{\log D(\sqrt{n}r t, \mathbf{G})} dt\\
    \le &  9\sqrt{n} r \int_0^1 \sqrt{\log C + \log (\eta + 1) + (\eta +1)\log (16 e) + \eta \log n + \eta \log 4 - 2\eta \log t} dt\\
    \le & 9 \int_0^1 \sqrt{2\log C + 15 - 3 \log t} dt\sqrt{\eta \log(n+1) n } r,
\end{align*}
where $\int_0^1 \sqrt{2\log C + 15 - 3\log t} dt< \infty$. We then obtain that for a (different) universal constant $\Tilde{C}$,
\begin{align*}
    \hat{\Rcal}_n(\Gcal^{(-k)}, r) \le \frac{\Tilde{C}}{32}\sqrt{\frac{\eta \log(n+1)}{n}} r.
\end{align*}
Therefore, for any samples $\{\Data_i\}_{i=1}^n$, any $\hat{r}_n \ge C\sqrt{\eta \log(n+1)/n}$ is a valid solution to $\hat{\Rcal}_n(\Gcal^{(-k)}, r)\le r^2/32$, which implies \cref{eq: r_hat_n}.

\paragraph{Step II: Critical radius for Rademacher complexity.} Let $r^*_n$ be the smallest positive solution to the inequality $\Rcal_n(\Gcal^{(-k)}, r)\le r^2/32$. We now bound $r^*_n$.

For any $t>0$, define the random variable
\begin{align*}
    W_n(t) = \expect_\epsilon\bracks{\sup_{g\in \mathcal{G}^{(-k)}, \norm{g}_2\le t} \abs{\frac{1}{n} \sum_{i=1}^n \epsilon_i g(\Data_i)}},
\end{align*}
so that $\Rcal_n(\Gcal^{(-k)}, r) = \expect_\Data[W_n(r)]$ by construction. Define the events
\begin{align*}
    & \Ecal_3(t) = \braces{\abs{W_n(t) - \Rcal_n(\Gcal^{(-k)}, t)} \le \frac{r_n^* t}{112}},\\
    & \Ecal_4 = \braces{\sup_{g\in\Gcal^{(-k)}} \frac{\abs{\norm{g}^2_n - \norm{g}^2_2}}{\norm{g}_2^2 + (r_n^*)^2} \le \frac{1}{2}}.
\end{align*}
Following the proof of Lemma EC.12 in \cite{hu2022fast},
\begin{align*}
    \pr\prns{\frac{r_n^*}{5} \le \hat{r}_n^* \le 3r_n^* } \ge \pr\prns{2\Ecal_3(r_n^*) \cap \Ecal_3(7 r_n^*) \cap \Ecal_4}.
\end{align*}
Lemma EC.10 in \cite{hu2022fast} implies that
\begin{align*}
    \pr\prns{\Ecal_4^c} \le 2e^{-\Tilde{c}_1 n (r_n^*)^2}
\end{align*}
for some universal constant $\Tilde{c}_1 >0$. Moreover, for any $\zeta\ge 1$, we have $\Rcal_n(\Gcal^{(-k)}, \zeta r^*_n) \ge \Rcal_n(r^*_n) \ge (r^*_n)^2/32$. By Theorem 16 in \cite{boucheron2003concentration},
\begin{align*} 
    \pr\prns{\Ecal_3^c(\zeta r^*_n)} \le 2e^{-\Tilde{c}_2 n (r_n^*)^2}
\end{align*}
for some universal constant $\Tilde{c}_2>0$. Combining all pieces we have
\begin{align} \label{eq: bounding r star}
    \pr\prns{\frac{r_n^*}{5} \le \hat{r}_n^* \le 3r_n^* } \ge  1-6e^{-( \Tilde{c}_1\wedge \Tilde{c}_2) n (r_n^*)^2}.
\end{align}
By step I in the proof, $ \pr\prns{\hat{r}_n^*\le \Tilde{C}_0\sqrt{\eta \log(n+1)/n}} = 1$ for some constant $\Tilde{C}_0$. Let $\Tilde{C}> 5\Tilde{C}_0$ be a constant such that $2^{-\Tilde{C}(\Tilde{c}_1\wedge \Tilde{c}_2)} < 1/6$. If $r_n^*> C\sqrt{\eta\log (n+1)/n}$, by \cref{eq: bounding r star} we have $ \pr\prns{\hat{r}_n^*> \Tilde{C}_0\sqrt{\eta \log(n+1)/n}} >0$, which leads to contradiction. Thus,
\begin{align*}
    r_n^* \le \Tilde{C}\sqrt{\eta\log(n+1)/n}.
\end{align*}
Finally, any $r \ge \Tilde{C}\sqrt{\eta\log(n+1)/n}$ solves the inequality $\Rcal_n(\Gcal^{(-k)}, r)\le r^2/32< r^2$.
\paragraph{Step III: Checking the conditions.} 
Given that $\tilde{C}\sqrt{\eta\log(n+1)/n}$ satisfies the inequality $\Rcal_n(\Gcal^{(-k)}, r)\le r^2$, it follows from similar arguments that $\tilde{C}\sqrt{2K\eta\log(n+1)/n}$ satisfies the corresponding inequality $\Rcal_{n/K}(\Gcal^{(-k)}, r)\le r^2$. The other two conditions, $3 n\tilde{r}^2/(64K) \ge \log \log_2(1/\tilde{r})$ and $2\exp(-3 n\tilde{r}^2/(64K))) \le \delta/(2K)$, are easily satisfied as long as we take $\Tilde{C}$ big enough.
\endproof

\proof{Proof of \cref{prop: natarajan}} 
Fix $\nuisance$ as any possible nuisance estimator, and define 
\begin{align*}
    \Gcal = \braces{\data \rightarrow\frac{\theta\prns{\data;\nuisance}^\top \prns{\pi(x) - \Tilde{\pi}^*(x)}}{2B\Theta}: \pi\in \Pi_\F}.
\end{align*}
Let $\pdim(\Gcal)$ be the VC-subgraph dimension of $\Gcal$. 

\paragraph{Step I: Bounding $\pdim(\Gcal)$ using $\ndim(\Pi_\F)$.}
Fix any sample $\braces{(\data_i, t_i)}_{i=1}^m$, and assume each $o_i$ includes an $x_i$ component. Let
\begin{align*}
 L\prns{\Pi_\F; \braces{(\data_i, t_i)}_{i=1}^m} = \braces{\prns{\pi(x_1), \dots, \pi(x_m)}: \pi\in \Pi_\F} \subseteq \prns{\Z^\angle}^m 
\end{align*}
be the multiclass label sequences that can be realized on this sample by $\Pi_\F$, and 
\begin{align*}
 L\prns{\Gcal; \braces{(\data_i, t_i)}_{i=1}^m} = \braces{\prns{\ind\bracks{t_i\le \frac{\theta\prns{\data_i;\nuisance}^\top \prns{\pi(x_i) - \Tilde{\pi}^*(x_i)}}{2B\Theta} }}_{i=1}^m : \pi\in\Pi_\F} \subseteq \braces{0,1}^m
\end{align*}
be the binary sequences that can be realized on this sample by the subgraph functions of $\Gcal$.
Since every labeling in $L\prns{\Gcal; \braces{(\data_i, t_i)}_{i=1}^m}$ can be realized by applying a coordinate-wise mapping on a labeling in $L\prns{\Pi_\F; \braces{(\data_i, t_i)}_{i=1}^m}$, we have
\begin{align*}
    \abs{L\prns{\Gcal; \braces{(\data_i, t_i)}_{i=1}^m}} \le \abs{L\prns{\Pi_\F; \braces{(\data_i, t_i)}_{i=1}^m}}.
\end{align*}
Maximizing over all samples,
\begin{align}\label{eq: growth inequality}
    \grow_\Gcal(m) = \sup_{\braces{(\data_i, t_i)}_{i=1}^m} \abs{ L\prns{\Gcal; \braces{(\data_i, t_i)}_{i=1}^m}} \le \sup_{\braces{x_i}_{i=1}^m} \abs{ L\prns{\Pi_\F; \braces{(\data_i, t_i)}_{i=1}^m}}= \grow_{\Pi_\F}(m).
\end{align}
By Natarajan's lemma \citep[Lemma 29.4]{shalev2014understanding}, 
\begin{align}\label{eq: growth ndim}
    \grow_{\Pi_\F}(m) \le m^{\ndim(\Pi_\F)} \abs{Z^\angle}^{2\ndim(\Pi_\F)}.
\end{align}
From \cref{eq: growth inequality,eq: growth ndim}, any solution to the following inequality serves as an upper bound on $\pdim(\Gcal)$:
\begin{align} \label{eq: m ineq}
    2^m > m^{\ndim(\Pi_\F)} \abs{Z^\angle}^{2\ndim(\Pi_\F)}.
\end{align}
It is easy to verify that
\begin{align*}
    m^* = 4\ndim(\Pi_\F)\prns{\log_2\prns{\abs{\Z^\angle}} + \log_2\prns{\ndim(\Pi_\F)} +1}
\end{align*}
is a valid solution to \cref{eq: m ineq}. Thus, we conclude that
\begin{align} \label{eq: pdim gcal}
  \pdim(\Gcal) \le   4\ndim(\Pi_\F)\prns{\log_2\prns{\abs{\Z^\angle}} + \log_2\prns{\ndim(\Pi_\F)} +1}.
\end{align}

\paragraph{Step II: Bounding the VC-Subgraph Dimension for $\Gcal^{(-k)}$.}
Let $\Gcal^\star=\braces{ag: g\in\Gcal, a\in [0,1]}$ be the star-hull of $\Gcal$. By \cref{lemma: star hull vc}, $\Gcal^\star$ has VC-dimension at most $4(\text{Pdim}(\Gcal)+1)\log_2(\text{Pdim}(\Gcal)+1)$, which is of order $\tilde{O}\prns{\ndim(\Pi_\F)\log\prns{\abs{\Z^\angle}}}$ given \cref{eq: pdim gcal}.

Finally, since the above arguments hold for any $\nuisance$, we conclude that $\Gcal^{(-k)}$ has VC-dimension of order $\tilde{O}\prns{\ndim(\Pi_\F)\log\prns{\abs{\Z^\angle}}}$ almost surely.
\endproof

\subsection{\cref{sec: theory bandit} Omitted Proofs}

\proof{Proof of \cref{example: DR}}
Because $\pi(x) - \Tilde{\pi}^* \in \operatorname{span}(\Z)$,  for the DM score we have 
\begin{align*}
&\expect_{X,Z,C}\bracks{\prns{\theta_{DM}(X,Z,C;  f^*, \Sigma^*)- \theta_{DM}(X, Z, C;\hat{f}, \hat{\Sigma})}^\top \prns{\pi(X)- \Tilde{\pi}^*(X)}  } \\
\le & \expect_{X,Z,C}\bracks{\operatorname{Proj}_{\operatorname{span}(\Z)}\prns{\theta_{DM}(X,Z,C; f^*, \Sigma^*)- \theta_{DM}(X, Z, C;\hat{f}, \hat{\Sigma})}^\top \prns{\pi(X)- \Tilde{\pi}^*(X)}} \\
\le& 2B \braces{\expect_X\bracks*{\magd*{\operatorname{Proj}_{\operatorname{span}(\Z)}(\hat{f}(X) - f^*(X))}^2}}^{1/2} = O(\chi_f(n,\delta)).
\end{align*}

For the ISW score, we have  
\begin{align*}
&\expect_{X,Z,C}\bracks{\prns{\theta_{ISW}(X,Z,C; f^*, \Sigma^*)- \theta_{ISW}(X, Z, C;\hat{f}, \hat{\Sigma})}^\top \prns{\pi(X)- \Tilde{\pi}^*(X)}} \\
\le& 2B \braces{\expect_X\bracks*{\magd*{(\hat{\Sigma}^\dagger(X) - (\Sigma^*)^\dagger(X))\Sigma^*(X)}_{\text{Fro}}^2}}^{1/2} = O(\chi_\Sigma(n,\delta)).
\end{align*}
For the doubly robust score, we can easily get 
\begin{align*}
&\expect_{X,Z,C}\bracks{\prns{\theta_{DR}(X,Z,C; f^*, \Sigma^*)- \theta_{DR}(X, Z, C;\hat{f}, \hat{\Sigma})}^\top \prns{\pi(X)- \Tilde{\pi}^*(X)}  } \\
    =& \expect_{X,Z,C}\bracks{\prns{\pi(X)- \Tilde{\pi}^*(X)}^\top   \prns{I - \hat\Sigma^+(X)\Sigma^*(X)}(\hat f(X) - f^*(X))} \\
    =& \expect_{X,Z,C}\bracks{\prns{\pi(X)- \Tilde{\pi}^*(X)}^\top   \prns{(\Sigma^*)^\dagger(X) - \hat\Sigma^\dagger(X)}\Sigma^*(X)(\hat f(X) - f^*(X))} \\
    =& \expect_{X,Z,C}\bracks{\prns{\pi(X)- \Tilde{\pi}^*(X)}^\top   \prns{(\Sigma^*)^\dagger(X) - \hat\Sigma^\dagger(X)}\Sigma^*(X)\operatorname{Proj}_{\operatorname{Span}(\Z)}(\hat f(X) - f^*(X))} \\
    \lesssim  & \braces{\expect_X\bracks*{\magd*{(\hat{\Sigma}^\dagger(X) - (\Sigma^*)^\dagger(X))\Sigma^*(X)}_{\text{Fro}}^2}}^{1/2} \braces{\expect_X\bracks*{\magd*{\operatorname{Proj}_{\operatorname{span}(\Z)}(\hat{f}(X) - f^*(X))}^2}}^{1/2} \\
    =& O(\chi_\Sigma(n,\delta)\chi_\Sigma(n,\delta)).
\end{align*}
Here the second equation holds because $\pi(x) - \Tilde{\pi}^*(x)$ belongs to the linear span of $\Z$, but  $I - (\Sigma^*)^\dagger(x)\Sigma^*$ is orthogonal to the linear span of $\Z$, as we already argued in the proof of \cref{prop: unbiased}. The third equation holds because the column space of $\Sigma^*(X)$ is identical to $\operatorname{span}(\Z)$ according to the coverage assumption. 
\endproof

\subsection{\cref{sec: main proof} Omitted Proofs} \label{sec: lemma proof}

\subsubsection{Supporting Lemmas} 

\begin{lemma}[Talagrand's inequality, \cite{talagrand1996new,bousquet2003concentration,sen2018gentle}]\label{lemma: talagrand}
Let $U_i, i=1,\dots, n$ be independent $\mathcal{U}$-valued random variables. Let $\mathcal{H}$ be a countable family of measurable real-valued functions on $\mathcal{U}$ such that $\norm{h}_\infty \le v$ and $\expect[h(U_1)] = \dots = \expect[h(U_n)] = 0$, for all $h\in\mathcal{H}$. Define
\begin{align*}
    v_n = 2v\expect\bracks{\sup_{h\in\mathcal{H}} \abs{\sum_{i=1}^n h(U_i)} } + \sum_{i=1}^n \sup_{h\in\mathcal{H}} \expect\bracks{h^2(U_i)}.
\end{align*}
Then, for all $t\ge 0$,
\begin{align*}
    \pr\prns{ \sup_{h\in\mathcal{H}} \abs{\sum_{i=1}^n h(U_i)} \ge \expect\bracks{\sup_{h\in\mathcal{H}} \abs{\sum_{i=1}^n h(U_i)}} +t } \le \exp\prns{\frac{-t^2}{2v_n + 2tv/3}}.
\end{align*}
\end{lemma}

\begin{lemma} \label{lemma: uniform bound}
Fix function $\nuisance$ independent of $\braces{\Data_i}_{i=1}^n$ such that $\norm{\theta(\data;\nuisance)} \le \Theta$ for all $\data$. Define the function class
\begin{align*}
    \Gcal = \braces{\data \rightarrow \frac{\theta(\data; \nuisance)^\top \prns{\pi(x) - \Tilde{\pi}^*(x)} \rho}{2B\Theta} :  \pi\in\Pi_\F, \rho\in[0,1]}.
\end{align*}
Let $\Tilde{r}$ be any solution to the inequality $\Rcal_n(\Gcal, r)\le r^2$ satisfying $3 n\Tilde{r}^2/64 \ge \log \log_2(1/\Tilde{r})$.
    Then we have
    \begin{align*}
       \pr\prns{\sup_{g\in\mathcal{G}} \frac{\abs{(\expect_n - \expect_\Data) g}}{\norm{g}_2 + \Tilde{r}} \ge 6 \Tilde{r}} \le 2\exp\prns{-\frac{3}{64} n\Tilde{r}^2},
    \end{align*}
    where
    \begin{align*}
        \expect_n (g) = \frac{1}{n}\sum_{i=1}^n g(\Data_i).
    \end{align*}
\end{lemma}
\proof{Proof of \cref{lemma: uniform bound}}
    When $\sup_{g\in \mathcal{G}} \abs{(\expect_n - \expect_\Data) g} / (\norm{g}_2 + \Tilde{r}) > 6 \Tilde{r}$, one of the following two events must hold true:
    \begin{align*}
       &  \mathcal{E}_1 = \braces{\abs{(\expect_n - \expect_\Data) g} \ge  6 \Tilde{r}^2 \text{ for some } g\in \mathcal{G} \text{ such that } \norm{g}_2\le \Tilde{r}}, \\
       &  \mathcal{E}_2 = \braces{\abs{(\expect_n - \expect_\Data) g} \ge 6\norm{g}_2\Tilde{r} \text{ for some } g\in \mathcal{G} \text{ such that } \norm{g}_2\ge \Tilde{r}}. \\
    \end{align*}
    Define
    \begin{align*}
       \Zcal_n (r) = \sup_{g\in \Gcal, \norm{g}_2\le r} \abs{(\expect_n - \expect_\Data) g}.
    \end{align*}
    Note that $\norm{g}_2\le r$ implies $\EE[(g-\EE_\Data(g))^2]\le r^2$, and we also have $\norm{g-\EE_\Data(g) }_{\infty} \le 2$. 
    By Talagrand's inequality (\cref{lemma: talagrand}) over the function class $\braces{g-\EE_\Data(g): g\in \Gcal}$, 
    \begin{align*}
        \pr\prns{\Zcal_n(r) \ge \EE[\Zcal_n(r)] + t} \le \exp\prns{-\frac{nt^2}{8\EE[\Zcal_n(r)] + 2r^2 +4t/3}}.
    \end{align*}
    We now bound the expectation $\EE[\Zcal_n(r)]$.
    Since $\Gcal$ is star-shaped\footnote{A function class $\Gcal$ is star-shaped if for any $g\in\Gcal$ and $\rho\in[0,1]$, we have $\rho g \in \Gcal$.}, by \cite[Lemma 13.6]{wainwright2019high}, $r \rightarrow \Rcal_n(\Gcal, r)/r$ is non-increasing.
    Thus, for any $r\ge \Tilde{r}$,
    \begin{align*}
        \EE[\Zcal_n(r)]\le 
        2\Rcal_n(\Gcal, r)
        \le \frac{2 r}{\Tilde{r}} \Rcal_n(\Gcal, \Tilde{r})
        \le 2r \Tilde{r},
    \end{align*}
    where the first inequality comes from a symmetrization argument, the second inequality uses the fact that $r \rightarrow \Rcal_n(\Gcal, r)/r$ is non-increasing, and the third inequality is by definition of $\Tilde{r}$.
    
    The Talagrand's then implies
    \begin{align} \label{eq: talagrand}
        \pr\prns{\Zcal_n(r) \ge 2r \Tilde{r} + t} \le \exp\prns{-\frac{nt^2}{16 r \Tilde{r} + 2r^2 +4t/3}}.
    \end{align}
    We first bound $\pr(\Ecal_1)$. Taking $r=\Tilde{r}$ and $t=4\Tilde{r}^2$ in \cref{eq: talagrand}, we get
    \begin{align*}
        \pr\prns{\Ecal_1} \le & \pr\prns{\Zcal_n(\Tilde{r}) \ge 6\Tilde{r}^2} 
        \le \exp\prns{-\frac{24}{35}n\Tilde{r}^2}.
    \end{align*}
    We now bound $\pr(\Ecal_2)$. Note that
    \begin{align*}
        \pr\prns{\Ecal_2} \le \pr\prns{\Zcal_n(\norm{g}_2) \ge 6 \norm{g}_2 \Tilde{r} \text{ for some } g\in\Gcal, \norm{g}_2 \ge \Tilde{r}}.
    \end{align*}
    Define
    \begin{align*}
        \Gcal_m = \braces{g\in\Gcal: 2^{m-1}\Tilde{r} \le \norm{g}_2\le 2^m \Tilde{r}}.
    \end{align*}
    Since $\norm{g}_2\le 1$, there exists $M\le \log_2(1/\Tilde{r})$ such that
    \begin{align*}
        \Gcal \cap \braces{g: \norm{g}_2\ge \Tilde{r}} \subseteq \cup_{m=1}^M \Gcal_m.
    \end{align*}
    Therefore,
    \begin{align*}
        \pr\prns{\Ecal_2} 
        \le \sum_{m=1}^M \pr\prns{\Zcal_n(\norm{g}_2) \ge 6 \norm{g}_2 \Tilde{r} \text{ for some } g\in\Gcal_m}.
    \end{align*}
    We now bound each term in the summation above. If there exists $g\in \Gcal_m$ such that $\Zcal_n(\norm{g}_2) \ge 6 \norm{g}_2 \Tilde{r}$, then we have 
    \begin{align*}
        \Zcal_n(2^m \Tilde{r}) 
        \ge  \Zcal_n(\norm{g}_2)
        \ge  6 \norm{g}_2 \Tilde{r} 
        \ge  3\cdot 2^m \Tilde{r}^2.
    \end{align*}
    Thus,
    \begin{align*}
       \pr\prns{\Zcal_n(\norm{g}_2) \ge 6 \norm{g}_2 \Tilde{r} \text{ for some } g\in\Gcal_m} 
       \le \pr\prns{\Zcal_n(2^m \Tilde{r}) \ge 3\cdot 2^m \Tilde{r}^2}.
    \end{align*}
    Now, taking $r = 2^m \Tilde{r}$ and $t = 2^m \Tilde{r}^2$ in \cref{eq: talagrand}, we get
    \begin{align*}
        \pr\prns{\Zcal_n(2^m \Tilde{r}) \ge 3\cdot 2^m \Tilde{r}^2} 
        \le \exp\prns{-\frac{3n\Tilde{r}^2}{13\cdot 2^{2-m} + 6} } 
        \le  \exp\prns{-\frac{3}{32} n\Tilde{r}^2}. 
    \end{align*}
    Therefore, if $\frac{3}{64}n\Tilde{r}^2 \ge \log \log_2(1/\Tilde{r})$, then 
    \begin{align*}
        \pr\prns{\Ecal_2} 
        \le M\exp\prns{-\frac{3}{32} n\Tilde{r}^2} 
        \le \exp\prns{-\frac{3}{64} n\Tilde{r}^2}.
    \end{align*}

    Combining the bounds on $\pr(\Ecal_1)$ and $\pr(\Ecal_2)$ leads to the final conclusion.
\endproof

\subsubsection{Proofs of \cref{lemma: uniform bound G2,lemma: inequality,lemma: inequality integer}}  

\proof{Proof of \cref{lemma: uniform bound G2}}
As a direct consequence of \cref{lemma: uniform bound}, 
\begin{align*}
       \pr\prns{\sup_{g\in\mathcal{G}^{(-k)}} \frac{\abs{(\expect_{n_k} - \expect_\Data) g}}{\norm{g}_2 + \Tilde{r}} \ge 6 \Tilde{r}} \le 2\exp\prns{-\frac{3n\Tilde{r}^2}{64K} }.
    \end{align*}

The conclusion then follows directly from the fact that $2\exp\prns{-3 n\Tilde{r}^2/(64K)} \le \delta/(2K)$.

\endproof

\proof{Proof of \cref{lemma: inequality}}
    First, note that
    \begin{align*}
        \frac{\partial}{\partial y} \prns{y - c_1 y^{\frac{\alpha}{2(1+\alpha)}}  - c_2} = 1 - \frac{c_1  \alpha}{2(1+\alpha)} y^{-\frac{2+\alpha}{2+2\alpha}}.
    \end{align*}
    The derivative is strictly increasing in $y$ and is eventually positive. Note the function $y - c_1 y^{\alpha/(2+2\alpha)}  - c_2$ takes a negative value at $y=0$. Then as $y$ increases, the value of $y - c_1 y^{\alpha/(2+2\alpha)}  - c_2$  first decreases and then increases. Therefore, if $y > 0$  satisfies the inequality $y - c_1 y^{\alpha/(2+2\alpha)}  - c_2 \ge 0$, then such $y$ also provides an upper bound on $x$.  
    Hence it is sufficient to show that $y = c_1^{\frac{2\alpha+2}{\alpha+2}} + 2c_2$ satisfies the inequality, or equivalently, 
    \begin{align} \label{eq: original inequality}
       c_1^{\frac{2\alpha+2}{\alpha+2}} + c_2 \ge c_1 \prns{c_1^{\frac{2\alpha+2}{\alpha+2}} + 2c_2}^{\frac{\alpha}{2(1+\alpha)}} .
    \end{align}
    Suppose $\alpha/(2+2\alpha)$ is a rational number. In this case, we can write $\alpha/(2+2\alpha) = m_1/m_2$, where $m_1$ and $m_2$ are positive integers such that $m_2\ge 2 m_1 +1$. \cref{eq: original inequality} is then equivalent to
    \begin{align} \label{eq: rational inequality}
       \prns{c_1^{\frac{m_2}{m_2-m_1}} + c_2}^{m_2} \ge c_1^{m_2} \prns{c_1^{\frac{m_2}{m_2-m_1}} + 2 c_2}^{m_1} .
    \end{align}
    Using the multinomial theorem, we have 
    \begin{align*}
        & \prns{c_1^{\frac{m_2}{m_2-m_1}} + c_2}^{m_2} - c_1^{m_2} \prns{c_1^{\frac{m_2}{m_2-m_1}} + 2 c_2}^{m_1}  \\ 
       > & \sum_{i=0}^{m_1} \prns{\frac{m_2!}{(i+m_2-m_1)! (m_1-i)!} - \frac{m_1! 2^{m_1 - i}}{(m_1-i)! i!}} c_1^{\frac{i m_2 }{m_2-m_1}+m_2} c_2^{m_1 - i}.
    \end{align*}
    Since $m_2 \ge 2m_1 + 1$, we have for any $i=0, \dots, m_1$,
    \begin{align*}
        \frac{m_2!}{(i+m_2-m_1)! (m_1-i)!} \ge  \frac{m_1! 2^{m_1 - i}}{(m_1-i)! i!}.
    \end{align*}
    This can be proved by showing the ratio of LHS over RHS is larger than $1$.
    Hence, \cref{eq: original inequality} holds true when $\alpha/(2+2\alpha)$ is a rational number.

    Finally, note that
    \begin{align*}
       c_1^{\frac{2\alpha+2}{\alpha+2}} + c_2 - c_1 \prns{c_1^{\frac{2\alpha+2}{\alpha+2}} + 2c_2}^{\frac{\alpha}{2(1+\alpha)}}
    \end{align*}
    is continuous in $\alpha$. Since any real number $\alpha$ can be viewed as the limit of a sequence of rational numbers and \cref{eq: original inequality} holds for all rational numbers, it also holds true for all $\alpha>0$. 
\endproof

\proof{Proof of \cref{lemma: inequality integer}}
Let $2(\alpha+1)/\alpha = m$, where $m$ is an integer by assumption.
Using similar arguments as in the proof of \cref{lemma: inequality}, it is sufficient to show that for $w = c_1^{\frac{2\alpha+2}{\alpha +2}} r^{\frac{2\alpha+2}{\alpha+2}} + \frac{2\alpha+2}{\alpha+2} c_1 y^{\frac{\alpha}{2\alpha+2}} r + \frac{2\alpha+2}{\alpha+2} c_2 r^2 + \frac{2\alpha+2}{\alpha+2} z + y$ and $c_2' = c_1 y^{\frac{\alpha}{2(1+\alpha)}} r + c_2 r^2 + z + y$, we have 
\begin{align*}
    w - c_1 w^{\frac{\alpha}{2(1+\alpha)}}r - c_2' \ge 0.
\end{align*}
This is equivalent to showing 
    \begin{align*}
        & c_1^{\frac{2\alpha+2}{\alpha +2}} r^{\frac{2\alpha+2}{\alpha+2}} + \frac{2\alpha+2}{\alpha+2} c_1 y^{\frac{\alpha}{2\alpha+2}} r + \frac{2\alpha+2}{\alpha+2} c_2 r^2 + \frac{2\alpha+2}{\alpha+2} z + y \\
        \ge & c_1 \prns{c_1^{\frac{2\alpha+2}{\alpha +2}} r^{\frac{2\alpha+2}{\alpha+2}} + \frac{2\alpha+2}{\alpha+2} c_1 y^{\frac{\alpha}{2\alpha+2}} r + \frac{2\alpha+2}{\alpha+2} c_2 r^2 + \frac{2\alpha+2}{\alpha+2} z + y}^{\frac{\alpha}{2(1+\alpha)}} r + c_1 y^{\frac{\alpha}{2(1+\alpha)}} r + c_2 r^2 + z + y.
    \end{align*}
    Since $\alpha = 2/(m-2)$, the above inequality is equivalent to 
    \begin{align*}
        & \prns{c_1^{\frac{m}{m-1}} r^{\frac{m}{m-1}} +\frac{1}{m-1} c_1 y^{\frac{1}{m}} r + \frac{1}{m-1} c_2 r^2 + \frac{1}{m-1}z}^m \\ 
        \ge & c_1^{m} r^m \prns{c_1^{\frac{m}{m-1}} r^{\frac{m}{m-1}} + \frac{m}{m-1} c_1 y^{\frac{1}{m}} r+ \frac{m}{m-1} c_2 r^2 + \frac{m}{m-1}z + y}.
    \end{align*}
    Using the multinomial theorem, it is easy to see that the expansion of LHS contains all terms on the RHS (plus additional positive terms). This finishes proving our conclusion.
\endproof

\subsection{\cref{sec: semi} Omitted Proofs}
\proof{Proof of \cref{coro: dgp semi}}
The ignorability condition in \cref{assump: dgp} requires $\mathbb{E}[Y_{ij} \mid Z_i, X_i] = \mathbb{E}[Y_{ij} \mid X_i]$ almost surely. This implies that 
\begin{align*}
    \mathbb{E}[\tilde Y_{ij} \mid Z_{ij} = 1, X_i] = \mathbb{E}[Y_{ij} \mid Z_{ij} = 1, X_i] = \mathbb{E}[Y_{ij} \mid X_i] = f^*(X_i). 
\end{align*}
Moreover, each diagonal entry of $\Sigma^*(X_i)$ is a propensity score: the $(j, j)$-th entry of $\Sigma^*(X_i)$ is the propensity score $e_j^*(X_i)$ for every $j\in \{1, \dots, d\}$. For any feasible component $j$, let $z$ be a vector with $z_j = 1$ and $z_k = 0$ for all $k \neq j$. Then the coverage condition in \cref{assump: dgp} requires $z^\top \Sigma^*(X_i)z > 0$ almost surely. This implies that $e_j^*(X_i) > 0$ almost surely.
\endproof

\proof{Proof of \cref{lemma: theta functions semi}}
We demonstrate that each score function can be used to recover the expected cost function $V(\pi) = \mathbb{E}\left[\sum_{j=1}^d Y_{ij}\pi_j(X_i)\right] = \mathbb{E}\left[\sum_{j=1}^d f_j^*(X_i)\pi_j(X_i)\right]$. For simplicity, we assume that all components are feasible. If not, we can simply exclude all infeasible components in the summation below. 

For the DM score, we have: 
$$
\mathbb{E}\left[\sum_{j=1}^d f_j^*(X_i)\pi_j(X_i)\right] = \mathbb{E}\left[\sum_{j=1}^d \mathbb{E}[Y_{ij}|X_i]\pi_j(X_i)\right] = \mathbb{E}\left[\mathbb{E}\left[\sum_{j=1}^d Y_{ij}\pi_j(X_i) \bigg| X_i\right]\right] = V(\pi).
$$

For the IPW score, we have: 
\begin{align*}
    \mathbb{E}\left[\sum_{j=1}^d \frac{\tilde{Y}_{ij} Z_{ij}}{e^*_j(X_i)} \pi_j(X_i)\right]
    &= \mathbb{E}\left[\sum_{j=1}^d \frac{Y_{ij} Z_{ij}}{e^*_j(X_i)} \pi_j(X_i)\right] \\
    &= \mathbb{E}\left[ \sum_{j=1}^d \frac{\pi_j(X_i)}{e^*_j(X_i)} \mathbb{E}\left[ Y_{ij} Z_{ij} \mid X_i \right] \right] && \text{(Law of Iterated Expectations)}\\
    &= \mathbb{E}\left[ \sum_{j=1}^d \frac{\pi_j(X_i)}{e^*_j(X_i)} \mathbb{E}[Y_{ij} \mid Z_{ij} = 1, X_i] \mathbb{E}[Z_{ij} \mid X_i] \right] && \text{(Unconfoundedness)} \\
    &= \mathbb{E}\left[ \sum_{j=1}^d \frac{\pi_j(X_i)}{e^*_j(X_i)} f_j^*(X_i) e_j^*(X_i) \right] && \text{(By \cref{coro: dgp semi})}  \\
    &= \mathbb{E}\left[ \sum_{j=1}^d \pi_j(X_i) f_j^*(X_i) \right] = V(\pi).
\end{align*}

For DR score, we can split it into two parts:
$$
\mathbb{E}\left[\sum_{j=1}^d \left( \frac{(\tilde{Y}_{ij} - f^*_j(X_i)) Z_{ij}}{e^*_j(X_i)} + f^*_j(X_i) \right) \pi_j(X_i)\right]
= \underbrace{\mathbb{E}\left[\sum_{j=1}^d \frac{(\tilde{Y}_{ij} - f^*_j(X_i)) Z_{ij}}{e^*_j(X_i)} \pi_j(X_i)\right]}_{\text{Term A}} + \underbrace{\mathbb{E}\left[\sum_{j=1}^d f^*_j(X_i) \pi_j(X_i)\right]}_{\text{Term B}}.
$$
Term B is equal to $V(\pi)$ by the logic of the direct method. We now show that Term A is zero.
\begin{align*}
    \text{Term A} = \mathbb{E}\left[ \sum_{j=1}^d \frac{\pi_j(X_i)}{e^*_j(X_i)} \mathbb{E}\left[ (\tilde{Y}_{ij} - f^*_j(X_i)) Z_{ij} \mid X_i \right] \right].
\end{align*}
The inner conditional expectation for any given $j$ is:
\begin{align*}
    \mathbb{E}\left[ (\tilde{Y}_{ij} - f^*_j(X_i)) Z_{ij} \mid X_i \right]
    &= \mathbb{E}\left[ \tilde{Y}_{ij} Z_{ij} \mid X_i \right] - \mathbb{E}\left[ f^*_j(X_i) Z_{ij} \mid X_i \right] \\
    &= \mathbb{E}[Y_{ij} \mid Z_{ij} = 1, X_i] \mathbb{E}[Z_{ij} \mid X_i] - f^*_j(X_i) \mathbb{E}[Z_{ij} \mid X_i]  \\
    &= f_j^*(X_i) e_j^*(X_i) - f_j^*(X_i) e_j^*(X_i) = 0.
\end{align*}
Since this inner term is zero for all $j$, Term A is zero. This finishes the proof for the DR score.
\endproof

\proof{Proof of \cref{nuisance rate semi}}
The DM score analysis follows exactly the same steps as in the proof of \cref{example: DR}.

For the IPW score, we have  
\begin{align*}
&\expect_{X,Z, \tilde Y}\bracks{\sum_{\text{feasible } j} \prns{\frac{Z_{ij}\tilde Y_{ij} }{e^*_j(X_i)}  -\frac{Z_{ij}\tilde Y_{ij} }{\hat e_j(X_i)}}\prns{\pi_j(X_i)- \Tilde{\pi}^*_j(X_i)}  } \\
    =& \expect_{X,Z, Y}\bracks{\sum_{\text{feasible } j} \prns{\frac{1 }{e^*_j(X_i)}  -\frac{1 }{\hat e_j(X_i)}}Z_{ij} Y_{ij}\prns{\pi_j(X_i)- \Tilde{\pi}_j^*(X_i)}  }\\
    = & O\prns{\prns{\mathbb{E}_X\bracks{\sum_{\text{feasible } j}\prns{{1}/{\hat e_j(X)} - {1}/{e^*_j(X)}}^2}}^{1/2}} = O(\chi_e(n, \delta)).
\end{align*}

For the DR score, we have
\begin{align*}
&\expect_{X,Z, \tilde Y}\bracks{\sum_{\text{feasible } j} \prns{f^*_j(X_i) +\frac{Z_{ij}\prns{\tilde Y_{ij} - f^*_j(X_i)}}{e^*_j(X_i)}  -\hat f_j(X_i) -\frac{Z_{ij}\prns{\tilde Y_{ij} - \hat f_j(X_i)}}{\hat e_j(X_i)}}\prns{\pi_j(X_i)- \Tilde{\pi}^*_j(X_i)}  } \\
    =& \sum_{\text{feasible } j} \expect_{X,Z, Y}\bracks{ \prns{f^*_j(X_i)  -\hat f_j(X_i) -\frac{Z_{ij}\prns{Y_{ij} - \hat f_j(X_i)}}{\hat e_j(X_i)}}\prns{\pi_j(X_i)- \Tilde{\pi}^*_j(X_i)}  } .
\end{align*}
Note that
\begin{align*}
& \sum_{\text{feasible } j} \expect_{X,Z, Y}\bracks{ \prns{\frac{Z_{ij}\prns{Y_{ij} - \hat f_j(X_i)}}{\hat e_j(X_i)}}\prns{\pi_j(X_i)- \Tilde{\pi}^*_j(X_i)}  } \\
= & \sum_{\text{feasible } j} \expect_{X,Z, Y}\bracks{\prns{\frac{\pi_j(X_i)- \Tilde{\pi}^*_j(X_i)}{\hat e_j(X_i)}} \expect\bracks{ Z_{ij}\prns{Y_{ij} - \hat f_j(X_i)} \mid X_i } }\\
= & \sum_{\text{feasible } j} \expect_{X,Z, Y}\bracks{\prns{\frac{\pi_j(X_i)- \Tilde{\pi}^*_j(X_i)}{\hat e_j(X_i)}} \prns{f^*_j(X_i) - \hat f_j(X_i)}e^*_j(X_i) }.
\end{align*}
Putting all pieces together, we get
\begin{align*}
&\expect_{X,Z, \tilde Y}\bracks{\sum_{\text{feasible } j} \prns{f^*_j(X_i) +\frac{Z_{ij}\prns{\tilde Y_{ij} - f^*_j(X_i)}}{e^*_j(X_i)}  -\hat f_j(X_i) -\frac{Z_{ij}\prns{\tilde Y_{ij} - \hat f_j(X_i)}}{\hat e_j(X_i)}}\prns{\pi_j(X_i)- \Tilde{\pi}^*_j(X_i)}  } \\
= & \sum_{\text{feasible } j} \expect_{X,Z, Y}\bracks{ \prns{\frac{1}{e^*_j(X_i)} - \frac{1}{\hat e_j(X_i)}} \prns{f^*_j(X_i)  -\hat f_j(X_i) } e^*_j(X_i)\prns{\pi_j(X_i)- \Tilde{\pi}^*_j(X_i)}  }  \\
    =& O\prns{\chi_f(n, \delta)\chi_e(n, \delta)}.
\end{align*}
\endproof

\section{Naive Extensions of Offline Contextual Bandit Learning} \label{sec: ipw}
In this section, we show some alternative approaches to solve the bandit-feedback stochastic shortest path problem in \cref{sec: numerical}. These approaches can be viewed as naive extensions of offline contextual bandit learning with discrete actions. 

Specifically, consider the feasible paths $z_1, \dots, z_m$ for $m = 70$. We can view them as separate discrete actions, and a feasible decision policy $\pi$ is a mapping from the covariates $X$ to one of the $m = 70$ feasible paths. 
We now adopt one-hot-encoding for the decisions. Consider a simplex $\Z^{\text{simplex}}$ in $\R{m}$ and zero-one vector $\breve{z} \in \{0, 1\}^m$ that takes the value $1$ on one and only one of its coordinate. Each feasible decision in $z_1, \dots, z_m$ corresponds to one vector $\breve{z}$, e.g., the decision $z_j$ corresponds to the $\breve z$ vector whose $j$-th entry is $1$ and other entries are all $0$.

Our decision-making problem restricted to the feasible decisions $z_1, \dots, z_m$ can be written as follows:
\begin{align*}
\min_{\breve z \in \Z^{\text{simplex}}} \sum_{j=1}^m \breve z_j \Eb{C \mid Z = z_j, X = x} \iff \min_{\breve z \in \Z^{\text{simplex}}} \breve z^\top \breve f^*(x),
\end{align*}
where $\breve f^*(x) = (\Eb{C \mid Z = z_1, X = x}, \dots, \Eb{C \mid Z = z_m, X = x})^\top$. For any given $x$, the resulting decision will be an one-hot vector that corresponds to an optimal decision at $x$ (which, under \cref{assump: dgp}, can be equivalently given by $\pi_{f^*}(x)$ for an plug-in policy in \cref{eq: plug in policy} at the true $f^*(x) = \Eb{Y \mid X = x}$). In this formulation, we view the decisions $z_1, \dots, z_m$ as separate discrete actions, and we do not take into account the original linear structure of the decision cost in learning the optimal policy. 

We can easily adapt the bandit-feedback ETO to this new formulation. Specifically, we can first construct an estimator $\hat{\breve{f}}$ for the $\breve f^*$ function. For offline bandit learning with discrete actions, this would usually be implemented by regressing the observed total cost $C$ with respect to the covariates $X$, within each subsample for each of the feasible decisions respectively.
Given the estimator  $\hat{\breve{f}}$, we can then solve the optimization problem $\min_{\breve z \in \Z^{\text{simplex}}} \breve z^\top \hat{\breve f}^*(x)$. We finally inspect which coordinate of the resulting solution is equal to $1$ and choose it as the decision. 

To adapt the IERM approach, we similarly define the plug-in policy for any given hypothesis $\breve f(x): \R{p} \to \R{m}$ for the function $\breve f^*(x) = (\Eb{C \mid Z = z_1, X = x}, \dots, \Eb{C \mid Z = z_m, X = x})^\top$:
\begin{align*}
    \breve \pi_{\breve f}(x) \in \arg\min_{\breve z \in \Z^{\text{simplex}}} \breve f(x)^\top \breve z,
\end{align*}
where ties are again broken by some fixed rules. Given a function class $\breve \F$, we can then consider the induced policy class $\breve\Pi_{\breve \F} = \{\breve \pi_{\breve f}: \breve f \in \breve \F\}$. For any policy $\breve \pi \in \breve\Pi_{\breve\F}$, its output is an one-hot vector, whose entry with value $1$ corresponds to the chosen decision among $z_1, \dots, z_m$.
For any observed decision $Z \in \{z_1, \dots, z_m\}$, we denote its one-hot transformation $\breve Z$ as the zero-one vector whose value-one entry corresponds to the value of $Z$. 
For a given observed total cost $C$, we denote $\breve C$ as the vector all of whose  entries are equal to $C$. 
In the lemma below, we show that the cost of each policy $\breve \pi$ can be also evaluated by some score funtions.

\begin{lemma}\label{lemma: naive-id}
    For any given policy $\breve \pi$ that maps any covariate value $x$ to an $m$-dimensional one-hot vector, its policy cost can be written as follows:
    \begin{align*}
        &V(\breve{\pi}) = \Eb{\breve{\theta}(X, \breve Z, \breve C; \breve f^*, \breve e^*)^\top\breve \pi(X)}, \\
        \text{where } & \breve f^*(x) = (\Eb{C \mid Z = z_1, X = x}, \dots, \Eb{C \mid Z = z_m, X = x})^\top \\
        & \breve e^*(x) = (e^*(z_1 \mid x), \dots, e^*(z_m \mid x)))^\top, ~ e^*(z \mid x) = \Prb{Z = z_j \mid X = x},
    \end{align*}
    and the score function $\breve \theta$ can take three different forms: 
      \begin{enumerate}
  \item (Direct Method) $\breve\theta_{\text{DM}}(x,\breve z,\breve c; \breve f,\breve e) = \breve f(x)$;
  \item (Inverse Propensity Weighting) $\breve\theta_{\text{IPW}}(x,\breve z,\breve c; \breve f,\breve e) = \frac{\breve z}{\breve e(x)}\breve c$;
      \item (Doubly Robust) $\theta_{\text{DR}}(x,\breve z,\breve c; \breve f,\breve e) = \breve f(x) + \frac{\breve z}{\breve e(x)}(\breve c - \breve f(x))$.
  \end{enumerate}
  In the three score functions above, all vector operations are entry-wise operations. 
\end{lemma}

From \cref{lemma: naive-id}, the new formulations above have very similar structure as our previous formulation in \cref{sec: ierm-bandit}. The major differences are that we restrict to the discrete actions $z_1, \dots, z_m$, redefine certain variables accordingly, and consider a simplex set as the constraint. We note that the identification in \cref{lemma: naive-id} mimics the DM, IPW and DR identification in offline contextual bandit learning with discrete actions \cite[\eg,][]{dudik2011doubly,athey2021policy,zhao2012estimating}. Since the identification formulae are analogous to those in \cref{sec: ierm-bandit}, we can easily apply the same policy learning methods  in  \cref{sec: ierm-bandit} and the surrogate losses (such as SPO+ relaxation) in \cref{sec: spo+}.

\section{Additional Experimental Details}\label{sec:exp-details}
In \cref{sec: numerical}, we provide experimental results for different methods under various model specifications and different logging policies. In this section, we further explain the details of  experiment setup, implementation, and provide additional experimental results. All experiments in the paper are implemented on a cloud computing platform with 128 CPUs of model Intel(R) Xeon(R) Platinum 8369B CPU @ 2.70GHz, 250GB RAM and 500GB storage. 

\subsection{Experimental Setup and Implementation Details}\label{sec:exp-setup}
\paragraph{Data generating process.} We first generate i.i.d draws of the covariates $X = (X_1, X_2, X_3)^\top \in \R{3}$ from independent standard normal distribution. Then we simulate the full feedback $Y$ according to the equation $Y = f^*(X) + \epsilon$, where $f^*(X) = a +  W_1^* X_1 + W_2^*X_2 + W_3^* X_3 + W_4^* X_1X_2 + W_5^* X_2X_3 + W_6^* X_1X_3 + W_7^* X_1X_2X_3$ for coefficient vectors $W_1^*, \dots, W_7^* \in \R{d}$ and a random noise $\epsilon$ drawn from the $\text{Unif}[-0.5, 0.5]$. 
To fix the coefficient vectors $W^*_1, \dots, W^*_7$, we draw their entries independently from the $\text{Unif}[0, 1]$ distribution. 
Each element of the intercept vector is sampled from a Gaussian distribution $\mathcal{N}(3,1)$.
We then use the resulting fixed coefficient vectors throughout the experiment.
We sample observed decisions $Z$ from the set of all feasible path decisions $\{z_1, \dots, z_{m}\}$ for $m = 70$, according to different logging policies that will be described shortly. 
For the bandit-feedback setting, the total cost $C = Y^\top Z$ is recorded in the observed data. 
For the semi-bandit feedback setting, the corresponding $Y$ components for the chosen $Z$ are recorded in the observed data. 

We consider two different logging policies. 
One is a random logging policy that uniformly samples each decision from the feasible decisions, regardless of the covariate value. 
The other is a covariate-dependent logging policy. 
For this policy, 
we first remove $20$ feasible decisions that correspond to the optimal decisions for some covariate observations in the testing data, and then randomly sample the observed decisions from the rest $50$ feasible decisions. 
This means that many promising decisions are not explored by the logging policies at all. 
We hope to use this logging policy to demonstrate the value of leveraging the linear structure of the decision-making problem, since exploiting the linear structure allows to extrapolate the feedback from the logged decisions to decisions never explored by the policies. 
We further divide the remaining $50$ feasible decisions into two even groups. 
Our covariate-dependent policy samples the decisions according to the signs of both $X_1$ and  $X_2$. When $X_1 > 0$ and $X_2 > 0$, the logging policy chooses the two groups with probabilities $2/3$ and $1/3$ respectively. When $X_1 > 0$ and $X_2 \le 0$, the logging policy chooses the two groups with probabilities $1/3$ and $2/3$ respectively. When $X_1 \le 0$ and $X_2 > 0$, the policy chooses the two groups with probabilities $3/4$ and $1/4$ respectively. 
When $X_1 \le 0$ and $X_2 \le 0$, the policy chooses the two groups with probabilities $1/4$ and $3/4$ respectively. Once deciding the group, the policy again uniformly samples one decision from the chosen group. We will refer to this policy as the $X_1X_2$-policy.

\paragraph{Specification of the policy-inducing model and nuisance model.} For the policy-inducing model and nuisance model, we consider three different classes. One is the correctly specified class $\{(x_1, x_3, x_3) \mapsto W_0 + W_1 x_1 + W_2x_2 + W_3 x_3 + W_4 x_1x_2 + W_5 x_2x_3 + W_6 x_1x_3 + W_7 x_1x_2x_3: W_0, \dots, W_7 \in \R{}\}$. The second model class $\{(x_1, x_3, x_3) \mapsto  W_0 + W_1 x_1 + W_2x_2 + W_3 x_3 + W_4 x_1x_2 + W_5 x_2x_3: W_0, \dots, W_5 \in \R{}\}$ omits two interaction terms and is thus misspecified (which we refer to as degree-2 misspecification). The third  model class   $\{(x_1, x_3, x_3) \mapsto W_0 +  W_1 x_1 + W_2x_2 + W_3 x_3 : W_1, W_2, W_3 \in \R{}\}$ omits all four interaction terms (which we refer to as degree-4 misspecification). 

\paragraph{Nuisance estimation for IERM methods.} The IERM methods involve two different nuisances. One is a least-squares solution $\tilde f^*$. We estimate this by the least squares regression in \cref{eq: nuisance f}, with the $\F^{\text{N}}$ class being one of the three classes described above (i.e., correct specification, degree-2 misspecification, degree-4 misspecification).
The other nuisance is the conditional Gram matrix $\Sigma^*(x)$, which we use the true value in most of the experiments. We explore the estimation of $\Sigma^*$ in \Cref{sec: estimated sigma}. In that part, 
we estimate the the nuisance $\Sigma^*(x)$ using the propensity score approach described in \cref{remark:sigma0}, which involves estimating the propensity scores $\Prb{Z = z_j \mid X = x}$ for $j = 1, \dots, m$. 
For the random logging policy, we simply estimate $\Prb{Z = z_j \mid X = x}$ for any $x$ by the empirical frequency of the decision $z_j$ in the observed data. 
For the $X_1X_2$-policy, we estimate the propensity scores by classification decision trees trained to classify each instance to one of the observed classes among $z_1, \dots, z_m$. These nuisances are all estimated using the two-fold cross-fitting described in \cref{sec:ierm}.

Since the $\Sigma^*$ matrix is rank-deficient, we cannot directly invert it. Besides taking the pseudo-inverse, we also implement the Lambda regularization described in \Cref{sec: numerical}. For Lambda regularization, we set the regularization parameter $\lambda$ to $1$.

\paragraph{Naive Benchmarks.} We also implemented the  benchmarks in \cref{sec: ipw}. For  ETO, SPO+ DM, and SPO+ DR, we estimate $\tilde f^*(x)$ by regressing $C$ with respect to $X$ using the subsample for each observed decision separately. The regression function class uses similar correctly specified class, degree-2 misspecified class, degree-4 misspecified class mentioned above, with only slight difference in the dimension since the output of $\breve f^*$ is $m$-dimensional while the output of $\tilde f^*$ is $d$-dimensional. 
For SPO+ DR and SPO+ IPW, we need to estimate the propensity scores. These are again estimated by either sample frequency or decision trees. 
Note that some feasible decisions are never observed in the training data. 
For these decisions, the corresponding component of $\tilde f^*$ is heuristically imputed by a pooled regression of $C$ against $X$ using all observed data (regardless of the decision). For SPO+ IPW and SPO+ DR, although the propensity scores for the unseen decisions are zero, they do not impact the policy evaluation since they only need the propensity score for decisions observed in the data.

\begin{table}[t]
\centering
\begin{tabular}{|c|c|c|c|c|c|c|c|c|c|} 
\hline
                  & \multicolumn{3}{c|}{$\F^{\text{N}}$ well-specified} & \multicolumn{3}{c|}{$\F^{\text{N}}$ misspecified degree 2} & \multicolumn{3}{c|}{$\F^{\text{N}}$ misspecified degree 4}  \\ 
\hline
Evaluation method & 400    & 1000   & 1600                              & 400    & 1000   & 1600                                     & 400    & 1000   & 1600                                      \\ 
\hline
DM                & 2.86\% & 0.05\% & 0.02\%                            & 7.34\% & 4.14\% & 3.90\%                                   & 7.29\% &4.45\% & 4.20\%                                    \\ 
\hline
DR                & 2.80\% & 0.16\% & 0.09\%                            & 6.30\% & 2.91\% & 2.46\%                                   & 7.81\% & 4.18\% & 3.64\%                                    \\ 
\hline
DR Lambda         & 3.03\% & 0.11\% & 0.07\%                            & 5.70\% & 2.85\% & 2.41\%                                   & 7.13\% & 3.70\% & 3.19\%                                    \\
\hline
\end{tabular}
\caption{Mean relative regret ratio of different evaluation methods when the nuisance model $\F^{\text{N}}$ is misspecified to different degrees and the policy-inducing model $\F$ is well-specified. The logging policy is a random policy and the feedback type is semi-bandit feedback.Optimization method is SPO+.}
\label{tab:Semi-bandit Nuisance Specifications}
\end{table}

\begin{table}[t]
\centering
\begin{tabular}{|c|c|c|c|c|c|c|c|c|c|} 
\hline
                    & \multicolumn{3}{c|}{DM}  & \multicolumn{3}{c|}{DR}  & \multicolumn{3}{c|}{DR Lambda}  \\ 
\hline
Optimization method & 400    & 1000   & 1600   & 400    & 1000   & 1600   & 400    & 1000   & 1600          \\ 
\hline
SPO+                & 2.86\% & 0.05\% & 0.02\% & 2.80\% & 0.16\% & 0.09\% & 3.03\% & 0.11\% & 0.07\%        \\ 
\hline
PGC                 & 0.84\% & 0.19\% & 0.03\% & 1.14\% & 0.24\% & 0.12\% & 0.88\% & 0.21\% & 0.07\%        \\ 
\hline
PGB                 & 0.82\% & 0.18\% & 0.03\% & 1.44\% & 0.38\% & 0.37\% & 0.85\% & 0.23\% & 0.08\%        \\ 
\hline
PFYL                & 1.74\% & 0.24\% & 0.08\% & 1.82\% & 0.22\% & 0.07\% & 1.76\% & 0.23\% & 0.11\%        \\
\hline
\end{tabular}
\caption{Mean relative regret ratio of different optimization methods when the nuisance model $\F^{\text{N}}$ and the policy-inducing model $\F$ are  well-specified. The logging policy is a random policy and the feedback type is semi-bandit feedback.}
\label{tab:Semi-bandit Evaluation Methods}
\end{table}

\subsection{Semi-bandit Feedback Experiment Result} \label{sec: semi experiments}
\cref{tab:Semi-bandit Nuisance Specifications} and \cref{tab:Semi-bandit Evaluation Methods} present results under semi-bandit feedback that align with the bandit-feedback findings in \cref{sec: synthetic}. 

\paragraph{Nuisance Class ($\mathcal{F}^{\text{N}}$) Specification.}
Compared to \cref{tab:Feedback type}, a key observation from \cref{tab:Semi-bandit Nuisance Specifications} is that the performance of our end-to-end IERM approach degrades significantly when the nuisance model $\mathcal{F}^{\text{N}}$ is misspecified. This sensitivity to nuisance reinforces the practical guideline suggested in \cref{sec: synthetic}: employing a flexible nuisance class $\mathcal{F}^{\text{N}}$ for accurate estimation while maintaining a tractable policy class $\mathcal{F}$.

Furthermore, the results confirm the clear tradeoff between DM and DR scores observed in bandit settings. When $\mathcal{F}^{\text{N}}$ is misspecified, the DR score (particularly the regularized DR-Lambda variant) outperforms DM, leveraging its debiasing structure to mitigate nuisance estimation errors. Conversely, with well-specified nuisances, DM achieves superior performance due to its lower variance, free from the instability introduced by inverse Gram matrix operations.

\paragraph{Comparison of Evaluation and Optimization Methods.}
\cref{tab:Semi-bandit Evaluation Methods} compares evaluation methods and optimization surrogates under semi-bandit feedback when both $\mathcal{F}$ and $\mathcal{F}^{\text{N}}$ are well-specified. Consistent with our bandit-feedback findings, DM outperforms DR in this regime, while all surrogate losses demonstrate decreasing regret with sample size, confirming their general effectiveness. Among optimization methods, PGC exhibits the strongest performance with small sample sizes ($n=400$), while SPO+ emerge as top performers as sample size increases, mirroring the pattern observed in bandit settings.

These results collectively demonstrate that the performance characteristics and methodological insights established for bandit feedback in \cref{sec: synthetic} extend consistently to the semi-bandit setting, reinforcing the robustness and general applicability of our IERM framework across partial-feedback environments.

\begin{table}[t]
\centering
\begin{tblr}{
  cells = {c},
  cell{1}{3} = {c=3}{},
  cell{1}{6} = {c=3}{},
  cell{3}{1} = {r=2}{},
  cell{5}{1} = {r=2}{},
  cell{7}{1} = {r=2}{},
  vlines,
  hline{1-3,5,7,9} = {-}{},
  hline{4,6,8} = {2-8}{},
}
                                               &  & DM      &         &        & DR PI   &         &        \\
$\F$ Specification            & Methods           & 400     & 1000~   & 1600   & 400     & 1000    & 1600   \\
$\F$ well-specified           & Naive Benchmark   & 46.52\% & 25.36\% & 3.73\% & 46.48\% & 27.63\% & 3.93\% \\
                                               & Our Method   & 4.66\%  & 0.30\%  & 0.09\% & 4.63\%  & 0.47\%  & 0.24\% \\
{$\F$ misspecified\\degree 2} & Naive Benchmark   & 46.82\% & 24.35\% & 4.70\% & 46.84\% & 26.42\% & 4.87\% \\
                                               & Our Method    & 5.47\%  & 2.04\%  & 1.78\% & 6.01\%  & 2.76\%  & 2.50\% \\
{$\F$ misspecified\\degree 4} & Naive Benchmark   & 46.81\% & 22.92\% & 5.34\% & 46.84\% & 24.82\% & 5.45\% \\
                                               & Our Method    & 5.03\%  & 3.17\%  & 3.03\% & 5.38\%  & 3.51\%  & 3.37\% 
\end{tblr}
\caption{Mean relative regret ratio of naive benchmark and our method under different evaluation method and different policy-inducing model $\F$ specifications when the nuisance model $\F^{\text{N}}$ is well-specified. The logging policy is a random policy, the optimization method is SPO+ and the feedback type is bandit feedback.} \label{tab:Naive Benchmark result}
\end{table}

\subsection{Naive Benchmarks Experiment Result} \label{sec: naive benchmarks}
We also considered the Naive method mentioned in Appendix \ref{sec: ipw} as a benchmark to compare with our proposed approach. The experimental comparison, as shown in \cref{tab:Naive Benchmark result}, was conducted under the random policy setting with bandit feedback, using SPO+ as the optimization method. The results demonstrate that our method significantly outperforms the Naive benchmark, regardless of whether the evaluation method is DM or DR PI. This advantage stems from our method's effective utilization of the linear structure among decisions, whereas the naive benchmark simply discretizes all feasible decisions without leveraging this linear structure, leading to its inferior performance.

\subsection{$X_1X_2$-policy Experiment Results} \label{sec: x1x2}

\cref{tab:Feedback type X1X2} to \cref{tab:Evaluation Methods X1X2} present experimental results under the $X_1X_2$-dependent logging policy, which exhibits more complex covariate-dependent exploration patterns than the random policy. We report the following findings consistent with the findings in \cref{sec: synthetic} for random logging policy:
\begin{itemize}
    \item Performance degrades for all methods as feedback becomes more limited, with regret increasing from semi-bandit to bandit feedback settings.
    \item When the policy-inducing class $\mathcal{F}$ is well-specified (\cref{tab:Feedback type X1X2}), the advantage of IERM methods over ETO becomes increasingly pronounced in partial-feedback settings. In the bandit feedback regime, IERM methods achieve substantially lower regret than ETO, mirroring the pattern observed with random logging.
    \item Under $\mathcal{F}$ misspecification (\cref{tab:Feedback type X1X2}), IERM methods maintain their consistent advantage over ETO across both semi-bandit and bandit settings, demonstrating the framework's robustness to model misspecification even with complex, covariate-dependent logging policies. 
    \item %
    The performance of IERM methods degrades significantly when $\mathcal{F}^{\text{N}}$ is misspecified, confirming the sensitivity of end-to-end approaches to accurate nuisance estimation across different logging policies. This reaffirms the practical guideline of employing flexible nuisance classes regardless of the data collection mechanism.
    \item The bias-variance tradeoff between DM and DR scores persists under this more complex logging policy. When $\mathcal{F}^{\text{N}}$ is misspecified, DR-Lambda consistently outperform DM and DR PI. Conversely, with well-specified nuisances, DM achieves superior performance.
    \item %
    Among surrogate losses, PGC exhibits strong performance with small sample sizes ($n=400$), while SPO+, PGC and PFYL emerge as top performers as sample size increases. All methods show decreasing regret with sample size, confirming their general effectiveness even under covariate-dependent logging.
\end{itemize}

The consistent findings across random and $X_1X_2$-policies demonstrate that the performance characteristics and methodological insights established in \cref{sec: synthetic} are robust to the specific data collection mechanism, further validating the general applicability of our IERM framework across diverse partial-feedback environments.

\begin{table}
\centering
\begin{tblr}{
  cells = {c},
  cell{1}{3} = {c=3}{},
  cell{1}{6} = {c=3}{},
  cell{3}{1} = {r=4}{},
  cell{7}{1} = {r=4}{},
  vlines,
  hline{1-3,7,11} = {-}{},
  hline{4-6,8-10} = {2-8}{},
}
                        &                & $\F$ well-specified &        &        & $\F$ misspecified degree 4 &         &         \\
Feedback                & Methods        & 400              & 1000~  & 1600   & 400                      & 1000    & 1600    \\
{Semi-bandit\\Feedback} & ETO            & 0.76\%           & 0.06\% & 0.03\% & 6.64\%                   & 5.17\%  & 4.80\%  \\
                        & SPO+ DM        & 1.50\%           & 0.07\% & 0.02\% & 4.38\%                   & 2.94\%  & 2.88\%  \\
                        & SPO+ DR PI        & 1.96\%           & 0.31\% & 0.15\% & 4.58\%                   & 3.36\%  & 3.26\%  \\
                        & SPO+ DR Lambda & 1.69\%           & 0.17\% & 0.07\% & 4.46\%                   & 3.20\%  & 3.09\%  \\
{Bandit\\Feedback}      & ETO            & 21.51\%           & 3.28\% & 0.79\% & 27.87\%                  & 19.06\% & 15.80\% \\
                        & SPO+ DM        & 7.09\%           & 0.31\% & 0.14\% & 6.44\%                   & 3.13\%  & 2.99\%  \\
                        & SPO+ DR PI~       & 7.18\%           & 0.52\% & 0.27\% & 6.74\%                   & 3.61\%  & 3.41\%  \\
                        & SPO+ DR Lambda & 6.65\%           & 0.34\% & 0.15\% & 6.51\%                   & 3.30\%  & 3.13\%  
\end{tblr}
\caption{Mean relative regret ratio of different optimization methods, different feedback type and different policy-inducing model $\F$ specifications when the nuisance model $\F^{\text{N}}$ is well-specified. The logging policy is a $X_1X_2$-policy and the feedback type is bandit feedback.} \label{tab:Feedback type X1X2}
\end{table}

\begin{table}
\centering
\begin{tabular}{|c|c|c|c|c|c|c|c|c|c|} 
\hline
                  & \multicolumn{3}{c|}{$\F^{\text{N}}$ well-specified} & \multicolumn{3}{c|}{$\F^{\text{N}}$ misspecified degree 2} & \multicolumn{3}{c|}{$\F^{\text{N}}$ misspecified degree 4}  \\ 
\hline
Evaluation method & 400    & 1000   & 1600               & 400    & 1000   & 1600                      & 400     & 1000    & 1600                     \\ 
\hline
DM                & 7.09\% & 0.31\% & 0.14\%             & 23.28\% & 12.59\% & 10.07\%                    & 24.20\% & 15.91\% & 13.45\%                  \\ 
\hline
DR PI             & 7.18\% & 0.52\% & 0.27\%             & 21.86\% & 12.67\% & 10.66\%                    & 23.52\% & 16.77\% & 14.10\%                  \\ 
\hline
DR Lambda         & 6.65\% & 0.34\% & 0.15\%             & 21.56\% & 11.49\% & 9.42\%                    & 23.27\% & 15.60\% & 12.51\%                  \\
\hline
\end{tabular}
\caption{Mean relative regret ratio of different evaluation methods when the nuisance model $\F^{\text{N}}$ is misspecified to different degrees and the policy-inducing model $\F$ is well-specified. The logging policy is a $X_1X_2$-policy and the feedback type is bandit feedback. Optimization method is SPO+.} \label{tab:Semi-bandit Nuisance Specifications X1X2}
\end{table}

\begin{table}
\centering
\begin{tabular}{|c|c|c|c|c|c|c|c|c|c|} 
\hline
                    & \multicolumn{3}{c|}{DM}  & \multicolumn{3}{c|}{DR PI} & \multicolumn{3}{c|}{DR Lambda}  \\ 
\hline
Optimization method & 400    & 1000   & 1600   & 400    & 1000   & 1600     & 400    & 1000   & 1600          \\ 
\hline
SPO+                & 7.09\% & 0.31\% & 0.14\% & 7.18\% & 0.52\% & 0.27\%  & 6.65\% & 0.34\% & 0.15\%        \\ 
\hline
PGC                 & 1.26\% & 0.26\% & 0.14\% & 1.75\% & 0.32\% & 0.27\%   & 1.62\% & 0.30\% & 0.24\%        \\ 
\hline
PGB                 & 1.44\% & 0.28\% & 0.16\% & 2.00\% & 0.30\% & 0.18\%   & 1.63\% & 0.43\% & 0.37\%        \\ 
\hline
PFYL                & 1.92\% & 0.24\% & 0.12\% & 2.07\% & 0.33\% & 0.19\%   & 2.00\% & 0.32\% & 0.19\%        \\
\hline
\end{tabular}
\caption{Mean relative regret ratio of different optimization methods when the nuisance model $\F^{\text{N}}$ and the policy-inducing model $\F$ are  well-specified. The logging policy is a $X_1X_2$-policy and the feedback type is bandit feedback.} \label{tab:Evaluation Methods X1X2}
\end{table}

\subsection{Estimation of $\Sigma^*$ in IERM} \label{sec: estimated sigma}
To validate the discussion in \cref{remark:sigma0} that estimating $\Sigma^*$ based on available data is feasible even when the true logging policy is not perfectly known, we conduct additional experimentals in this section. For the random logging policy case in \cref{tab:Sigma Experiment Random}, we estimate the $\Sigma^*$ using empirical frequency $\hat{e}(z_j)$ as a substitute for the propensity score $e^*(z_j|X)$ derived from historically observed trajectories. For the $X_1X_2$-policy case in \cref{tab:Sigma Experiment depth 3 X1X2}, we employ decision tree models of varying depths to learn the propensity score, ultimately obtaining $\hat{e}(z_j|X)$ to estimate the estimated $\Sigma^*$. 

Experimental results demonstrate that the final regret performance remains largely unaffected regardless of whether the true $\Sigma^*$ or its estimated counterpart $\hat{\Sigma}$ is used, under both random and $X_1X_2$-logging policies. This robustly validates the reasonableness and effectiveness of our proposed $\Sigma^*$ estimation method using historical data as presented in \cref{remark:sigma0}.

\begin{table}
\centering
\begin{tabular}{|c|c|c|c|c|c|c|} 
\hline
                    & \multicolumn{3}{c|}{\begin{tabular}[c]{@{}c@{}}DR Lambda\\Estimated $\Sigma^*$\end{tabular}} & \multicolumn{3}{c|}{\begin{tabular}[c]{@{}c@{}}DR Lambda\\True $\Sigma^*$\end{tabular}}  \\ 
\hline
Optimization method & 400    & 1000   & 1600                                                                 & 400    & 1000   & 1600                                                              \\ 
\hline
SPO+                & 4.37\% & 0.28\% & 0.12\%                                                              & 4.35\% & 0.29\% & 0.12\%                                                             \\ 
\hline
PGC                 & 1.49\% & 0.30\% & 0.20\%                                                         & 1.49\% & 0.30\% & 0.20\%                                                          \\ 
\hline
PGB                  & 1.46\% & 0.42\% & 0.34\%                                                       & 1.46\% & 0.42\% & 0.34\%                                                       \\ 
\hline
PFYL                & 1.78\% & 0.31\% & 0.16\%                                                        & 1.78\% & 0.31\% & 0.16\%                                                     \\
\hline
\end{tabular}
\caption{Mean relative regret ratio of different optimization methods using $\Sigma$ estimated by the frequency or computed based on the true logging policy when the nuisance model $\F^{\text{N}}$ and the policy-inducing model $\F$ are  well-specified. The logging policy is a random policy, evaluation method is DR Lambda and the feedback type is bandit feedback.}
\label{tab:Sigma Experiment Random}
\end{table}

\begin{table}
\centering
\begin{tblr}{
  cells = {c},
  cell{1}{2} = {c=3}{},
  cell{1}{5} = {c=3}{},
  cell{1}{8} = {c=3}{},
  hlines,
  vlines,
}
                    & {DR Lambda \\~Tree max depth 3} &        &        & {DR Lambda~Tree\\~Tree~max depth 2} &        &        & {DR Lambda\\True $\Sigma^*$} &        &        \\
Optimization method & 400                                  & 1000   & 1600   & 400                                  & 1000   & 1600   & 400                     & 1000   & 1600   \\
SPO+                & 6.65\%                               & 0.34\% & 0.15\% & 6.66\%                               & 0.32\% & 0.15\% & 7.03\%                  & 0.32\% & 0.15\% \\
PGC                 & 1.62\%                               & 0.30\% & 0.24\% & 1.62\%                               & 0.30\% & 0.24\% & 1.62\%                  & 0.30\% & 0.24\% \\
PGB                 & 1.63\%                               & 0.43\% & 0.37\% & 1.63\%                               & 0.43\% & 0.37\% & 1.63\%                  & 0.43\% & 0.38\% \\
PFYL                & 2.00\%                               & 0.32\% & 0.19\% & 2.00\%                               & 0.31\% & 0.19\% & 2.01\%                  & 0.32\% & 0.20\% 
\end{tblr}
\caption{Mean relative regret ratio of different optimization methods using $\Sigma$ estimated by the decision tree with different max depth or computed based on the true logging policy when the nuisance model $\F^{\text{N}}$ and the policy-inducing model $\F$ are  well-specified. The logging policy is a $X_1X_2$-policy, evaluation method is DR Lambda and the feedback type is bandit feedback. }
\label{tab:Sigma Experiment depth 3 X1X2}
\end{table}

\subsection{Real data Experiment Result} \label{sec: real data exp extra}

For real data experiments, we also conducted tests under the $X_1X_2$-policy as the logging policy. Specifically, we selected the windspeed $x_{w}$ and visibility $x_{v}$ features and partitioned all feasible solutions into two groups. The selection mechanism was designed as follows. 

When $x_{w} > 5.5$ and $x_{v} > 9.5$, the logging policy chooses the two groups with probabilities $2/3$ and $1/3$ respectively. When $x_{w} > 5.5$ and $x_{v} \leq 9.5$, the logging policy chooses the two groups with probabilities $1/3$ and $2/3$ respectively. When $x_{w} \leq 5.5$ and $x_{v} > 9.5$, the logging policy chooses the two groups with probabilities $3/4$ and $1/4$ respectively. When $x_{w} \leq 5.5$ and $x_{v} \leq 9.5$, the logging policy chooses the two groups with probabilities $1/4$ and $3/4$ respectively. Once deciding the group, the policy again uniformly samples one decision from the chosen group. We will refer to this policy as the $X_1X_2$-policy.

The results from our real-world experiment under the covariate-dependent $X_1X_2$-logging policy are summarized in \cref{tab:Real data result X1X2}. These findings are fully consistent with the patterns observed in both our synthetic experiments and the main real data analysis in \cref{sec: real data}.

Under this covariate-dependent logging policy, we again observe a clear  increase in regret when moving from semi-bandit to bandit feedback settings. This reaffirms the fundamental challenge of learning with partial information, even when dealing with real-world data and covariate-dependent data collection mechanisms. The end-to-end IERM methods demonstrate a decisive advantage in the partial-feedback settings. This performance gap is particularly pronounced at shorter time horizons (\eg, 6-12 months), where data scarcity amplifies the limitations of the two-stage ETO approach. The consistent superiority of IERM methods under bandit feedback highlights the framework's robustness and adaptability when learning from limited, partially observable real-world data.

These results collectively validate that the core insights from our synthetic experiments and main real-data analysis extend to more complex, covariate-dependent logging scenarios, further strengthening the practical relevance of our proposed IERM framework for real-world applications.

\begin{table}
\centering
\begin{tabular}{|c|c|c|c|c|c|c|c|c|} 
\hline
               & \multicolumn{4}{c|}{Semi-bandit Feedback}                                                                                        & \multicolumn{4}{c|}{Bandit Feedback}                                                                                                \\ 
\hline
Period (Months) & ETO    & SPO+ DM & \begin{tabular}[c]{@{}c@{}}SPO+\\DR PI\end{tabular} & \begin{tabular}[c]{@{}c@{}}SPO+ DR\\Lambda\end{tabular} & ETO     & SPO+ DM & \begin{tabular}[c]{@{}c@{}}SPO+ \\DR PI\end{tabular} & \begin{tabular}[c]{@{}c@{}}SPO+ DR\\Lambda\end{tabular}  \\ 
\hline
24              & 1.14\% & 0.84\%  & 0.84\%                                              & 0.84\%                                                  & 18.79\% & 10.80\% & 8.27\%                                               & 7.60\%                                                   \\ 
\hline
18            & 1.36\% & 0.92\%  & 0.92\%                                              & 0.92\%                                                  & 22.85\% & 12.29\% & 10.32\%                                              & 10.77\%                                                  \\ 
\hline
12              & 1.69\% & 0.98\%  & 0.98\%                                              & 0.98\%                                                  & 27.18\% & 16.58\% & 15.32\%                                              & 15.29\%                                                  \\ 
\hline
6           & 2.21\% & 1.01\%  & 1.01\%                                              & 1.01\%                                                  & 27.45\% & 18.22\% & 18.05\%                                              & 18.41\%                                                  \\
\hline
\end{tabular}
\caption{Mean relative regret ratio of different evaluation methods in real data experiment. The logging policy is a $X_1X_2$-policy.}
\label{tab:Real data result X1X2}
\end{table}

\end{document}